\documentclass{article}

\usepackage[main, final]{neurips_2026}

\usepackage[utf8]{inputenc} % allow utf-8 input
\usepackage[T1]{fontenc}    % use 8-bit T1 fonts
\usepackage{hyperref}       % hyperlinks
\usepackage{url}            % simple URL typesetting
\usepackage{booktabs}       % professional-quality tables
\usepackage{amsfonts}       % blackboard math symbols
\usepackage{nicefrac}       % compact symbols for 1/2, etc.
\usepackage{microtype}      % microtypography
\usepackage{xcolor}         % colors

\usepackage{amsmath}
\usepackage{amssymb}
\usepackage{amsthm}
\usepackage{mathtools}
\usepackage{graphicx}
\usepackage{upgreek}
\usepackage{enumitem}

\usepackage{nicefrac}

\usepackage{algorithm}
\usepackage{algpseudocode}
\usepackage{placeins}
\usepackage{wrapfig}
\usepackage{multirow}

\newcommand{\iid}{\stackrel{\mathrm{iid}}{\sim}}

\newcommand{\ind}{\perp\!\!\!\perp}

\newtheorem{theorem}{Theorem}
\newtheorem{corollary}{Corollary}
\newtheorem{lemma}{Lemma}
\newtheorem{proposition}{Proposition}

\theoremstyle{definition}

\theoremstyle{remark}
\newtheorem{remark}{Remark}

\DeclareMathOperator*{\argmin}{arg\,min}

\definecolor{ggplot1}{HTML}{E41A1C}
\definecolor{ggplot2}{HTML}{377EB8}

\newcommand{\safeincludegraphics}[2][]{%
  \IfFileExists{#2}{\includegraphics[#1]{#2}}{%
    \fbox{\begin{minipage}[c][0.28\textheight][c]{0.92\linewidth}
      \centering Missing figure file:\\[0.5ex]{\scriptsize\nolinkurl{#2}}
    \end{minipage}}%
  }%
}

\renewcommand{\theHfigure}{main.\arabic{figure}}
\renewcommand{\theHtable}{main.\arabic{table}}

\title{When Does Trimming Help Conformal Prediction?\\
  A Retained-Law Diagnostic under Calibration Contamination
}

\author{%
  Congye Wang\textsuperscript{1,2} \\
  \textsuperscript{1}MARS,
  School of Mathematical Sciences, Lancaster University, UK \\
  \textsuperscript{2}School of Mathematics, Statistics \& Physics,
  Newcastle University, UK \\
  \texttt{c.wang58@lancaster.ac.uk}
}

\begin{document}

\maketitle
\begin{abstract}
  Trimming suspicious calibration points is a common response to contamination in conformal prediction. Its effect on clean-target coverage, however, is governed by the retained law induced by trimming, not by the contamination level alone. We analyse fixed-threshold trimming as conditioning rather than purification. It replaces the contaminated calibration law with a retained law, reducing clean-target coverage to a one-dimensional score-CDF transfer problem with an exact finite-sample identity. A componentwise bound on the transfer gap gives a population-level diagnostic. This separates a clean-side covariance cost from a retained-contamination cost, governed by the dirty-to-clean retention ratio. Trimming helps when the anomaly score separates retention probabilities while remaining score-neutral on the clean population. Otherwise, it cannot substantially reduce contamination through the retained mixture coefficient. We also give finite-sample certificate templates that provide numerical guarantees under independent audit.
\end{abstract}

\section{Introduction}
Trimming suspicious calibration points is a common response to contamination. A practitioner fits an anomaly score, such as an isolation, density, Stein, or geometric outlier score. Points with large anomaly values are then discarded, and ordinary split conformal prediction is applied to the retained sample. This procedure is intuitive and often believed to restore validity under calibration contamination. However, its effect on coverage depends on the retained law induced by trimming, as well as on properties of the anomaly score that cannot be inferred from the contamination level alone.

The difficulty is calibration--test exchangeability. Ordinary split conformal prediction is valid because, conditional on the fitted predictor and the nonconformity score, the calibration scores and the test score are exchangeable \citep{vovk2005algorithmic, shafer2008tutorial, lei2018distributionfree, angelopoulos2021gentle}. Under calibration contamination, this exchangeability fails. The calibration sample is drawn from $P^\star = (1-\varepsilon)P + \varepsilon Q$, while the test point is drawn from the clean law $P$. Thus, calibration observations remain exchangeable among themselves, but not jointly exchangeable with the test point.

Trimming does not repair this mismatch. With a fixed threshold $t^\star$, it replaces $P^\star$ by the retained law
\begin{equation}
R = P^\star(\,\cdot \mid S \le t^\star),
\label{eq:retained-law}
\end{equation}
which is generally neither $P^\star$ nor $P$. Conditional on the retained count, the retained observations are i.i.d. from $R$. Hence, conformal rank validity is exact only for a test point drawn from $R$, not from $P$. Clean-target validity is therefore a transfer problem from $R$ back to $P$.

We turn this reframing into a population-level diagnostic for when trimming pays off. The transfer loss reduces to the supremum CDF gap between $R$ and $P$ under the nonconformity score. It also admits a finite-sample lower bound that is sharp within the scalar-discrepancy information class. Substituting the retained-law mixture decomposition into this scalar loss gives a componentwise upper bound with two interpretable terms. The first is the clean-side cost of trimming. When $S$ depends on the nonconformity score $A$ under the clean law, conditioning on $\{S \le t^\star\}$ distorts the calibration distribution even without contamination. We show that this cost is exactly a one-sided covariance between the retention event and the score CDF under $P$. This formalises the design principle that a good anomaly score should be score-neutral on the clean population. The second term is the retained-contamination cost. Its coefficient is governed by the dirty-to-clean retention ratio. When the score separates clean and contaminated points, this coefficient can be made much smaller than $\varepsilon$. When the retention probabilities coincide, trimming cannot reduce contamination through the mixture coefficient. A favourable retained dirty score law is then only a separate, secondary mechanism. These two mechanisms organise both the theory and the experiments, through a score-visible regime and a no-separation regime.

This population diagnostic becomes a numerical guarantee only when its ingredients are bounded by independent information. We therefore complement it with finite-sample certificate templates. These templates provide coverage guarantees under independent audit, either through componentwise auditing of the mixture decomposition or through direct binomial inversion on an independent clean audit sample. We treat same-sample threshold selection separately, using a finite-grid empirical-CDF bound.

\section{Background}
\label{sec:background}
\emph{Split conformal prediction.} We work in the standard split conformal setting. Let $Z = (X, Y) \in \mathcal{Z}$, and let $A : \mathcal{Z} \to \mathbb{R}$ be a nonconformity score. The score is constructed using data independent of the final calibration set. The split-conformal prediction set at level $\alpha$ is
\begin{equation*}
    C_\alpha(x) = \{y : A(x, y) \le \hat\tau_\alpha\},
\end{equation*}
where $\hat\tau_\alpha$ is the inclusive empirical $(1-\alpha)$-quantile of the calibration scores. Once the model-fitting stage is fixed, validity reduces to a one-dimensional rank argument. The calibration scores and the test score are exchangeable. Hence, the resulting set covers $Y$ with probability at least $1-\alpha$ \citep{vovk2005algorithmic, shafer2008tutorial, lei2018distributionfree, angelopoulos2021gentle}. Regression-oriented variants, such as conformalised quantile regression, modify the score construction. They nevertheless preserve the same calibration-quantile logic \citep{romano2019conformalized}. Other work seeks localised, weighted, or conditional guarantees, rather than the marginal clean-test guarantee considered here \citep{guan2023localized, gibbs2025conditional, hore2025localweights}. A growing literature studies failures of exchangeability caused by exogenous test-time distribution shift. Examples include covariate shift, label shift, time-series and online shift, non-exchangeability, and distributional robustness \citep{tibshirani2019conformal, podkopaev2021labelshift, gibbs2021adaptive, gibbs2024online, barber2023beyond, oliveira2024split, qiu2023prediction, yang2024doubly, zaffran2022adaptive, bhatnagar2023improved, aolaritei2025levy}. We study a different breakdown. It occurs on the calibration side and is induced endogenously by the practitioner’s preprocessing.

\emph{Contamination model.} The calibration observations are i.i.d. from
$P^\star = (1-\varepsilon) P + \varepsilon Q$,
where $Q$ is an unknown contaminating law and $\varepsilon\in[0,1)$ is the contamination level. The test point retains the clean law $P$. This is the classical Huber $\varepsilon$-contamination model \citep{huber1964robust}, applied at the calibration stage. The contamination is non-adaptive. The law $Q$ may be arbitrary, but it is not chosen with knowledge of the realised calibration sample or the trimming rule. The closest existing analysis is \citet{clarkson2024contamination}. They derive worst-case robust quantile corrections, which inflate the conformal threshold to absorb arbitrary $Q$. Their guarantee is uniform over $Q$. In contrast, our diagnostic conditions on the practitioner’s chosen anomaly score and threshold. Their question is how much to inflate to be safe against any $Q$; ours is when trimming helps, given that the practitioner is already trimming. Adjacent work studies robust conformal validation under bounded corruption \citep{bashari2025robust, cauchois2024robustvalidation, cauchois2024weak}, local or cell-wise outlier models \citep{peng2025cellwise}, and label-noise corrections for classification \citep{romano2020classification, einbinder2024labelnoise, sesia2025adaptive, penso2024score, bortolotti2025noise, xi2025exploring}. These target distinct corruption structures. Label-noise work, in particular, corrupts $Y$ given $X$ but preserves $X$. Our model imposes neither restriction. It isolates how trimming interacts with the conformal quantile itself.

\emph{Trimming.} To trim, we introduce a second score $S : \mathcal{Z} \to \mathbb{R}$ for filtering. We also choose a threshold $t^\star \in \overline{\mathbb{R}} := \mathbb{R} \cup \{+\infty\}$. Calibration points are retained when $S(Z) \le t^\star$. Trimming has a long tradition in robust statistics, from Tukey’s trimmed means to Huber’s influence-function analysis \citep{huber1964robust}. In that setting, the target is an estimator, and trimming limits the influence of contaminating mass. Here, we study trimming as a transformation of the calibration law. The target is the conformal quantile. Trimming acts by conditioning the calibration law on a retention event, not by purifying it. The threshold is fixed after conditioning on auxiliary information. This may be a clean reference split, an independent tuning split, or external domain knowledge. The calibration sample therefore remains i.i.d. and independent of $t^\star$. Same-sample threshold selection is treated separately in Section~\ref{sec:method}.

\emph{Choice of anomaly score.} We keep the score $S$ general. It may be covariate-only, response-dependent, or externally supplied. The retained-law theorem treats it as an arbitrary fixed measurable function. In practice, the choice of $S$ matters. Response-dependent filtering often couples retention with the nonconformity score under the clean law. This can inflate the clean-side distortion identified in Section~\ref{sec:method}. Our operational examples therefore favour covariate-geometric scores. For these scores, the clean distortion and the contamination-removal effect can be diagnosed separately. Section~\ref{sec:method} formalises the central object of this paper, the retained law
\begin{equation*}
    R = P^{\star}(\,\cdot \mid S \le t^{\star}).
\end{equation*}
It then develops the diagnostic and certificate routes introduced above.

\section{Methodology}
\label{sec:method}
We work conditionally on an auxiliary sigma-field $\mathcal{G}$. This captures the realised model-fitting and design stages. The fitted predictor, the nonconformity score $A$, the anomaly score $S$, and the trimming threshold $t^\star \in \overline{\mathbb{R}}$ are all $\mathcal{G}$-measurable.

The final calibration sample is i.i.d. from the contaminated law

\begin{equation*}
    P^\star = (1-\varepsilon)P + \varepsilon Q,
\end{equation*}
with $\varepsilon\in[0,1)$. The clean test point is independent of the calibration sample and has law $P$.

All statements below are marginal in the clean test point and conditional on $\mathcal{G}$. We do not claim conditional, subgroup, or adversarial-contamination validity. Throughout, we distinguish population-level diagnostics from certificate templates. The latter yield numerical guarantees only when their ingredients are bounded by independent audit information.

\subsection{Setup and the retained law}
\label{subsec:trimmed-setup}

Given a fixed threshold $t^\star$, the trimmed procedure keeps calibration points with $S(Z_i) \le t^\star$. It then computes the inclusive $(1-\alpha)$-quantile of their nonconformity scores. Let $N_{\rm keep}$ denote the retained count. We set $\hat\tau_\alpha = +\infty$ if $N_{\rm keep} = 0$ or if
\begin{equation*}
    \lceil(N_{\rm keep}+1)(1-\alpha)\rceil = N_{\rm keep}+1.
\end{equation*}
Otherwise, $\hat\tau_\alpha$ is the $\lceil(N_{\rm keep}+1)(1-\alpha)\rceil$-th retained order statistic. The prediction set is
\begin{equation*}
    C_\alpha(x) = \{y : A(x,y) \le \hat\tau_\alpha\},
\end{equation*}
% with $F_M^A(+\infty)=1$. The full algorithm is given in Appendix~\ref{app:algorithm}.

For any probability law $M$ on $\mathcal Z$, let $M^A$ denote the distribution of the nonconformity score $A(Z)$ when $Z\sim M$. Define its cumulative distribution function by

\begin{equation*}
    F_M^A(a):=M{A(Z)\leq a}.
\end{equation*}

We use the convention $F_M^A(+\infty)=1$. For two probability laws $M$ and $N$, define the one-sided score-CDF discrepancy

\begin{equation*}
    d_{M\to N,+}^A := \sup_{a\in\mathbb R} \bigl(F_M^A(a)-F_N^A(a)\bigr)_+,
\end{equation*}
where $(u)_+=\max{u,0}$. This discrepancy measures the largest positive difference between the score CDFs under $M$ and $N$. The full trimming algorithm is given in Appendix~\ref{app:algorithm}. The central object of our analysis is the \emph{retained law}
\begin{equation}
  R = P^\star(\,\cdot \mid S \le t^\star\,).
  \label{eq:retained-law-formal}
\end{equation}
With clean and dirty retention probabilities $p_c = P(S \le t^\star)$ and $p_d = Q(S \le t^\star)$, and $P_{\rm keep}, Q_{\rm keep}$ the corresponding conditional laws when defined, Bayes' rule gives
\begin{equation}
  R \;=\; (1 - \tilde\varepsilon_\star)\,P_{\rm keep} \;+\; \tilde\varepsilon_\star\, Q_{\rm keep},
  \qquad
  \tilde\varepsilon_\star \;=\; \frac{\varepsilon p_d}{(1-\varepsilon)p_c + \varepsilon p_d}.
  \label{eq:retained-law-decomposition}
\end{equation}
We assume $\mu_{\rm keep} := (1-\varepsilon)p_c + \varepsilon p_d > 0$. If $p_d=0$, $Q_{\rm keep}$ may be defined arbitrarily, and $\tilde\varepsilon_\star D_{Q,+}$ is interpreted as zero. If $p_c=0$, the direct gap $d_{R\to P,+}^A$ remains meaningful, but componentwise quantities involving $P_{\rm keep}$ are not used. The \emph{retained mixture coefficient} $\tilde\varepsilon_\star$ is not itself a measure of harmful contamination. Its effect on undercoverage is mediated by the product $\tilde\varepsilon_\star D_{Q,+}$.

Three score-law distances control the analysis. The first is the retained-to-clean discrepancy,
\begin{equation*}
    d_{R \to P,+}^A = \sup_t(F_R^A(t) - F_P^A(t))_{+},
\end{equation*}
which governs clean-coverage transfer. The second is the clean trimming distortion,
\begin{equation*}
    \Delta_{\rm trim,+}^{A} = \sup_t(F_{P_{\rm keep}}^A(t) - F_P^A(t))_{+},
\end{equation*}
which measures how filtering perturbs the clean score CDF. The third is the retained dirty discrepancy,
\begin{equation*}
    D_{Q,+} = \sup_t(F_{Q_{\rm keep}}^A(t) - F_P^A(t))_+,
\end{equation*}
which measures how the retained contaminating component differs from the clean target.

Two-sided and downward variants are defined analogously. They are used for upper-coverage diagnostics in Appendix~\ref{app:upper-bound}.

\subsection{Retained-law transfer and finite-sample sharpness}
\label{subsec:exact-scalar-mixture}

Conditional on $N_{\rm keep}=n$, the retained calibration scores are i.i.d.\ from $R$. Hence the trimmed cutoff is exchangeable with an independent test score from the retained law. This gives the standard rank guarantee under $R$. It then transfers to the clean target through the score-CDF gap.

\begin{proposition}[Retained-law rank validity]
  \label{prop:retained-law-transfer}
  Under the setup of Section~\ref{subsec:trimmed-setup} with $\mu_{\rm keep} > 0$, $\mathbb{E}\{F_R^A(\hat\tau_\alpha) \mid \mathcal{G}\} \ge 1 - \alpha$, and for any independent target law $P_0$,
  \[
    \mathbb{P}_{Z_{\rm new} \sim P_0}\{A(Z_{\rm new}) \le \hat\tau_\alpha \mid \mathcal{G}\}
    \;\ge\; 1 - \alpha - \sup_t\bigl(F_R^A(t) - F_{P_0}^A(t)\bigr)_+.
  \]
\end{proposition}

The marginal clean coverage has an exact finite-sample identity. Let
\begin{equation*}
    Q_M^A(u) := \inf\{a\in\mathbb R : F_M^A(a)\ge u\}
\end{equation*}
denote the lower generalised quantile of the score law $M$.

\begin{theorem}[Exact finite-sample marginal coverage identity]
  \label{thm:fixed-threshold-trimmed-coverage}
  Under the setup of Section~\ref{subsec:trimmed-setup} with $\mu_{\rm keep} > 0$, let $r_n = \lceil(n+1)(1-\alpha)\rceil$. Almost surely in $\mathcal{G}$,
  \begin{equation}
    \label{eq:exact-clean-coverage-identity}
    \mathbb{P}\{Y_{\rm new} \in C_\alpha(X_{\rm new}) \mid \mathcal{G}\}
    \;=\;
    \sum_{n=0}^{m} \binom{m}{n}\mu_{\rm keep}^n(1-\mu_{\rm keep})^{m-n}\Psi_n(P, R),
  \end{equation}
  where $\Psi_n(P, R) = 1$ if $r_n = n+1$, and $\Psi_n(P, R) = \mathbb{E}\bigl[F_P^A\{Q_R^A(B_{r_n, n+1-r_n})\}\bigr]$ otherwise, with $B_{r_n, n+1-r_n} \sim \mathrm{Beta}(r_n, n+1-r_n)$. The same identity holds for any independent target $P_0$ in place of $P$. The lower generalised inverse convention does not require continuity of $F_R^A$; see Appendix~\ref{app:technical-remarks}, Lemma~\ref{lem:generalized-quantile-order},
for the atom-handling details.
\end{theorem}

The exact identity is unwieldy. Retaining only $d := d_{R \to P,+}^A$ gives the scalar lower bound
\begin{equation}
  L_{\rm fs}(d)
  \;:=\;
  \sum_{n=0}^{m} \binom{m}{n}\mu_{\rm keep}^n(1-\mu_{\rm keep})^{m-n}\,\psi_n(d),
  \qquad
  \psi_n(d) = \mathbb{E}\bigl[(B_{r_n, n+1-r_n} - d)_+\bigr]
  \label{eq:Lfs}
\end{equation}
($\psi_n(d) = 1$ when $r_n = n+1$), obtained by replacing $F_P^A \circ Q_R^A$ in \eqref{eq:exact-clean-coverage-identity} by $u \mapsto (u - d)_+$.

\begin{proposition}[Finite-sample scalar transfer]
  \label{prop:finite-sample-scalar-bound}
  Under the setup of Theorem~\ref{thm:fixed-threshold-trimmed-coverage}, almost surely in $\mathcal{G}$,
  \begin{equation}
    \label{eq:finite-sample-scalar-bound}
    \mathbb{P}\{Y_{\rm new} \in C_\alpha(X_{\rm new}) \mid \mathcal{G}\}
    \;\ge\;
    L_{\rm fs}(d_{R \to P,+}^A)
    \;\ge\;
    \bigl[\,1 - \alpha - d_{R \to P,+}^A\,\bigr]_+.
  \end{equation}
  In closed form, $\psi_n(d) = \tfrac{r_n}{n+1}\{1 - I_d(r_n+1, n+1-r_n)\} - d\{1 - I_d(r_n, n+1-r_n)\}$ when $r_n \le n$, with $I_x(a,b)$ the regularised incomplete beta function.
\end{proposition}

The scalar layer deliberately discards the joint structure of $F_R^A$ and $F_P^A$. The next theorem makes the sharpness claim precise. Here, ``best possible'' means minimax optimal among uniform lower bounds that use only the scalar discrepancy constraint.

\begin{theorem}[Finite-sample sharpness within the scalar information class]
  \label{thm:one-sided-transfer-sharpness}
  Fix $\alpha \in (0,1)$, $m \ge 0$, $\mu \in (0,1]$, and $d \in [0,1]$. For score laws $R^A,P^A$ on $\mathbb R$, define $\mathcal C_{m,\mu}(R^A,P^A)$ as the clean-target coverage of the abstract retained-law conformal experiment: $N\sim\mathrm{Binomial}(m,\mu)$; given $N=n$, $X_1,\ldots,X_n\overset{\mathrm{iid}}{\sim}R^A$ and independent $X_0\sim P^A$; cutoff $+\infty$ if $r_n=n+1$, otherwise $X_{(r_n)}$. Let $\mathfrak C_d := \{(R^A,P^A): \sup_t(F_R^A(t)-F_P^A(t))_+\le d\}$. Then
  \begin{equation}
    \label{eq:sharpness-minimax-form}
    \inf_{(R^A,P^A)\in\mathfrak C_d}
    \mathcal C_{m,\mu}(R^A,P^A)
    \;=\;
    L_{\rm fs}^{m,\mu}(d)
    \;:=\;
    \sum_{n=0}^{m}\binom{m}{n}\mu^n(1-\mu)^{m-n}\,\psi_n(d),
  \end{equation}
  attained by continuous $R_d^A,P_d^A$ with $\sup_t(F_{R_d}^A(t)-F_{P_d}^A(t))_+=d$. No uniform lower bound using only $m,\mu,\alpha$ and $d_{R\to P,+}^A\le d$ can exceed $L_{\rm fs}^{m,\mu}(d)$. As $m\to\infty$ with $\mu$ fixed, $L_{\rm fs}^{m,\mu}(d)\to[(1-\alpha)-d]_+$.
\end{theorem}

The extremal laws are constructed for the abstract scalar problem. Sharper bounds may be obtained from the mixture decomposition \eqref{eq:retained-law-decomposition}, smoothness of $F_R^A$, or audit information.

Substituting \eqref{eq:retained-law-decomposition} into the definition of $d_{R\to P,+}^A$ and applying the triangle inequality gives
\begin{equation*}
    d_{R \to P,+}^A \le (1 - \tilde\varepsilon_\star) \Delta_{\rm trim,+}^A + \tilde\varepsilon_\star D_{Q,+}.
\end{equation*}
This is an upper bound, not an exact decomposition. Combined with Proposition~\ref{prop:finite-sample-scalar-bound}, it yields
\begin{equation*}
    \mathbb{P}\{Y_{\rm new} \in C_\alpha(X_{\rm new}) \mid \mathcal{G}\}
    \ge
    \bigl[1 - \alpha - (1-\tilde\varepsilon_\star)\Delta_{\rm trim,+}^A - \tilde\varepsilon_\star D_{Q+}\bigr]_+.
\end{equation*}
The first term is the clean-side conditioning cost. The second is the retained-contamination cost. The next two subsections analyse these two terms in turn.

\subsection{Clean-side cost: a covariance identity}
\label{subsec:covariance-identity}

The clean trimming distortion captures how filtering perturbs the calibration distribution, even without contamination. Rewriting it as a covariance under $P$ makes the mechanism clear.

\begin{proposition}[Covariance identity for clean trimming distortion]
  \label{prop:trimming-bias-orthogonality}
  Let $K = \mathbf{1}\{S(Z) \le t^\star\}$ and assume $p_c > 0$. For every $t \in \mathbb{R}$,
  \begin{equation}
    F_{P_{\rm keep}}^A(t) - F_P^A(t) \;=\; \frac{\mathrm{Cov}_P\bigl(\mathbf{1}\{A(Z) \le t\},\,K\bigr)}{p_c},
    \label{eq:orthogonality-cov-pointwise}
  \end{equation}
  so $\Delta_{\rm trim,+}^A = p_c^{-1}\sup_t\{\mathrm{Cov}_P(\mathbf{1}\{A(Z) \le t\}, K)\}_+$, and $\Delta_{\rm trim}^A = 0$ whenever $A(Z) \perp\!\!\!\perp K$ under $P$.
\end{proposition}

% Read this as a design rule: a useful anomaly score should be \emph{contamination-selective across $P$ and $Q$}, but \emph{score-neutral within $P$}. A simple sufficient condition is covariate-only retention with a residual-based score in a homoscedastic regression model. In this case, $A(Z) = \sigma|\xi|$ is independent of any function of $X$, so $\Delta_{\rm trim}^A = 0$.

% Under misspecification $||f^\star - \hat f||_{\infty} \le \eta$, the same condition gives
% \begin{equation*}
%     \Delta_{\rm trim}^A \le \omega_{A_0}(\eta)
%     := \sup_t P\{A_0 \in [t-\eta, t+\eta]\},
% \end{equation*}
% where $A_0 = \sigma|\xi|$. In particular, $\Delta_{\rm trim}^A \le 2M\eta$ if $A_0$ has density bounded by $M$.

A useful anomaly score should be selective for contamination across $P$ and $Q$, while remaining neutral with respect to score within $P$. To make this condition precise, consider the clean homoscedastic regression model $Y=f^\star(X)+\sigma\xi, \xi\perp X,$ where $f^\star$ is the true regression function and $\widehat f$ is a fitted predictor. Let the nonconformity score be the absolute residual $A(X,Y)=|Y-\widehat f(X)|.$ Suppose the retention indicator $K$ depends only on $X$. If $\widehat f=f^\star$, then $A(X,Y)=A_0:=\sigma|\xi|,$ which is independent of $K$ under $P$. Consequently, $\Delta_{\mathrm{trim}}^A=0$. More generally, suppose that $\|f^\star-\widehat f\|_\infty\leq\eta.$ The same argument gives

\begin{equation*}
    \Delta_{\mathrm{trim}}^A
    \leq
    \omega_{A_0}(\eta) := \sup_t P\{A_0\in[t-\eta,t+\eta]\}.
\end{equation*}
In particular, if $A_0$ has a density bounded above by $M$, then $\Delta_{\mathrm{trim}}^A\leq 2M\eta.$

For regularised least-squares backbones, standard learning-theory controls can provide such prediction-error inputs under additional assumptions \citep{caponnetto2007optimal}. Our diagnostic only needs the realised bound. Response-dependent anomaly scores can make $\Delta_{\rm trim,+}^A$ arbitrarily large, even without contamination. This motivates our use of covariate-geometric scores in the operational examples.

\subsection{Retained mixture coefficient and retention-probability separation}
\label{subsec:separation-dichotomy}

The retained-contamination cost $\tilde\varepsilon_\star D_{Q,+}$ is governed by $\tilde\varepsilon_\star$, which depends only on the retention probabilities. If the retained dirty score law is harmless, so that $D_{Q,+}=0$, the dirty contribution vanishes even when $\tilde\varepsilon_\star$ is large.

\begin{proposition}[Coefficient-level separation dichotomy]
  \label{thm:separation-dichotomy-main}
  Under the setup of Section~\ref{subsec:trimmed-setup} with $p_c > 0$:
  \begin{itemize}[leftmargin=2em, itemsep=0.3ex]
    \item If $p_d \le \lambda p_c$, then $\tilde\varepsilon_\star \le \varepsilon\lambda / (1 - \varepsilon + \varepsilon\lambda)$; in particular, $p_d/p_c \to 0$ forces $\tilde\varepsilon_\star \to 0$.
    \item If $p_d = p_c$, then $\tilde\varepsilon_\star = \varepsilon$.
  \end{itemize}
\end{proposition}

When the score induces \emph{retention-probability separation}, dirty points are retained much less often than clean ones. Then $\tilde\varepsilon_\star$ can be far below $\varepsilon$, and trimming can help. Without separation, $\tilde\varepsilon_\star = \varepsilon$ for any threshold. Trimming then cannot reduce contamination through the mixture coefficient. Improvement can only come from a favourable $Q_{\rm keep}$. Yet trimming still incurs the clean covariance cost in Proposition~\ref{prop:trimming-bias-orthogonality}, so it can strictly worsen the diagnostic lower bound.

We call this the \emph{coefficient-level separation dichotomy}. Section~\ref{sec:experiments} organises the experiments around it. The proposition rules out improvement through the retained-contamination \emph{coefficient}, not through all possible differences between $Q_{\rm keep}$ and $Q$. A stronger statement appears in Theorem~\ref{thm:no-separation-fixed}.

\subsection{Selection-aware bounds and certificate templates}
\label{subsec:operational}

% Two extensions are needed in practice: same-sample threshold selection and numerical coverage certificates.

\paragraph{Practical choice of the anomaly score and threshold.} The anomaly score should be chosen according to the expected contamination mechanism. Covariate-based scores can be useful when contamination changes the distribution of $X$. They have no direct access to label-only corruption when $Q_X=P_X$.

For a candidate score $S$ and threshold $t$, two quantities are especially important. The first is the clean trimming distortion $\Delta_{\mathrm{trim},+}^A(t)$. The second is the retained contamination coefficient

\begin{equation*}
    \widetilde\varepsilon(t)
    =
    \frac{\varepsilon p_d(t)}
    {(1-\varepsilon)p_c(t)+\varepsilon p_d(t)}.
\end{equation*}

A favourable rule should keep the clean score distribution nearly unchanged while making the ratio $p_d(t)/p_c(t)$ small.

When independent tuning information is available, a practical strategy is to consider a finite library of candidate scores and thresholds. Candidates can be screened using independently justified upper bounds on the population coverage loss. Among candidates that pass the coverage screen, one may choose the rule with the smallest estimated prediction-set size.

If no candidate passes the screen, the procedure returns to a calibration-independent fallback, such as no trimming. This avoids an unsupported trimming decision. However, the fallback does not automatically guarantee nominal clean-target coverage under contamination.

Theorem~\ref{thm:independent-adaptive-threshold} justifies independent tuning without an additional selection penalty. Proposition~\ref{prop:same-sample-threshold-selection} provides a separate bound for same-sample selection from a fixed finite threshold grid. The finite-library extension is given in Appendix~\ref{app:same-sample-selection}.

\paragraph{Same-sample threshold selection.}
Same-sample selection breaks the retained-law rank identity. What remains is a selection-uniform empirical-CDF transfer bound. For a finite grid $\mathcal{T} = \{t_1, \ldots, t_K\}$, define $N_t$, $R_t$, $\hat\tau_t$, and $d_t = d_{R_t \to P,+}^A$ analogously. For $N_t\ge1$, let
\begin{equation*}
    \eta_t = \sqrt{\log(2K/\beta) / (2N_t)}.
\end{equation*}
Then, with conditional probability at least $1-\beta$, the following holds simultaneously for all $t \in \mathcal{T}$ with $N_t \ge 1$ and $r_t \le N_t$:
\begin{equation*}
    \mathbb{P}_P\{A(Z_{\rm new}) \le \hat\tau_t \mid Z_1, \ldots, Z_m, \mathcal{G}\}
    \ge
    \frac{r_t}{N_t} - d_t - \eta_t.
\end{equation*}

Proposition~\ref{prop:same-sample-threshold-selection} in Appendix~\ref{app:proofs} gives the precise statement. The bound holds uniformly over the grid, allowing any data-dependent $\widehat{t} \in \mathcal{T}$ with a $\sqrt{\log K / N_t}$ penalty. It is data-driven only through the empirical-CDF term. It becomes a numerical certificate only after the population losses $d_t$ are independently bounded. This view is related to conformal selection and adaptive novelty detection \citep{jin2023selection, marandon2024adaptive}, but here the selected object is the calibration subset used to form the conformal quantile. The same argument extends to a finite library of score-threshold pairs fixed before the final calibration sample is observed. The resulting bound incurs a $\sqrt{\log H/N_h}$ term. This extension permits data-dependent selection from a prespecified library, but does not cover unrestricted fitting of a new anomaly score on the final calibration sample. See Corollary~\ref{cor:finite_library} in Appendix~\ref{app:same-sample-selection}.

% \paragraph{Certificate templates.}
% The diagnostics in Sections~\ref{subsec:exact-scalar-mixture}--\ref{subsec:separation-dichotomy} become numerical certificates only when their ingredients are bounded by independent information. There are two routes.

\paragraph{Procedure, diagnostic, and certificate.} We distinguish three objects. Algorithm~\ref{alg:fixed-threshold-trimmed-cp} constructs a prediction set using the nonconformity score $A$, the anomaly score $S$, a fixed threshold $t^\star$, and the contaminated calibration sample. It does not require knowledge of $P$, $Q$, the contamination level, or the latent contamination labels. Theorem~\ref{thm:fixed-threshold-trimmed-coverage} provides an exact population characterisation of clean target coverage. The component-wise diagnostic further explains the roles of clean trimming distortion and retained contamination. These population quantities describe the behaviour of a proposed trimming rule. They are not inputs to Algorithm~\ref{alg:fixed-threshold-trimmed-cp}. Numerical coverage certificates require additional information. Proposition~\ref{prop:operational-certificate} converts independently justified component bounds into a clean coverage guarantee. Proposition~\ref{prop:exact-audit-certificate} instead uses an independent, cleanly labelled audit sample to certify the final prediction set.

Under unrestricted Huber contamination, clean target coverage cannot generally be identified from contaminated observations alone. Thus, the additional information required for numerical certification reflects an information limit of the model, rather than a missing step in Algorithm~\ref{alg:fixed-threshold-trimmed-cp}.

We consider two certificate routes. The first audits the mixture components. Given $L_c \le p_c$, $p_d \le U_d$, $\Delta_{\rm trim,+}^A \le B_{\Delta,+}$, $D_{Q,+} \le B_{Q,+}$, and $\varepsilon \le \varepsilon_{\max}$,
\begin{equation*}
    \mathbb{P}\{Y_{\rm new} \in C_\alpha(X_{\rm new}) \mid \mathcal{G}\}
    \ge
    \bigl[1 - \alpha - B_{\Delta,+} - \bar\varepsilon\max\{B_{Q,+} - B_{\Delta,+}, 0\} \bigr]_+,
\end{equation*}
where
\begin{equation*}
    \bar\varepsilon =
    \frac{\varepsilon_{\max} U_d}
    {(1-\varepsilon_{\max})L_c + \varepsilon_{\max}U_d}.
\end{equation*}
The worst-case bound $B_{Q,+}=1$ is always valid, but conservative.

The second route audits the final selected set directly, using an independent clean audit sample. Exact one-sided binomial inversion certifies the realised coverage. A one-sided Kolmogorov audit certifies any $\mathcal{A}$-measurable cutoff, and applies to any selected-subset procedure. Formal statements are given in Propositions~\ref{prop:operational-certificate} and~\ref{prop:exact-audit-certificate} in Appendix~\ref{app:proofs}.

The diagnostics should be read as \emph{simulation-evaluable diagnostics} when $P$ and $Q$ are known. Otherwise, they should be read as \emph{certificate templates}. Numerical certification requires either independent component bounds or an independent audit sample. We maintain this distinction throughout Section~\ref{sec:experiments}.

\section{Experimental Results}
\label{sec:experiments}
\subsection{Setup and reporting conventions}
\label{subsec:experiments-setup}

% We evaluate the retained-law analysis on a heteroscedastic regression task with score-visible covariate contamination. We also include stress regimes that probe the failure modes predicted in Section~\ref{sec:method}. 

We evaluate the retained law analysis through controlled regression experiments and semi-synthetic classification studies. The regression experiments isolate the two mechanisms in Section~\ref{sec:method}, clean trimming distortion and retained contamination. The classification experiments use real US Census covariates and labels with controlled cross state contamination. We also study corruption that affects labels only, the allocation of a limited clean data budget, and the agreement between theoretical diagnostics and empirical performance. These experiments include both favourable and unfavourable regimes for trimming. The clean design is
\begin{equation*}
    X \sim N(0,1), \qquad
    Y = X + 0.6 (1 + 0.6 |X|) \xi,
    \qquad \xi \sim N(0,1),
\end{equation*}
and the contaminating distribution has
\begin{equation*}
    X_{\rm dirty} \sim N(6,1),
    \qquad
    Y_{\rm dirty} = X_{\rm dirty} + 0.05 \xi.
\end{equation*}
The contamination level is $\varepsilon = 0.2$. The contaminating responses are close to the regression surface. However, the contaminating covariates are clearly separated from the clean support. Thus, a covariate-geometric anomaly score should be able to detect them. Importantly, the clean design is heteroscedastic. Therefore, the homoscedastic-residual orthogonality in Proposition~\ref{prop:trimming-bias-orthogonality} does \emph{not} apply, and nonzero clean trimming distortion is expected.

% The nominal miscoverage level is $\alpha = 0.1$ throughout, so the clean-coverage target is $0.9$. Unless stated otherwise, the calibration size is $m = 320$, and all summaries use $100$ Monte Carlo repetitions.
The nominal miscoverage level is $\alpha=0.1$ throughout, so the clean coverage target is $0.9$. Unless stated otherwise, the synthetic regression experiments use $m=320$ calibration observations and $100$ Monte Carlo repetitions. The ACS classification and clean budget experiments use 50 repetitions, as specified in their experimental settings. We report empirical clean coverage and prediction set width or size separately from population diagnostics and independently justified certificate bounds. The population quantities are evaluated using the known simulation mechanisms. They explain the observed behaviour but are not themselves data only certificates. The prediction rule and the Stein score-norm anomaly score are both fitted on independent clean auxiliary data. Thresholds are either fixed population thresholds or chosen on an independent reference split. No main-table row selects the threshold from the calibration sample used for the conformal quantile.

We report three types of quantities. \emph{Empirical} clean coverage and width describe finite-sample performance, with Monte Carlo intervals across repetitions. \emph{Population diagnostics}—$p_c$, $p_d$, $\tilde\varepsilon_\star$, $\Delta_{\rm trim,+}^A$, and $D_{Q,+}$—are evaluated from the known simulation laws. They are mechanistic explanations, not certificates. \emph{Conservative lower bounds} (“Cons.\ LB”) are componentwise audit-style certificates. They use independent diagnostic samples and concentration corrections, so they are operationally safe but deliberately loose.

When $D_{Q,+}$ is available from the simulation, we also report the population diagnostic
\begin{equation*}
    L_{\rm mix,+}
    =
    1 - \alpha
    - (1 - \tilde\varepsilon_\star)\Delta_{\rm trim,+}^A
    - \tilde\varepsilon_\star D_{Q,+}.
\end{equation*}
  This is the lower bound induced by the mixture upper bound on the scalar transfer loss. When the retained contaminating component is too rare to estimate reliably, we use the worst-case fallback $D_{Q,+} \le 1$.

\subsection{Score-visible separation and the main trade-off}
\label{subsec:experiment-tradeoff}
Appendix~\ref{app:additional covariance and scalar-bound checks} reports additional checks. As the Stein threshold rises from $q = 0.950$ to $q = 0.990$, the clean side covariance envelope falls from $0.0130$ to $0.0036$. It also reports the finite sample scalar bound curve, and Theorem~\ref{thm:one-sided-transfer-sharpness} establishes its sharpness.

Combining the two mechanisms, clean-side covariance and retained contamination, Table~\ref{tab:main-tradeoff} reports the main score-visible trade-off. Ordinary split conformal under-covers the clean target, with mean coverage $0.8709$. All three Stein thresholds make dirty retention negligible, so $\tilde\varepsilon_\star$ is essentially zero in the trimmed rows. Coverage improves as $q$ increases. At $q=0.990$, mean coverage is $0.8950$ (unrounded $0.89499$), still below the nominal target $0.9$, but close to the clean-calibration benchmark $0.8984$. By Table~\ref{tab:covariance-identity-summary}, the improvement comes from reduced clean covariance distortion, not stronger contamination removal. The retained-contamination axis was already saturated at $q=0.950$.

\begin{table}[t]
  \centering
  \scriptsize
  \setlength{\tabcolsep}{2.1pt}
  \resizebox{\linewidth}{!}{%
  \begin{tabular}{lcccccccc}
    \toprule
    Method & Clean coverage & Width & $p_c$ & $p_d$ & $\tilde\varepsilon_\star$ & $\Delta_{\rm trim,+}^A$ & $L_{\rm mix,+}$ & Cons.\ LB \\
    \midrule
    Ordinary split & $0.8709\,[0.8667, 0.8750]$ & $2.7308$ & $1.0000$ & $1.0000$ & $0.2000$ & $0.0000$ & $\ge 0.7000$ & $0.5499$ \\
    Stein $q = 0.950$ & $0.8857\,[0.8819, 0.8896]$ & $2.8584$ & $0.9494$ & $2.8 \times 10^{-5}$ & $7.0 \times 10^{-6}$ & $0.0137$ & $\ge 0.8863$ & $0.7140$ \\
    Stein $q = 0.975$ & $0.8914\,[0.8877, 0.8951]$ & $2.9107$ & $0.9752$ & $6.1 \times 10^{-5}$ & $1.6 \times 10^{-5}$ & $0.0075$ & $\ge 0.8925$ & $0.7225$ \\
    Stein $q = 0.990$ & $0.8950\,[0.8913, 0.8986]$ & $2.9448$ & $0.9897$ & $2.6 \times 10^{-4}$ & $6.4 \times 10^{-5}$ & $0.0036$ & $0.8964$ & $0.7276$ \\
    Clean oracle & $0.8984\,[0.8948, 0.9021]$ & $2.9791$ & $1.0000$ & $0.0000$ & $0.0000$ & $0.0000$ & $0.9000$ & $0.7322$ \\
    \bottomrule
  \end{tabular}}
  \caption{Score-visible regime with $\varepsilon = 0.2$, $m = 320$, and $100$ repetitions. $L_{\rm mix,+}$ is the population diagnostic lower bound. Entries marked $\ge$ use the worst-case fallback $D_{Q,+} \le 1$. “Cons.\ LB” is a separate componentwise audit-style certificate. Trimming improves coverage from $0.8709$ to $0.8950$. The unrounded mean for $q=0.990$ is $0.89499$, so this row should be read as near-nominal and within Monte Carlo error of the clean oracle $0.8984$, not as a formal $0.9$ certificate. Width intervals are omitted for compactness; full intervals are given in Appendix~\ref{app:detailed-experiments}.}
  \label{tab:main-tradeoff}
\end{table}

The $q = 0.990$ row also shows how the componentwise diagnostic separates the two mechanisms. The retained dirty discrepancy is $D_{Q,+} = 0.31$, but the retained mixture coefficient is only $\tilde\varepsilon_\star = 6.4 \times 10^{-5}$. Thus the dirty contribution is $\tilde\varepsilon_\star D_{Q,+} \approx 2.0 \times 10^{-5}$, which is dominated by the clean contribution $(1 - \tilde\varepsilon_\star)\Delta_{\rm trim,+}^A \approx 0.0036.$ The resulting population diagnostic is $L_{\rm mix,+} = 0.8964$, close to the empirical mean $0.8950$ within Monte Carlo error. The retained-law mixture upper bound therefore explains where the residual under-coverage comes from.

\subsection{The dichotomy in failure modes}
\label{subsec:experiment-failures}

Proposition~\ref{thm:separation-dichotomy-main} predicts that, without retention-probability separation, trimming cannot reduce $\tilde\varepsilon_\star$ through the mixture coefficient. Table~\ref{tab:regime-diagnostics} examines four regimes that stress this dichotomy from different sides.

\begin{table}[t]
  \centering
  \scriptsize
  \setlength{\tabcolsep}{2.4pt}
  \resizebox{\linewidth}{!}{%
  \begin{tabular}{lccccccccc}
    \toprule
    Regime & Method & Clean coverage & $p_c$ & $p_d$ & $\tilde\varepsilon_\star$ & $\Delta_{+}$ & $D_{Q,+}$ & $L_{\rm mix,+}$ & Cons.\ LB$_+$ \\
    \midrule
    Score-visible $(\varepsilon = 0.2)$ & Stein $q = 0.990$ & $0.8950$ & $0.9897$ & $2.6 \times 10^{-4}$ & $6.4 \times 10^{-5}$ & $0.0036$ & $0.31$ & $0.8964$ & $0.7277$ \\
    Perfect dirty rejection & Stein clean-ref $q = 0.950$ & $0.9015$ & $0.9444$ & $0.0000$ & $0.0000$ & $0.0040$ & N/A & $0.8960$ & $0.7136$ \\
    No separation, $m = 800$ & Stein clean-ref $q = 0.950$ & $0.8754$ & $0.9420$ & $0.9429$ & $0.2002$ & $0.0037$ & $0.94$ & $0.7081$ & $0.6028$ \\
    Label-only, $\varepsilon = 0.20$ & Ordinary split & $0.9998$ & $1.0000$ & $1.0000$ & $0.2000$ & $0.0000$ & $8.7 \times 10^{-5}$ & $0.9000$ & $0.7631$ \\
    Large $m = 10{,}000$ & Stein $q = 0.950$ & $0.8888$ & $0.9495$ & $0.0057$ & $0.0015$ & $0.0134$ & $0.86$ & $0.8853$ & $0.8464$ \\
    \bottomrule
  \end{tabular}}
  \caption{Diagnostic regimes. The score-visible row is the main result of Section~\ref{subsec:experiment-tradeoff}. The no-separation row tests Proposition~\ref{thm:separation-dichotomy-main}. The label-only and large-$m$ rows show distinct failure modes.}
  \label{tab:regime-diagnostics}
\end{table}

\paragraph{Perfect dirty rejection.}
Placing the contamination cluster far from the clean support drives $p_d$ to zero. The remaining loss is pure clean trimming distortion, with $\Delta_{\rm trim,+}^A = 0.0040$. Coverage reaches $0.9015$, slightly above nominal. This is because the nonzero clean retention probability, $p_c = 0.94$, does not perturb the conformal quantile in the under-coverage direction.

\paragraph{No separation.}
When clean and dirty points are retained with similar probabilities ($p_c = 0.942$, $p_d = 0.943$), Proposition~\ref{thm:separation-dichotomy-main} predicts $\tilde\varepsilon_\star \approx \varepsilon$. Empirically, $\tilde\varepsilon_\star = 0.200$, and the dirty contribution dominates, giving coverage $0.8754$. Trimming does not help. Although the componentwise diagnostic is conservative because $D_{Q,+}$ is large, $\tilde\varepsilon_\star\approx\varepsilon$ confirms the no-separation mechanism. A separation sweep with $\varepsilon = 0.3$ and $m = 800$ shows the transition. As separation increases, $p_d$ drops from $1.00$ to $0.00$, $\tilde\varepsilon_\star$ drops from $0.300$ to $0.000$, and coverage rises from $0.859$ to $0.902$. The improvement comes from retention separation, not trimming \emph{per se}.

\paragraph{Label-only contamination.}
When only the response is corrupted ($Q_X = P_X$), a covariate-geometric score has no leverage. Ordinary split conformal over-covers ($0.9998$) by producing inflated intervals. Trimming with a covariate score cannot improve this, because the corrupted points are not detectable from $X$ alone. This is a genuine limitation of covariate trimming, not a failure of the retained-law analysis.

\paragraph{Large calibration size.}
% Increasing $m$ to $10{,}000$ does not raise coverage above $L_{\rm mix,+}=0.8853$, even though the conservative audit lower bound increases substantially, from $0.7132$ to $0.8464$. The limiting coverage is governed by retained-to-clean score-law transfer, not sample-size slack. Finite-sample concentration shrinks, but the population-level mixture loss remains.
In the separate large calibration experiment, the empirical clean coverage is $0.8888$ at $m=10{,}000$, while the population diagnostic lower bound is $L_{\mathrm{mix},+}=0.8853$. The empirical coverage lies above this lower bound, as expected. As the calibration size increases from $320$ to $10{,}000$, the conservative audit style lower bound rises from $0.7132$ to $0.8464$. In contrast, the population mixture diagnostic remains approximately unchanged within this experiment. These results show that increasing the calibration size reduces finite sample uncertainty in the certificate but does not by itself, remove the retained to clean score law discrepancy. The population diagnostic describes a potential coverage loss that persists beyond finite sample calibration effects. Further details are given in Appendix~\ref{app:large calibration behaviour}.

\paragraph{Threshold selection and audit.}
When the threshold is selected from the same calibration sample over a finite grid, Proposition~\ref{prop:same-sample-threshold-selection} applies with a $\sqrt{\log K / N}$ penalty. When independent clean labelled audit data are available, Proposition~\ref{prop:exact-audit-certificate} certifies the final selected set directly by binomial inversion, regardless of how the set was constructed. Across the tested policies, fixed and clean-reference $q=0.950$ rules give coverage near $0.892$. A more conservative KSD-geometric policy gives coverage $0.898$ with larger width. Thus, the policy choice changes the coverage-efficiency frontier, but not the mechanism. Non-Gaussian geometries are reported in Appendix~\ref{app:detailed-experiments}. The qualitative conclusion is that the anomaly score should match the clean support geometry while remaining score-neutral within $P$, rather than being a universal choice.

%%%%%%%%%%%%%%%%%%%%%%%%%%%%%%%%%%%%%%%%%%%%%%%%%%%%%%%%%%%%%%%%%%%%%%%%%%%%%%%%%%%%%%%%%%%%%%%%%%%%%%%%%%%%%%%%%%%%%%%%%%%%%%%%%%%%%%%%%%%%%%%%%%%%%%%%%%

\subsection{Semi-synthetic ACS classification}

We also evaluate the diagnostic on the 2018 American Community Survey (ACS) ACSEmployment classification task using Folktables \citep{ding2021retiring}. California (CA) defines the clean target distribution $P$. Texas (TX), New York (NY), Florida (FL), and their equal mixture define alternative contaminating distributions $Q$. We consider contamination levels $\varepsilon\in\{0.1,0.2,0.3\}$. All methods use the same logistic regression predictor and randomised adaptive prediction sets (APS) score. We compare ordinary APS, direct clean calibration, source oracle trimming, and three implementable covariate anomaly scores: Isolation Forest, Mahalanobis, and GMM Mahalanobis. Source labels are used only to construct the contamination and evaluate the oracle. They are not supplied to the implementable trimming rules. Table~\ref{tab:acs_classification} reports the results for the heterogeneous mixture. Across the three implementable anomaly scores, clean retention probabilities range from $0.9490$ to $0.9502$, while dirty retention probabilities range from $0.9477$ to $0.9527$. The retained contamination coefficient therefore remains close to the original contamination level. All practical trimming rules achieve coverage close to ordinary APS. Across all 12 combinations of contamination source and level, the largest absolute paired mean coverage change is only $0.0005$. This is an informative negative result. Ordinary APS already achieves near nominal clean coverage in this setting. Moreover, common covariate anomaly scores do not meaningfully separate the clean and contaminating domains. These findings support the coefficient level prediction of Proposition~\ref{thm:separation-dichotomy-main}. Trimming cannot substantially reduce the retained contamination coefficient when $p_c\approx p_d$. The experiment also illustrates an important limitation, successful trimming in a controlled, score visible setting does not imply similar benefits under realistic domain overlap.

\begin{table}[t]
    \centering
    \small
    \setlength{\tabcolsep}{4pt}
    \resizebox{\linewidth}{!}{%
    \begin{tabular}{ccccc}
        \toprule
        $\varepsilon$
        & Ordinary APS coverage / size
        & Trimming coverage range
        & Trimming size range
        & $\tilde{\varepsilon}$ range \\
        \midrule
        $0.1$ & $0.9004$ / $1.4432$
              & $0.9003$--$0.9006$
              & $1.4430$--$1.4438$
              & $0.0998$--$0.1003$ \\
        $0.2$ & $0.9006$ / $1.4438$
              & $0.9003$--$0.9004$
              & $1.4429$--$1.4432$
              & $0.1997$--$0.2006$ \\
        $0.3$ & $0.8992$ / $1.4393$
              & $0.8989$--$0.8992$
              & $1.4384$--$1.4393$
              & $0.2996$--$0.3008$ \\
        \bottomrule
    \end{tabular}%
    }
    \caption{Semi-synthetic ACS classification with California as the clean
    target and an equal mixture of Texas, New York, and Florida as
    contamination. Results are based on 50 Monte Carlo repetitions.
    Trimming ranges cover Isolation Forest, Mahalanobis, and
    GMM-Mahalanobis. Similar clean and dirty retention probabilities
    lead to little change in retained contamination or empirical coverage.}
    \label{tab:acs_classification}
\end{table}

\section{Discussion}
\label{sec:discussion}
The retained-law viewpoint treats trimming as a transformation of the calibration law. Fixed-threshold trimming does not recover $P$, but induces a retained law $R$, under which conformal validity is exact. The cost is described by a one-dimensional score-CDF transfer. Trimming helps when the anomaly score separates retention under $P$ and $Q$, while remaining nearly score-neutral within $P$. Otherwise, contamination is not reduced through the retained mixture coefficient, and any gain must come from the retained dirty score law. The diagnostic $L_{\rm mix,+}$ captures this trade-off. Numerical certification requires componentwise bounds (Proposition~\ref{prop:operational-certificate}) or direct audit of the selected set (Proposition~\ref{prop:exact-audit-certificate}). The threshold is fixed before the final contaminated calibration sample is used. Same-sample selection is handled only through a finite-grid empirical-CDF bound. The guarantees are marginal, not subgroup-conditional. Sharpness holds only within the scalar-discrepancy information class. Sharper bounds may exploit structure in $P_{\rm keep}$ or $Q_{\rm keep}$, smoothness, or audit data.

\begin{ack}
    The author gratefully acknowledges Research England E3 for funding the ``MARS: Mathematical for AI in Real-world Systems'' programme at Lancaster University. The author was supported by the China Scholarship Council under Grant Number No. 202208890004. The author is grateful to Chris Oates for helpful discussions and feedback on this work.
\end{ack}

\bibliographystyle{plainnat}
\bibliography{reference}

\clearpage
\appendix
\renewcommand{\theHfigure}{app.\arabic{figure}}
\renewcommand{\theHtable}{app.\arabic{table}}

\section{Appendix roadmap}
Appendix~\ref{app:algorithm} gives the fixed-threshold trimmed split-conformal algorithm.
Appendix~\ref{app:supplementary-generic-theory} collects supplementary generic theory, including independent threshold tuning, upper-coverage diagnostics, efficiency bounds, and marginalisation of high-probability conditional guarantees.
Appendix~\ref{app:technical-remarks} records quantile and edge-case conventions used in the finite-sample identity.
Appendix~\ref{app:auxiliary-diagnostic-lemmas}  contains auxiliary diagnostic lemmas used in the retained-law and certificate analyses.
Appendix~\ref{app:proofs} gives the proofs of the main results, including retained-law transfer, the exact finite-sample coverage identity, finite-sample scalar sharpness, same-sample threshold selection, and audit certificates.
Appendix~\ref{app:benchmark-regimes} presents benchmark regimes illustrating ideal separation and no-separation behaviour.
Appendix~\ref{app:detailed-experiments} reports detailed experimental settings and additional simulations.
Appendix~\ref{app:experimental-diagnostics} collects additional reporting conventions, diagnostic table variants, and guidance on interpreting empirical performance, population diagnostics, and conservative certificate bounds.

\section{Algorithm}
\label{app:algorithm}
For completeness, Algorithm~\ref{alg:fixed-threshold-trimmed-cp} gives the fixed-threshold trimmed split-conformal procedure analysed in Section~\ref{subsec:trimmed-setup}. The construction is standard split conformal prediction \citep{vovk2005algorithmic, lei2018distributionfree} applied to the retained calibration subset
\begin{equation*}
    \{Z_i : S(Z_i) \le t^\star\}.
\end{equation*}
We use the conservative convention $\hat\tau_\alpha = +\infty$ when fewer than $\lceil 1/\alpha \rceil - 1$ points are retained, or when the conformal rank exceeds the retained sample size.

\begin{algorithm}[H]
  \caption{Fixed-threshold trimmed split conformal prediction}
  \label{alg:fixed-threshold-trimmed-cp}
  \begin{algorithmic}[1]
    \Require Nonconformity score $A$, anomaly score $S$, calibration-independent threshold $t^\star$, contaminated calibration sample $\{Z_i=(X_i,Y_i)\}_{i=1}^m$, miscoverage level $\alpha\in(0,1)$.
    \Ensure Prediction set $C_\alpha(x)$.
    \For{$i=1,\dots,m$}
      \State Compute $A_i:=A(Z_i)$ and $S_i:=S(Z_i)$.
      \State Set $K_i:=\mathbf 1\{S_i\le t^\star\}$.
    \EndFor
    \State Let $\mathcal I_{\mathrm{keep}}:=\{i:K_i=1\}$ and $N_{\mathrm{keep}}:=|\mathcal I_{\mathrm{keep}}|$.
    \If{$N_{\mathrm{keep}}=0$}
      \State Set $\hat\tau_\alpha:=+\infty$.
    \Else
      \State Set $r_{\mathrm{keep}}:=\lceil (N_{\mathrm{keep}}+1)(1-\alpha)\rceil$.
      \If{$r_{\mathrm{keep}}=N_{\mathrm{keep}}+1$}
        \State Set $\hat\tau_\alpha:=+\infty$.
      \Else
        \State Set $\hat\tau_\alpha$ to the $r_{\mathrm{keep}}$-th order statistic of $\{A_i:i\in\mathcal I_{\mathrm{keep}}\}$.
      \EndIf
    \EndIf
    \State Return $C_\alpha(x):=\{y\in\mathcal Y:A(x,y)\le\hat\tau_\alpha\}$.
  \end{algorithmic}
\end{algorithm}

\section{Supplementary generic theory}
\label{app:supplementary-generic-theory}

This appendix collects generic statements for implementation, diagnostics, and threshold tuning. They are not needed for the main retained-law transfer argument. Section~\ref{sec:method} contains the fixed-threshold coverage theorem, the finite-sample scalar sharpness result, the same-sample finite-grid selection bound, and the audit certificates. This appendix keeps independent-tuning and auxiliary diagnostic material separate.

\subsection{Unified notation for score-CDF gaps and mixture diagnostics}
\label{app:unified-notation}

Unless stated otherwise, all quantities in this table are evaluated conditionally on the auxiliary sigma-field. Under this conditioning, $A$, $S$, and the threshold are fixed. Suprema are taken over $a\in\mathbb{R}$, and $F_M^A(a):=M\{A(Z)\le a\}$ denotes the score CDF under law $M$.

For the fixed threshold $t^\star$, write
\begin{equation*}
    K=\mathbf 1\{S(Z)\le t^\star\},\qquad
    p_c=P(K=1),\qquad
    p_d=Q(K=1),
\end{equation*}
and
\begin{equation*}
\mu_{\rm keep}=(1-\varepsilon)p_c+\varepsilon p_d.
\end{equation*}
When $p_d=0$, $Q_{\rm keep}$ may be chosen arbitrarily, and all zero-weighted dirty terms are interpreted as zero. Componentwise clean-side quantities involving $P_{\rm keep}$ require $p_c>0$.

\begin{table}[h]
  \centering
  \scriptsize
  \setlength{\tabcolsep}{3pt}
  \renewcommand{\arraystretch}{1.35}
  \begin{tabular}{p{0.23\linewidth}p{0.51\linewidth}p{0.19\linewidth}}
    \toprule
    Symbol & Definition & Role / condition \\
    \midrule
    $Q_M^A(u)$ & $\inf\{a\in\mathbb R:F_M^A(a)\ge u\}$, $u\in(0,1)$ & Lower generalised score quantile \\
    $P_{\rm keep}$, $Q_{\rm keep}$ & $P(\,\cdot\mid S\le t^\star)$, $Q(\,\cdot\mid S\le t^\star)$ & Defined when the corresponding retention probability is positive \\
    $R$ & $P^\star(\,\cdot\mid S\le t^\star)$ & Retained calibration law, requires $\mu_{\rm keep}>0$ \\
    $\tilde\varepsilon_\star$ & $\displaystyle \frac{\varepsilon p_d}{(1-\varepsilon)p_c+\varepsilon p_d}$ & Retained mixture coefficient \\
    $d_{M\to N,+}^A$ & $\displaystyle \sup_a\bigl(F_M^A(a)-F_N^A(a)\bigr)_+$ & One-sided score-CDF gap from $M$ to $N$ \\
    $d_{R\to P,+}^A$ & $\displaystyle \sup_a\bigl(F_R^A(a)-F_P^A(a)\bigr)_+$ & Scalar loss controlling clean undercoverage \\
    $d_{P\to R,+}^A$ & $\displaystyle \sup_a\bigl(F_P^A(a)-F_R^A(a)\bigr)_+$ & Mirror loss for upper-coverage diagnostics \\
    $\Delta_{\rm trim,+}^A$ & $\displaystyle d_{P_{\rm keep}\to P,+}^A$ & Clean-side trimming distortion, requires $p_c>0$ \\
    $\Delta_{\rm trim,-}^A$ & $\displaystyle d_{P\to P_{\rm keep},+}^A$ & Mirror clean-side distortion \\
    $\Delta_{\rm trim}^A$ & $\displaystyle \sup_a|F_{P_{\rm keep}}^A(a)-F_P^A(a)|=\max\{\Delta_{\rm trim,+}^A,\Delta_{\rm trim,-}^A\}$ & Two-sided clean distortion \\
    $D_{Q,+}$ & $\displaystyle d_{Q_{\rm keep}\to P,+}^A$ & Retained dirty discrepancy relevant to undercoverage \\
    $D_{Q,-}$ & $\displaystyle d_{P\to Q_{\rm keep},+}^A$ & Mirror retained dirty discrepancy \\
    $D_Q$ & $\displaystyle \sup_a|F_{Q_{\rm keep}}^A(a)-F_P^A(a)|=\max\{D_{Q,+},D_{Q,-}\}$ & Two-sided retained dirty discrepancy \\
    $\delta_{\rm mix,+}^{\rm ref}$ & $\displaystyle (1-\tilde\varepsilon_\star)\Delta_{\rm trim,+}^A+\tilde\varepsilon_\star D_{Q,+}$ & Componentwise upper bound on $d_{R\to P,+}^A$ \\
    $\delta_{\rm mix,-}^{\rm ref}$ & $\displaystyle (1-\tilde\varepsilon_\star)\Delta_{\rm trim,-}^A+\tilde\varepsilon_\star D_{Q,-}$ & Componentwise upper bound on $d_{P\to R,+}^A$ \\
    $\delta_{\rm mix}^{\rm ref}(D_Q)$ & $\displaystyle (1-\tilde\varepsilon_\star)\Delta_{\rm trim}^A+\tilde\varepsilon_\star D_Q$ & Symmetric fallback diagnostic \\
    $\delta_{\rm mix}^{\rm wc}$ & $\displaystyle (1-\tilde\varepsilon_\star)\Delta_{\rm trim}^A+\tilde\varepsilon_\star$ & Worst-case dirty-side fallback, using $D_Q\le1$ \\
    $L_{\rm mix,+}$ & $\displaystyle [1-\alpha-\delta_{\rm mix,+}^{\rm ref}]_+$ unless a finite-sample $L_{\rm fs}$ version is explicitly stated & Population lower-bound diagnostic, not by itself a data-only certificate \\
    \bottomrule
  \end{tabular}
  \caption{Unified notation for one-sided, mirror, two-sided, and mixture score-CDF diagnostics used in the appendix.}
  \label{tab:unified-notation}
\end{table}

The notation $\delta_{\rm mix,+}^{\rm ref}$ and $\delta_{\rm mix,-}^{\rm ref}$ abbreviates the displayed quantities with inputs $D_{Q,+}$ and $D_{Q,-}$, respectively. The inequalities
\begin{equation*}
    d_{R\to P,+}^A\le\delta_{\rm mix,+}^{\rm ref},
\qquad
    d_{P\to R,+}^A\le\delta_{\rm mix,-}^{\rm ref},
    \qquad
    \max\{d_{R\to P,+}^A,d_{P\to R,+}^A\}\le\delta_{\rm mix}^{\rm ref}(D_Q)
\end{equation*}
are upper bounds from the retained-law mixture representation and the triangle inequality. They are not exact decompositions in general.

\subsection{Adaptive thresholds chosen on an independent tuning split}
\label{subsec:adaptive-thresholds}

The theorem allows data-adaptive threshold selection without extra selection penalties, provided the tuning data are independent of the contaminated calibration sample used for the final conformal quantile. Same-sample threshold selection is handled separately in Proposition~\ref{prop:same-sample-threshold-selection} and incurs an empirical-CDF uniformity penalty.

\begin{theorem}[Independent adaptive threshold]
  \label{thm:independent-adaptive-threshold}
  Let \(\mathcal H\) be an auxiliary or tuning sigma-field such that, conditionally on \(\mathcal G\vee\mathcal H\), the contaminated calibration sample remains i.i.d. from \(P^\star\) and independent of the clean test point.  Suppose \(A\), \(S\), and a random threshold \(\widehat t\) are \(\mathcal G\vee\mathcal H\)-measurable, and apply the trimmed split conformal procedure with \(t^\star=\widehat t\).  Define \(p_c(\widehat t)\), \(p_d(\widehat t)\), \(\mu_{\mathrm{keep}}(\widehat t)\), \(R_{\widehat t}\), and \(d_{R_{\widehat t}\to P,+}^{A}\) by replacing \(t^\star\) with \(\widehat t\) in Section~\ref{subsec:trimmed-setup}.  On the event \(\mu_{\mathrm{keep}}(\widehat t)>0\),
  \begin{equation}
    \label{eq:adaptive-threshold-coverage}
    \mathbb P\bigl(
      Y_{\mathrm{new}}\in C_\alpha^{\widehat t}(X_{\mathrm{new}})
      \,\big|\,
      \mathcal G\vee\mathcal H
    \bigr)
    \ge
    \bigl[1-\alpha-d_{R_{\widehat t}\to P,+}^{A}\bigr]_+.
  \end{equation}
  The exact identity \eqref{eq:exact-clean-coverage-identity} also holds conditionally on $\mathcal G\vee\mathcal H$. All retained-law quantities are evaluated at $\widehat t$. As in Theorem~\ref{thm:fixed-threshold-trimmed-coverage}, the identity uses generalised retained-score quantiles and does not require continuity of $F_{R_{\widehat t}}^A$.
\end{theorem}

\begin{proposition}[Oracle inequality for independent grid tuning]
  \label{prop:threshold-oracle-inequality}
  Let \(\mathcal T\) be a finite set of candidate thresholds.  For each \(t\in\mathcal T\), let \(B(t)\) be any population loss satisfying
  \[
    d_{R_t\to P,+}^{A}\le B(t),
  \]
  and let \(W(t)\) be any population efficiency functional.  Suppose an independent tuning stage produces \(\widehat B(t)\) and \(\widehat W(t)\), measurable with respect to \(\mathcal H\), and that on an event \(\mathcal E_{\mathrm{tune}}\in\mathcal H\),
  \[
    \sup_{t\in\mathcal T}|\widehat B(t)-B(t)|\le r_B,
    \qquad
    \sup_{t\in\mathcal T}|\widehat W(t)-W(t)|\le r_W.
  \]
  For \(\eta\ge0\), define
  \[
    \widehat{\mathcal T}_\eta=\{t\in\mathcal T:\widehat B(t)\le \eta\}.
  \]
  Fix a calibration-independent fallback threshold \(t_{\rm fb}\) with positive retained mass, for example \(+\infty\) for no trimming.  Define
  \[
    \widehat t=
    \begin{cases}
      \argmin_{t\in\widehat{\mathcal T}_\eta}\widehat W(t), & \widehat{\mathcal T}_\eta\ne\varnothing,\\
      t_{\rm fb}, & \widehat{\mathcal T}_\eta=\varnothing,
    \end{cases}
  \]
  with any deterministic tie-breaking rule.  On \(\mathcal E_{\mathrm{tune}}\cap\{\widehat{\mathcal T}_\eta\ne\varnothing\}\),
  \begin{equation}
    \label{eq:oracle-coverage-loss-control}
    B(\widehat t)\le \eta+r_B,
  \end{equation}
  and hence
  \[
    \mathbb P\{Y_{\mathrm{new}}\in C_\alpha^{\widehat t}(X_{\mathrm{new}})\mid
    \mathcal G\vee\mathcal H\}
    \ge
    \bigl[1-\alpha-\eta-r_B\bigr]_+.
  \]
  If the oracle-feasible set \(\{t\in\mathcal T:B(t)\le \eta-r_B\}\) is nonempty, then on the same event
  \begin{equation}
    \label{eq:oracle-efficiency-control}
    W(\widehat t)
    \le
    \inf_{t\in\mathcal T:B(t)\le\eta-r_B}W(t)+2r_W.
  \end{equation}
  Consequently, if
  \[
    \mathbb P\bigl(\mathcal E_{\mathrm{tune}}\cap
      \{\widehat{\mathcal T}_\eta\ne\varnothing\}\bigr)
    \ge 1-\beta-\beta_\varnothing,
  \]
  then the marginal clean-test coverage is at least
  \[
    [1-\beta-\beta_\varnothing]_+
    [1-\alpha-\eta-r_B]_+
    \ge
    1-\alpha-\eta-r_B-\beta-\beta_\varnothing,
  \]
  provided the selected threshold has positive retained mass on the event where the tuning rule is used. If a certified fallback loss $B(t_{\rm fb})\le B_{\rm fb}$ is available, the bound can be sharpened. On the empty-feasible-set event, one can use the conditional fallback loss instead of subtracting $\beta_\varnothing$.
\end{proposition}

\subsection{Operational certificate details}
\label{subsec:operational-certificate-details}

Proposition~\ref{prop:operational-certificate}, stated and proved in Appendix~\ref{app:operational-certificate}, clarifies how to use the population mixture upper-bound diagnostic. The proposition is deliberately modular. An implementation must provide a lower bound on clean retention, an upper bound on dirty retention, a clean one-sided distortion certificate, a retained dirty-side one-sided discrepancy certificate, and either the true contamination fraction or a defensible upper bound.

The worst-case dirty-side choice $B_{Q,+}=1$ is always valid, but may be conservative. Proposition~\ref{prop:exact-audit-certificate} gives a complementary route. If independent clean labelled audit data are available, the final selected set can be certified directly by exact binomial inversion. Selected score CDFs can also be certified by a one-sided KS band, without separately auditing every mixture component.

%%%%%%%%%%%%%%%%%%%%%%%%%%%%%%%%%%%%%%%%%%%%%%%%%%%%%%%%%%%%%%%%%%%

\paragraph{Audit sample size and certification margin.} Proposition~\ref{prop:exact-audit-certificate} provides an exact one sided binomial certificate for the coverage of a fixed final prediction set. Let $M$ of $n_{\mathrm{aud}}$ independent clean audit observations be covered. At level $1-\beta$, the lower confidence bound is

\begin{equation*}
    L_{\mathrm{bin}}
    =
    \operatorname{BetaInv}
    \{\beta, M, n_{\mathrm{aud}} - M + 1\}.
\end{equation*}

The final prediction set is certified to have clean coverage of at least $c$ whenever $L_{\mathrm{bin}}\geq c$.

For example, at 95\% confidence, 29 successes out of 29 observations yield a lower bound of approximately 0.902. This example also illustrates a limitation of auditing. If the true coverage is only slightly above the nominal target, a large audit sample may be needed to certify that target with high probability.

%%%%%%%%%%%%%%%%%%%%%%%%%%%%%%%%%%%%%%%%%%%%%%%%%%%%%%%%%%%%%%%%%%%

Auditing does not modify the prediction set. Its role is different from direct clean calibration, which constructs a new cut-off using clean observations.

\paragraph{Example: Non identifiability of clean coverage.} Let the contaminated score distribution be $M^\star=\operatorname{Bernoulli}(1/2)$, with a contamination level $\varepsilon=0.2$.

This observed distribution has two distinct Huber decompositions

\begin{equation*}
    P_1^A=Q_1^A=\operatorname{Bernoulli}(1/2),
\end{equation*}

and

\begin{equation*}
    P_2^A=\operatorname{Bernoulli}(5/8),
    \qquad
    Q_2^A=\delta_0.
\end{equation*}

Each decomposition satisfies

\begin{equation*}
    M^\star=(1-\varepsilon)P_i^A+\varepsilon Q_i^A.
\end{equation*}

Yet the clean score CDFs differ at zero

\begin{equation*}
    F_{P_1}^A(0)=\frac12,
    \qquad
    F_{P_2}^A(0)=\frac38.
\end{equation*}

Hence, the same contaminated score distribution may correspond to different clean coverage probabilities at the same cut-off. Clean coverage is therefore not identifiable from the contaminated score distribution alone.

The additional information required by Proposition~\ref{prop:operational-certificate} expresses this identification limit. It is not an estimator missing from the trimming procedure. Independent clean audit data or structural assumptions are required to obtain a numerical certificate.

\begin{remark}[Population and data-driven status of the certificates]
    \label{rem:operational-status}
    The refined quantities in Theorem~\ref{thm:fixed-threshold-trimmed-coverage} are population quantities for the realised auxiliary stage. In simulations, they can be computed from the known data-generating mechanism and used as diagnostics. In applications, they become coverage certificates only after the bounds in Proposition~\ref{prop:operational-certificate} are justified, without reusing the final contaminated calibration scores for post-hoc tuning. Alternatively, an independent audit sample can certify the final set through Proposition~\ref{prop:exact-audit-certificate}. If the certificate inputs hold only on a high-probability auxiliary event, Corollary~\ref{cor:marginalized-coverage} gives the corresponding marginal coverage reduction.
\end{remark}

\begin{remark}[Gaussian Mahalanobis certificate example]
  Suppose the clean covariates are certified to satisfy $X\sim N(\mu,\Sigma)$. Take
\begin{equation*}
    S(x)=(x-\mu)^\top\Sigma^{-1}(x-\mu),
\end{equation*}
with threshold $t=\chi^2_{d,q}$. Then the clean-retention certificate is $L_c=q$. If $\mu$, $\Sigma$, and the threshold are estimated from an independent clean reference split, one may instead use a DKW-lowered empirical version \citep{dvoretzky1956asymptotic,massart1990tight}.

Dirty retention and retained dirty-score discrepancy still require external information. One may assume or audit $p_d\le U_d$. When no retained-dirty score-law information is available, $B_{Q,+}=1$ is the always-valid fallback. With these inputs, Proposition~\ref{prop:operational-certificate} gives a fully numerical clean-coverage lower bound. This example shows that the certificate interface is operational only when its inputs are independently justified.
\end{remark}

\subsection{Upper-bound and threshold-efficiency statements}
\label{app:upper-bound}

\begin{proposition}[One-sided coverage upper bound under no ties or randomised ties]
  \label{prop:coverage-upper-bound}
  Under the setup of Section~\ref{subsec:trimmed-setup}, let
  \begin{equation*}
    R:=P^\star(\,\cdot\mid S(Z)\le t^\star)
  \end{equation*}
  be the retained mixture law, and assume $\mu_{\mathrm{keep}}>0$. The first inequality below requires only this retained-law positivity. The subsequent bounds involving $\delta_{\mathrm{mix},-}^{\mathrm{ref}}$ and $\delta_{\mathrm{mix}}^{\mathrm{ref}}(D_Q)$ additionally require $p_c>0$.

  If the distribution of $A(W)$, $W\sim R$, is continuous, then the ordinary inclusive trimmed split-conformal set satisfies
  \begin{align}
    \label{eq:coverage-upper-bound}
    \mathbb P\bigl(
      Y_{\mathrm{new}}\in C_\alpha(X_{\mathrm{new}})
      \,\big|\,
      \mathcal G
    \bigr)
    &\le
    1-\alpha+d_{P\to R,+}^{A}+\beta_m(\mu_{\mathrm{keep}})\notag\\
    &\le
    1-\alpha+\delta_{\mathrm{mix},-}^{\mathrm{ref}}
    +\beta_m(\mu_{\mathrm{keep}})\notag\\
    &\le
    1-\alpha+\delta_{\mathrm{mix}}^{\mathrm{ref}}(D_Q)
    +\beta_m(\mu_{\mathrm{keep}}).
  \end{align}
  The same numerical upper bound holds without the continuity assumption under the following randomised lexicographic rule. Attach independent variables $U_i,U_0\sim{\rm Unif}(0,1)$ to the retained calibration scores and the test score. Order the retained calibration pairs $(A(Z_i),U_i)$ lexicographically. Set $\widehat T_\alpha$ to the $r_{\rm keep}$-th retained pair when $r_{\rm keep}\le N_{\rm keep}$, and to $(+\infty,1)$ otherwise. Replace the prediction event by
  \[
    \{(A(Z_{\mathrm{new}}),U_0)\le_{\rm lex}\widehat T_\alpha\}.
  \]
  Here
  \begin{equation}
    \label{eq:beta-m-definition}
    \beta_m(\mu)
    :=
    \mathbb E\!\left[\frac{1}{N+1}\right]
    =
    \frac{1-(1-\mu)^{m+1}}{(m+1)\mu},
    \qquad
    N\sim\operatorname{Binomial}(m,\mu).
  \end{equation}
  Thus, the coverage excess over $1-\alpha$ is controlled by the mirror one-sided retained-to-clean distortion, plus the usual finite-sample conformal granularity term. The final symmetric inequality is a conservative fallback.
\end{proposition}

The next proposition is an efficiency diagnostic, not a coverage guarantee. The coverage theorem above does not require its local density condition.

\begin{proposition}[Threshold and width perturbation under a local density condition]
  \label{prop:threshold-efficiency}
  Assume \(\mu_{\rm keep}>0\) and \(p_c>0\), so that \(R\), \(P_{\mathrm{keep}}\), and \(\delta_{\mathrm{mix}}^{\mathrm{ref}}(D_Q)\) are well defined.
  Let \(p:=1-\alpha\), and let
  \[
    q_P(p):=\inf\{x\in\mathbb R:F_P^A(x)\ge p\}.
  \]
  Assume \(q_P(p)\in\mathbb R\), \(F_P^A\) is continuous at \(q_P(p)\), and there exist \(\lambda_{\min}>0\) and \(\rho>0\) such that the two-sided CDF increment condition \eqref{eq:cdf-increment-lower} of Lemma~\ref{lem:quantile-perturbation} holds for \(F=F_P^A\).  A sufficient condition is that \(F_P^A\) is absolutely continuous and admits a density version satisfying \(f_P^A(x)\ge\lambda_{\min}\) for Lebesgue-a.e. \(x\in[q_P(p)-\rho,q_P(p)+\rho]\).
  Fix an integer
  \[
    n_0\ge \lceil 1/\alpha\rceil-1
  \]
  and \(\beta\in(0,1)\). Define
  \[
    \eta_{n_0,\beta}
    :=
    \sqrt{\frac{\log(2/\beta)}{2n_0}},
    \qquad
    \Gamma_{n_0,\beta}
    :=
    \delta_{\mathrm{mix}}^{\mathrm{ref}}(D_Q)
    +
    \eta_{n_0,\beta}
    +
    \frac{2}{n_0}.
  \]
  If \(\Gamma_{n_0,\beta}< \lambda_{\min}\rho\), then, conditional on the fixed auxiliary stage,
  \begin{equation}
    \label{eq:threshold-efficiency-prob}
    \mathbb P\left(
      |\hat\tau_\alpha-q_P(1-\alpha)|
      \le
      \frac{\Gamma_{n_0,\beta}}{\lambda_{\min}}
    \right)
    \ge
    1-\mathbb P(N_{\mathrm{keep}}<n_0)-\beta.
  \end{equation}
  If, in addition, the nonconformity score is the absolute residual
  \(A(x,y)=|y-\hat f(x)|\), so that retained scores are nonnegative and
  \[
    C_\alpha(x)=[\hat f(x)-\hat\tau_\alpha,\hat f(x)+\hat\tau_\alpha],
  \]
  then the same event implies
  \[
    \operatorname{width}(C_\alpha(x))
    \le
    2q_P(1-\alpha)
    +
    \frac{2\Gamma_{n_0,\beta}}{\lambda_{\min}}.
  \]
\end{proposition}

\begin{remark}[Role of the density condition]
  \label{rem:density-condition}
    Proposition~\ref{prop:threshold-efficiency} is not needed for validity. It is an efficiency statement, so it requires additional regularity. In particular, the local density lower bound excludes atoms and flat regions near the clean oracle cutoff. For discrete or heavily tied nonconformity scores, the lower-bound theorem remains valid. However, this quantile-perturbation route to width control requires a separate discrete-score analysis.
\end{remark}

\subsection{Marginalising high-probability conditional guarantees}

\begin{corollary}[From conditional eventwise guarantees to marginal coverage]
  \label{cor:marginalized-coverage}
  Let \(\mathcal G\) be an auxiliary sigma-field, and let \(\mathcal E\in\mathcal G\) be an event with
  \[
    \mathbb P(\mathcal E^c)\le \beta.
  \]
  Suppose that, on \(\mathcal E\),
  \[
    \mathbb P\bigl(
      Y_{\mathrm{new}}\in C_\alpha(X_{\mathrm{new}})
      \,\big|\,
      \mathcal G
    \bigr)
    \ge
    1-\alpha-L
  \]
  for a nonnegative \(\mathcal G\)-measurable loss \(L\), and that \(L\le L_0\) on \(\mathcal E\). Then
  \begin{equation}
    \label{eq:marginalized-coverage-product}
    \mathbb P\bigl(
      Y_{\mathrm{new}}\in C_\alpha(X_{\mathrm{new}})
    \bigr)
    \ge
    (1-\beta)[1-\alpha-L_0]_+.
  \end{equation}
  In particular, since coverage probabilities are bounded by one,
  \begin{equation}
    \label{eq:marginalized-coverage-additive}
    \mathbb P\bigl(
      Y_{\mathrm{new}}\in C_\alpha(X_{\mathrm{new}})
    \bigr)
    \ge
    1-\alpha-L_0-\beta.
  \end{equation}
  For example, if
  \[
    \mathcal E=\mathcal E_s\cap\mathcal E_{\mathrm{ref}}(\delta_\ell),
    \qquad
    \mathbb P(\mathcal E_s^c)\le \beta_s,
  \]
  and the clean-reference DKW event \citep{dvoretzky1956asymptotic,massart1990tight} satisfies
  \[
    \mathbb P\bigl(\mathcal E_{\mathrm{ref}}(\delta_\ell)^c\bigr)
    \le
    2e^{-2\ell\delta_\ell^2},
  \]
  then one may take
  \[
    \beta=\beta_s+2e^{-2\ell\delta_\ell^2}.
  \]
\end{corollary}

\subsection{Contaminated-mixture evaluation}

\begin{proposition}[Coverage under the contaminated test distribution]
  \label{prop:mixture-coverage-fixed}
  Under the calibration and procedure setup of Section~\ref{subsec:trimmed-setup}, suppose the test point is distributed as
  \[
    Z_{\mathrm{new}}\sim P^\star=(1-\varepsilon)P+\varepsilon Q
  \]
  independently of the calibration sample, after conditioning on \(\mathcal G\).  Assume \(\mu_{\rm keep}>0\), and define
  \[
    d_{R\to P^\star,+}^{A}:=\sup_t\bigl(F_R^A(t)-F_{P^\star}^A(t)\bigr)_+.
  \]
  Then, almost surely in \(\mathcal G\),
  \begin{equation}
    \label{eq:mixture-direct-transfer}
    \mathbb P_{Z_{\mathrm{new}}\sim P^\star}
    \{Y_{\mathrm{new}}\in C_\alpha(X_{\mathrm{new}})\mid\mathcal G\}
    \ge
    \bigl[1-\alpha-d_{R\to P^\star,+}^{A}\bigr]_+.
  \end{equation}
  The exact identity \eqref{eq:exact-clean-coverage-identity} also holds with \(P\) replaced by \(P^\star\).  In particular, when \(t^\star=+\infty\), \(R=P^\star\), and the direct bound recovers the ordinary split-conformal lower bound \(1-\alpha\).

  The conservative clean-component reduction always remains valid
  \begin{equation}
    \label{eq:mixture-reduction}
    \mathbb P_{Z_{\mathrm{new}}\sim P^\star}
    \{Y_{\mathrm{new}}\in C_\alpha(X_{\mathrm{new}})\mid\mathcal G\}
    \ge
    (1-\varepsilon)
    \mathbb P_{Z_{\mathrm{new}}\sim P}
    \{Y_{\mathrm{new}}\in C_\alpha(X_{\mathrm{new}})\mid\mathcal G\}.
  \end{equation}
  If, in addition, \(p_c>0\), applying the clean-target mixture upper-bound diagnostic gives
  \begin{equation}
    \label{eq:mixture-coverage-fixed}
    \mathbb P_{Z_{\mathrm{new}}\sim P^\star}
    \{Y_{\mathrm{new}}\in C_\alpha(X_{\mathrm{new}})\mid\mathcal G\}
    \ge
    (1-\varepsilon)\bigl[1-\alpha-\delta_{\mathrm{mix}}^{\mathrm{ref}}(D_Q)\bigr]_+,
  \end{equation}
  and the worst-case dirty-side version is
  \begin{equation}
    \label{eq:mixture-coverage-fixed-worst}
    \mathbb P_{Z_{\mathrm{new}}\sim P^\star}
    \{Y_{\mathrm{new}}\in C_\alpha(X_{\mathrm{new}})\mid\mathcal G\}
    \ge
    (1-\varepsilon)\bigl[1-\alpha-\delta_{\mathrm{mix}}^{\mathrm{wc}}\bigr]_+.
  \end{equation}
  The last two bounds are deliberately conservative, because they ignore any coverage of the dirty test component.
\end{proposition}

\begin{remark}[Choosing the fixed threshold]\label{rem:threshold-choice}
  Throughout the theoretical analysis, $t^\star$ is treated as fixed after conditioning on the clean auxiliary stage. This is essential for the pointwise-thinning argument in Theorem~\ref{thm:fixed-threshold-trimmed-coverage}. Once the auxiliary clean data have been realised, the retention indicators
  \begin{equation*}
      K_i=\mathbf{1}\{S_i\le t^\star\}
  \end{equation*}
  must be defined by applying a deterministic threshold separately to each calibration observation.

  In practice, however, $t^\star$ need not be chosen arbitrarily. Several principled strategies are available. One can choose $t^\star$ using external validation data or prior domain knowledge about anomalous values of $S$. Another option is a sensitivity analysis over a prespecified threshold grid, reporting coverage and retained sample size across the grid.

  A third option is to tune $t^\star$ using an auxiliary clean reference set, when available. More generally, Theorem~\ref{thm:independent-adaptive-threshold} allows arbitrary threshold-selection rules based on an independent tuning split. Proposition~\ref{prop:threshold-oracle-inequality} gives a finite-grid oracle inequality when the tuning split provides uniform estimates of coverage loss and efficiency.

  In implementations with a fitted auxiliary anomaly score, the clean reference split is used only for threshold selection. The anomaly score itself is already fixed after the auxiliary score-fitting stage.

  Our theory makes clear that the choice of $t^\star$ affects robustness through
  two distinct quantities:
  \begin{equation*}
    \tilde\varepsilon_\star
    =
    \frac{\varepsilon p_d}{(1-\varepsilon)p_c+\varepsilon p_d},
    \qquad
    \Delta_{\mathrm{trim}}^{A}
    =
    \sup_{t\in\mathbb{R}}
    \bigl|F_{P_{\mathrm{keep}}}^A(t)-F_P^A(t)\bigr|.
  \end{equation*}
  The first term, $\tilde\varepsilon_\star$, is the retained mixture coefficient of the dirty component in the retained calibration set. The second term, $\Delta_{\mathrm{trim}}^{A}$, measures how much the same threshold distorts the clean nonconformity distribution.

  This suggests a practical rule for choosing $t^\star$. If the main goal is reliable clean-test coverage, one should first keep $\Delta_{\mathrm{trim}}^{A}$ small. Overly aggressive trimming can strongly perturb the clean conformal quantile, even when it reduces the retained mixture coefficient.

  When $S$ clearly separates clean and contaminated observations, more aggressive trimming can help. It can reduce $\tilde\varepsilon_\star$ without causing a comparable increase in $\Delta_{\mathrm{trim}}^{A}$. In weak- or no-separation regimes, aggressive trimming is unlikely to help. It mainly removes clean calibration points and increases clean trimming distortion, without substantially reducing the retained mixture coefficient.

  If an auxiliary clean reference split is available, a natural choice is a clean-reference quantile rule for $t^\star$. This stabilises the clean retention probability at a target level. More broadly, threshold selection balances two objectives: making contaminated points unlikely to survive trimming, and keeping the retained clean nonconformity distribution close to its untrimmed version.
\end{remark}

\section{Technical remarks and edge cases}
\label{app:technical-remarks}
\begin{lemma}[Generalised quantiles commute with order statistics]
  \label{lem:generalized-quantile-order}
  Let \(Q:(0,1)\to\mathbb R\cup\{\pm\infty\}\) be non-decreasing, and let \(U_1,\ldots,U_n\) be real numbers in \((0,1)\).  Set \(X_i=Q(U_i)\).  Then for every \(r\in\{1,\ldots,n\}\),
  \[
    X_{(r)}=Q(U_{(r)}),
  \]
  where order statistics are taken in non decreasing order. Consequently, if $Q=Q_F$ is the lower generalised quantile of a distribution function $F$, and $U_1,\ldots,U_n\iid {\rm Unif}(0,1)$, then $Q_F(U_i)\sim F$. The $r$-th order statistic of the transformed sample has the same distribution as $Q_F(U_{(r)})$, even when $F$ has atoms.
\end{lemma}

\begin{proof}
  Sort the uniforms as $U_{(1)}\le\cdots\le U_{(n)}$. Since $Q$ is non-decreasing, the transformed values satisfy
  \begin{equation*}
      Q(U_{(1)})\le\cdots\le Q(U_{(n)}),
  \end{equation*}
  with equalities allowed. Hence the $r$-th ordered transformed value is $Q(U_{(r)})$. The distributional statement follows from the generalised inverse transform.
\end{proof}
\begin{remark}[Degenerate infinite-threshold event]
  \label{rem:degenerate-threshold}
  Because the theorem gives a lower bound, it is important to isolate the event on which the algorithm returns the degenerate cutoff $\hat\tau_\alpha=+\infty$. Since the retention indicators
  \begin{equation*}
    K_i=\mathbf 1\{S(Z_i)\le t^\star\}
  \end{equation*}
  are i.i.d. Bernoulli\((\mu_{\mathrm{keep}})\), we have
  \begin{equation*}
    N_{\mathrm{keep}}
    \sim
    \mathrm{Binomial}\bigl(m,\mu_{\mathrm{keep}}\bigr).
  \end{equation*}
  Define
  \begin{equation*}
    n_{\infty,\alpha}
    :=
    \max\Bigl\{0,\left\lceil \alpha^{-1}\right\rceil-2\Bigr\}.
  \end{equation*}
  By the algorithmic definition of the cutoff,
  \begin{equation*}
    \hat\tau_\alpha=+\infty
    \qquad\Longleftrightarrow\qquad
    N_{\mathrm{keep}}\le n_{\infty,\alpha}.
  \end{equation*}
  For the experimental nominal level $\alpha=0.1$, this threshold is $n_{\infty,\alpha}=8$. Retaining eight or fewer calibration points forces the uninformative infinite cutoff. Equivalently,
  \begin{equation*}
    \mathbb P\bigl(\hat\tau_\alpha=+\infty\bigr)
    =
    \sum_{j=0}^{n_{\infty,\alpha}}
    \binom{m}{j}
    \mu_{\mathrm{keep}}^{\,j}
    (1-\mu_{\mathrm{keep}})^{m-j}.
  \end{equation*}
  In particular,
  \begin{equation*}
    \mathbb P\bigl(N_{\mathrm{keep}}=0\bigr)
    =
    (1-\mu_{\mathrm{keep}})^m
    \le
    e^{-m\mu_{\mathrm{keep}}}.
  \end{equation*}
  If $m\mu_{\mathrm{keep}}>n_{\infty,\alpha}$, Hoeffding's inequality \citep{hoeffding1963probability} further
  gives
  \begin{equation*}
    \mathbb P\bigl(\hat\tau_\alpha=+\infty\bigr)
    \le
    \exp\!\left(
      -2m
      \left(
        \mu_{\mathrm{keep}}-\frac{n_{\infty,\alpha}}{m}
      \right)^2
    \right).
  \end{equation*}

  This event is controlled explicitly by the retained mass $\mu_{\mathrm{keep}}$. It should be interpreted as an efficiency failure, not a validity failure. If $C_\alpha(x)=\mathcal Y$ on this event, then the conditional coverage is indeed one. However, the prediction set is uninformative. For symmetric interval-valued conformal rules, this corresponds to infinite width.
\end{remark}

\begin{remark}[Comparison with the ordinary untrimmed baseline]
  \label{rem:trimmed-vs-untrimmed}
  The theorem specialises immediately to ordinary split conformal prediction.
  Indeed, if one sets \(t^\star=+\infty\), then every calibration point is
  retained, so
  \begin{equation*}
    p_c=p_d=1,
    \qquad
    P_{\mathrm{keep}}=P,
    \qquad
    \Delta_{\mathrm{trim}}^{A}=0,
    \qquad
    \tilde\varepsilon_\star=\varepsilon.
  \end{equation*}
  In that case $Q_{\mathrm{keep}}=Q$, and the retained-dirty discrepancy
  reduces to
  \begin{equation*}
    D_Q^{\mathrm{untrim}}
    :=
    \sup_{t\in\mathbb R}
    \bigl|
      F_Q^{A}(t)-F_P^A(t)
    \bigr|.
  \end{equation*}
  Therefore Theorem~\ref{thm:fixed-threshold-trimmed-coverage} yields the theorem-level
  lower bound for the untrimmed rule
  \begin{equation*}
    \mathbb P\bigl(
      Y_{\mathrm{new}}\in C_\alpha^{\mathrm{untrim}}(X_{\mathrm{new}})
      \,\big|\,
      \mathcal G
    \bigr)
    \ge
    1-\alpha-\varepsilon D_Q^{\mathrm{untrim}}.
  \end{equation*}

  Comparing this with the trimmed lower bound
  \begin{equation*}
    1-\alpha
    -
    (1-\tilde\varepsilon_\star)\Delta_{\mathrm{trim}}^{A}
    -
    \tilde\varepsilon_\star D_Q,
  \end{equation*}
  shows that the trimmed theorem-level lower bound is sharper than the untrimmed theorem-level lower bound whenever
  \begin{equation}
    \label{eq:trimmed-vs-untrimmed-condition}
    (1-\tilde\varepsilon_\star)\Delta_{\mathrm{trim}}^{A}
    +
    \tilde\varepsilon_\star D_Q
    <
    \varepsilon D_Q^{\mathrm{untrim}}.
  \end{equation}
  This comparison isolates the paper’s central trade-off. Trimming improves the lower bound only when the reduction in retained contamination is large enough to offset the clean-law distortion caused by filtering. The comparison is intentionally one-sided. It concerns theorem-level lower bounds, and does not by itself imply actual coverage dominance, efficiency improvement, or width reduction.
\end{remark}

\begin{remark}[On the clean trimming distortion $\Delta_{\mathrm{trim}}^{A}$]\label{rem:trimming-bias}
  The quantity
  \begin{equation*}
    \Delta_{\mathrm{trim}}^{A}
    =
    \sup_{t\in\mathbb{R}}
    \bigl|F_{P_{\mathrm{keep}}}^A(t) - F_P^A(t)\bigr|
  \end{equation*}
  measures how much the distribution of the nonconformity score $A(Z)$ is distorted by the fixed-threshold trimming rule. Here $F_P^A$ denotes the CDF of $A(Z)$ when $Z\sim P$, while $F_{P_{\mathrm{keep}}}^A$ denotes the CDF of $A(Z)$ under the retained clean law
  \begin{equation*}
    P_{\mathrm{keep}}
    :=
    P(\,\cdot\,\mid S(Z)\le t^\star),
  \end{equation*}
  that is, the conditional law of a clean point given retention by the population trimming rule induced by the anomaly score $S$ and fixed threshold $t^\star$. The assumption $p_c>0$ in Theorem~\ref{thm:fixed-threshold-trimmed-coverage} ensures that this conditional law, and hence $\Delta_{\mathrm{trim}}^{A}$, is well-defined.

  This quantity is not merely a distributional discrepancy. Split conformal prediction is driven by a high quantile of the clean nonconformity law. Therefore, changing the retained-clean distribution directly perturbs the conformal cutoff. In particular, Corollary~\ref{cor:delta-bias-quantile} shows that, under a local density lower bound near the clean $(1-\alpha)$-quantile, the induced quantile shift is controlled linearly by $\Delta_{\mathrm{trim}}^{A}$. Thus, $\Delta_{\mathrm{trim}}^{A}$ should be read as a direct measure of how much trimming can move the clean conformal threshold. It is not merely an abstract discrepancy between two distributions.

  In particular, $\Delta_{\mathrm{trim}}^{A}$ captures the dependence between the anomaly score $S$ and the nonconformity score $A$ under the clean distribution $P$

  \begin{itemize}
    \item If $S$ and $A$ are independent under $P$, then conditioning on the event
      $\{S(Z)\le t^\star\}$ does not change the marginal distribution of $A(Z)$.
      Consequently,
      \begin{equation*}
        F_{P_{\mathrm{keep}}}^A \equiv F_P^A,
        \qquad
        \Delta_{\mathrm{trim}}^{A} = 0.
      \end{equation*}
      In this case, trimming does not distort the clean conformal score distribution. The loss term in the population coverage bound therefore comes from residual retained contamination. Separate finite-sample effects, such as the probability of a degenerate infinite threshold and the granularity of the retained conformal quantile, affect efficiency or certificate sharpness. They do not add a separate loss term to the finite-sample rank-valid retained-law conformal step.

    \item If $S$ and $A$ are positively associated under $P$, trimming may preferentially remove part of the right tail of the clean score distribution. For example, ``hard'' clean samples may both look more anomalous under $S$ and have larger nonconformity scores $A$. In this case, the retained score distribution is typically shifted toward smaller values than the full clean score distribution. Equivalently, $F_{P_{\mathrm{keep}}}^A$ is typically larger than $F_P^A$ at relevant thresholds, and $\Delta_{\mathrm{trim}}^{A}>0$. This creates a potential under-coverage effect, because the retained calibration scores can be systematically smaller than the full clean calibration scores.
  \end{itemize}

  From a modelling perspective, $\Delta_{\mathrm{trim}}^{A}$ is a sensitivity parameter. It quantifies how much the trimming rule changes the clean calibration distribution. Thus, when designing $S$ and $A$, one should choose an anomaly score $S$ that separates contaminated points well, while remaining as weakly dependent as possible on the clean nonconformity score $A$.
\end{remark}

\section{Auxiliary and diagnostic results}
\label{app:auxiliary-diagnostic-lemmas}
\subsection{Diagnostic retained-count concentration}
The following concentration lemma is used as a diagnostic retained-count certificate and to control rare small-retained-sample events. It is not needed for the proof of Theorem~\ref{thm:fixed-threshold-trimmed-coverage}, which conditions directly on the realised retained sample size.

\begin{lemma}[Diagnostic retained-count and contamination bounds for fixed-threshold trimming]
  \label{lem:fixed-threshold}

  Let $Z_1,\dots,Z_m$ be generated under the $\varepsilon$-contamination
  model \citep{huber1964robust}
  \begin{equation*}
    P^\star
    :=
    (1-\varepsilon)P+\varepsilon Q,
    \qquad
    0\le\varepsilon<1.
  \end{equation*}
  For bookkeeping, realise this mixture on an expanded probability space through independent latent labels
  \begin{equation*}
    C_i\sim \operatorname{Bernoulli}(1-\varepsilon),
    \qquad
    D_i:=1-C_i,
  \end{equation*}
  with
  \begin{equation*}
    Z_i\mid \{C_i=1\}\sim P,
    \qquad
    Z_i\mid \{D_i=1\}\sim Q,
  \end{equation*}
  independently across $i=1,\dots,m$. Marginally, each $Z_i$ has law $P^\star$.

  Let $S:\mathcal Z\to\mathbb R$ be a measurable score function and fix a deterministic threshold $t^\star\in\overline{\mathbb R}$. For each $i=1,\dots,m$,
  define
  \begin{equation*}
    K_i:=\mathbf 1\{S(Z_i)\le t^\star\},
  \end{equation*}
  and define the retained clean, dirty, and total counts
  \begin{equation*}
    N_{\mathrm{clean}}
    :=
    \sum_{i=1}^m C_iK_i,
    \qquad
    N_{\mathrm{dirty}}
    :=
    \sum_{i=1}^m D_iK_i,
    \qquad
    N_{\mathrm{keep}}
    :=
    \sum_{i=1}^m K_i
    =
    N_{\mathrm{clean}} + N_{\mathrm{dirty}}.
  \end{equation*}

  Let
  \begin{equation*}
    p_c
    :=
    \mathbb P_{Z\sim P}\bigl(S(Z)\le t^\star\bigr),
    \qquad
    p_d
    :=
    \mathbb P_{Z\sim Q}\bigl(S(Z)\le t^\star\bigr),
  \end{equation*}
  and set
  \begin{equation*}
    \pi_c
    :=
    (1-\varepsilon)p_c,
    \qquad
    \pi_d
    :=
    \varepsilon p_d,
    \qquad
    \mu_{\mathrm{keep}}
    :=
    \pi_c+\pi_d
    =
    (1-\varepsilon)p_c+\varepsilon p_d.
  \end{equation*}

  Then the following bounds hold.
  \begin{enumerate}[label=(\roman*)]
    \item (\emph{Retained-count lower bound})
      For any \(\eta>0\),
      \begin{equation}
        \label{eq:keep-lower-prob}
        \mathbb P\Bigl(
          N_{\mathrm{keep}}
          \ge
          m(\mu_{\mathrm{keep}}-\eta)
        \Bigr)
        \ge
        1-\exp(-2m\eta^2).
      \end{equation}

    \item (\emph{Simultaneous clean, dirty, and retained count bounds})
      For any $\eta>0$, with probability at least
      \begin{equation*}
        1-3\exp(-2m\eta^2),
      \end{equation*}
      the following inequalities hold simultaneously
      \begin{align}
        N_{\mathrm{clean}}
        &\ge
        m(\pi_c-\eta),
        \label{eq:clean-lower}
        \\
        N_{\mathrm{dirty}}
        &\le
        m(\pi_d+\eta),
        \label{eq:dirty-upper}
        \\
        N_{\mathrm{keep}}
        &\ge
        m(\mu_{\mathrm{keep}}-\eta).
        \label{eq:keep-lower}
      \end{align}
      The factor of three is the cost of a simple union bound. Since $N_{\mathrm{keep}} = N_{\mathrm{clean}} + N_{\mathrm{dirty}}$, sharper simultaneous multinomial concentration inequalities could improve this diagnostic constant. They are not needed for the retained-law validity theorem.

    \item (\emph{Retained contamination proportion})
      For any $\eta_d>0$ and any $\eta_k\in(0,\mu_{\mathrm{keep}})$, with
      probability at least
      \begin{equation*}
        1-\exp(-2m\eta_d^2)-\exp(-2m\eta_k^2),
      \end{equation*}
      one has
      \begin{equation}
        \label{eq:retained-contamination-ratio-two-etas}
        \frac{N_{\mathrm{dirty}}}{N_{\mathrm{keep}}}
        \le
        \frac{\pi_d+\eta_d}{\mu_{\mathrm{keep}}-\eta_k}
        =
        \frac{\varepsilon p_d+\eta_d}
        {(1-\varepsilon)p_c+\varepsilon p_d-\eta_k}.
      \end{equation}
      In particular, taking $\eta_d=\eta_k=\eta\in(0,\mu_{\mathrm{keep}})$
      gives the simpler bound
      \begin{equation}
        \label{eq:retained-contamination-ratio}
        \frac{N_{\mathrm{dirty}}}{N_{\mathrm{keep}}}
        \le
        \frac{\pi_d+\eta}{\mu_{\mathrm{keep}}-\eta}
        =
        \frac{\varepsilon p_d+\eta}
        {(1-\varepsilon)p_c+\varepsilon p_d-\eta}
      \end{equation}
      with probability at least \(1-2\exp(-2m\eta^2)\).
  \end{enumerate}
\end{lemma}

\begin{proof}
  Define
  \begin{equation*}
    X_i^{(c)} := C_iK_i, \qquad
    X_i^{(d)} := D_iK_i, \qquad
    i=1,\dots,m.
  \end{equation*}
  Each of $X_i^{(c)}$, $X_i^{(d)}$, and $K_i$ takes values in $\{0,1\}$. Since the latent labels and observations are independent across $i$, and the threshold $t^\star$ is fixed, each of the three sequences
  \begin{equation*}
    \{X_i^{(c)}\}_{i=1}^m,
    \qquad
    \{X_i^{(d)}\}_{i=1}^m,
    \qquad
    \{K_i\}_{i=1}^m
  \end{equation*}
  is individually i.i.d.\ Bernoulli across $i$. Within a single index $i$,
  however, the variables are functionally related through
  \begin{equation*}
    K_i=X_i^{(c)}+X_i^{(d)}.
  \end{equation*}
  Only cross-\(i\) independence is used below, through separate Hoeffding bounds \citep{hoeffding1963probability}
  and union bounds. By construction,
  \begin{equation*}
    N_{\mathrm{clean}}
    =
    \sum_{i=1}^m X_i^{(c)},
    \qquad
    N_{\mathrm{dirty}}
    =
    \sum_{i=1}^m X_i^{(d)},
    \qquad
    N_{\mathrm{keep}}
    =
    \sum_{i=1}^m K_i.
  \end{equation*}

  Under the latent-label realisation of the $\varepsilon$-contamination model,
  \begin{equation*}
    \mathbb P(X_i^{(c)}=1)
    =
    \mathbb P(C_i=1,\ S(Z_i)\le t^\star)
    =
    (1-\varepsilon)
    \mathbb P_{Z\sim P}\bigl(S(Z)\le t^\star\bigr)
    =
    \pi_c.
  \end{equation*}
  Similarly,
  \begin{equation*}
    \mathbb P(X_i^{(d)}=1)
    =
    \varepsilon
    \mathbb P_{Z\sim Q}\bigl(S(Z)\le t^\star\bigr)
    =
    \pi_d,
  \end{equation*}
  and
  \begin{equation*}
    \mathbb P(K_i=1)
    =
    (1-\varepsilon)p_c+\varepsilon p_d
    =
    \mu_{\mathrm{keep}}.
  \end{equation*}

  Hoeffding's one-sided lower-tail inequality \citep{hoeffding1963probability} for i.i.d.\ variables
  $Y_1,\dots,Y_m\in[0,1]$ gives
  \begin{equation*}
    \mathbb P\Bigl(
      \frac{1}{m}\sum_{i=1}^m Y_i-\mathbb E[Y_1]
      \le
      -\eta
    \Bigr)
    \le
    \exp(-2m\eta^2).
  \end{equation*}
  Applying this with $Y_i=K_i$ proves \eqref{eq:keep-lower-prob}. Applying one-sided Hoeffding bounds to
  \begin{equation*}
    \sum_{i=1}^m X_i^{(c)},
    \qquad
    \sum_{i=1}^m X_i^{(d)},
    \qquad
    \sum_{i=1}^m K_i,
  \end{equation*}
  and taking a union bound gives the following conservative simultaneous control. This uses the deterministic identity $K_i=X_i^{(c)}+X_i^{(d)}$ only for bookkeeping. A multinomial concentration argument could sharpen the constants.
  \begin{equation*}
    \mathbb P\Bigl(
      N_{\mathrm{clean}}\ge m(\pi_c-\eta),\
      N_{\mathrm{dirty}}\le m(\pi_d+\eta),\
      N_{\mathrm{keep}}\ge m(\mu_{\mathrm{keep}}-\eta)
    \Bigr)
    \ge
    1-3\exp(-2m\eta^2),
  \end{equation*}
  proving \eqref{eq:clean-lower}--\eqref{eq:keep-lower}.

  On the event
  \begin{equation*}
    \Bigl\{
      N_{\mathrm{dirty}}\le m(\pi_d+\eta_d)
    \Bigr\}
    \cap
    \Bigl\{
      N_{\mathrm{keep}}\ge m(\mu_{\mathrm{keep}}-\eta_k)
    \Bigr\},
  \end{equation*}
  the denominator is positive because $\eta_k<\mu_{\mathrm{keep}}$, and
  \begin{equation*}
    \frac{N_{\mathrm{dirty}}}{N_{\mathrm{keep}}}
    \le
    \frac{m(\pi_d+\eta_d)}{m(\mu_{\mathrm{keep}}-\eta_k)}
    =
    \frac{\pi_d+\eta_d}{\mu_{\mathrm{keep}}-\eta_k}.
  \end{equation*}
  Hoeffding’s upper-tail bound for $N_{\mathrm{dirty}}$ and lower-tail bound for $N_{\mathrm{keep}}$, followed by a union bound, give probability at least
  \begin{equation*}
      1-\exp(-2m\eta_d^2)-\exp(-2m\eta_k^2).
  \end{equation*}
This proves \eqref{eq:retained-contamination-ratio-two-etas}. The one-parameter version is the special case $\eta_d=\eta_k$.
\end{proof}

\subsection{Quantile perturbation and trimming-distortion identities}
\begin{lemma}[Quantile perturbation under Kolmogorov distance]
  \label{lem:quantile-perturbation}
  Let $F$ and $G$ be cumulative distribution functions on $\mathbb R$.  For $p\in(0,1)$, define the lower generalised $p$-quantiles
  \begin{equation*}
    q_F(p) := \inf\{ x\in\mathbb R : F(x) \ge p \},
    \qquad
    q_G(p) := \inf\{ x\in\mathbb R : G(x) \ge p \}.
  \end{equation*}
  Let $x_0=q_F(p)$. Assume $x_0\in\mathbb R$, $F$ is continuous at $x_0$, and there exist constants $\lambda_{\min}>0$ and $\rho>0$ such that, for every $h\in[0,\rho]$,
  \begin{equation}
    \label{eq:cdf-increment-lower}
    F(x_0+h)-F(x_0)\ge \lambda_{\min}h,
    \qquad
    F(x_0)-F(x_0-h)\ge \lambda_{\min}h.
  \end{equation}
  Suppose also that
  \begin{equation}
    \label{eq:kolmogorov-distance}
    \sup_{t\in\mathbb R}|F(t)-G(t)|\le \Delta
  \end{equation}
  and
  \begin{equation}
    \label{eq:delta-small}
    \Delta<\lambda_{\min}\rho.
  \end{equation}
  Then
  \begin{equation}
    \label{eq:quantile-perturbation}
    |q_G(p)-q_F(p)|\le \frac{\Delta}{\lambda_{\min}}.
  \end{equation}
  In particular, $q_G(p)\in[x_0-\rho,x_0+\rho]$. Condition \eqref{eq:cdf-increment-lower} holds, for example, if $F$ is absolutely continuous and has a density version satisfying $f_F(x)\ge\lambda_{\min}$ for Lebesgue-a.e. $x\in[x_0-\rho,x_0+\rho]$.
\end{lemma}

\begin{proof}
  By the definition of $x_0=q_F(p)$, $F(x)<p$ for all $x<x_0$, and $F(x_0)\ge p$.  Continuity at $x_0$ gives $F(x_0)=p$.

  Set $h_+=\Delta/\lambda_{\min}$. By \eqref{eq:delta-small}, $h_+<\rho$ unless $\Delta=0$, in which case $h_+=0\le\rho$. From \eqref{eq:cdf-increment-lower},
  \begin{equation*}
    F(x_0+h_+)\ge p+\Delta.
  \end{equation*}
  Therefore \eqref{eq:kolmogorov-distance} implies
  \begin{equation*}
    G(x_0+h_+)\ge F(x_0+h_+)-\Delta\ge p,
  \end{equation*}
  so $q_G(p)\le x_0+\frac{\Delta}{\lambda_{\min}}$.

  For the lower bound, let $\xi>0$ and set $h_\xi=(\Delta+\xi)/\lambda_{\min}$. For sufficiently small $\xi$, $h_\xi\le\rho$. Again by \eqref{eq:cdf-increment-lower},
  \begin{equation*}
    F(x_0 - h_\xi)\le p - (\Delta + \xi).
  \end{equation*}
  Hence
  \begin{equation*}
    G(x_0-h_\xi)\le F(x_0-h_\xi)+\Delta\le p-\xi<p.
  \end{equation*}
  By the definition of the lower generalised quantile,
  \begin{equation*}
    q_G(p)\ge x_0-h_\xi=x_0-\frac{\Delta+\xi}{\lambda_{\min}}.
  \end{equation*}
  Letting $\xi\downarrow0$ gives $q_G(p)\ge x_0-\frac{\Delta}{\lambda_{\min}}$. Combining the two bounds proves \eqref{eq:quantile-perturbation}. The final inclusion follows from \eqref{eq:delta-small}.
\end{proof}

\begin{remark}[Scope of the quantile perturbation lemma]
  \label{rem:quantile-perturbation-scope}
  Lemma~\ref{lem:quantile-perturbation} is an efficiency and stability tool, not a prerequisite for conformal validity. The strict inequality in \eqref{eq:delta-small} is intentional. At the boundary $\Delta=\lambda_{\min}\rho$, the local CDF-increment information may not extend far enough to support the left-tail perturbation argument without enlarging the neighbourhood. If the score law has atoms, or if the CDF is locally flat near $q_F(p)$, the lemma may not apply. Discrete nonconformity scores require a separate perturbation analysis.
\end{remark}

\begin{lemma}[Universal upper bound for the clean trimming distortion]
  \label{lem:delta-bias-tv}
  Let $P$ be the clean distribution, let $S:\mathcal Z\to\mathbb R$ be a measurable score, and fix a deterministic threshold $t^\star\in\overline{\mathbb R}$.
  Define
  \begin{equation*}
    p_c
    :=
    \mathbb P_{Z\sim P}\bigl(S(Z)\le t^\star\bigr).
  \end{equation*}
  Assume $p_c>0$, and define the retained clean law
  \begin{equation*}
    P_{\mathrm{keep}}
    :=
    P(\,\cdot\mid S(Z)\le t^\star).
  \end{equation*}
  Let
  \begin{equation*}
    \Delta_{\mathrm{trim}}^{A}
    :=
    \sup_{t\in\mathbb R}
    \bigl|
    F_{P_{\mathrm{keep}}}^{A}(t)-F_{P}^{A}(t)
    \bigr|,
  \end{equation*}
  where $F_{P_{\mathrm{keep}}}^{A}$ and $F_P^{A}$ denote the cumulative distribution functions of $A(Z)$ under $Z\sim P_{\mathrm{keep}}$ and $Z\sim P$, respectively.

  The case $t^\star=+\infty$ is included. Then $p_c=1$, $P_{\mathrm{keep}}=P$, and the bound below gives zero distortion.

  Then
  \begin{equation}
    \label{eq:delta-bias-tv-bound}
    \Delta_{\mathrm{trim}}^{A}
    \le
    \operatorname{TV}\bigl(P_{\mathrm{keep}}^{A},P^{A}\bigr)
    \le
    \operatorname{TV}\bigl(P_{\mathrm{keep}},P\bigr)
    =
    1-p_c.
  \end{equation}
\end{lemma}

\begin{proof}
  By definition of $\Delta_{\mathrm{trim}}^{A}$,
  \begin{equation*}
    \Delta_{\mathrm{trim}}^{A}
    =
    \sup_{t\in\mathbb R}
    \bigl|
    P_{\mathrm{keep}}^{A}(({-\infty},t])
    -
    P^{A}(({-\infty},t])
    \bigr|.
  \end{equation*}
  Since the half-lines $({-\infty},t]$ form a subclass of all measurable sets on
  $\mathbb R$, we immediately obtain
  \begin{equation*}
    \Delta_{\mathrm{trim}}^{A}
    \le
    \sup_{B\in\mathcal B(\mathbb R)}
    \bigl|P_{\mathrm{keep}}^{A}(B)-P^{A}(B)\bigr|
    = 
    \operatorname{TV}\bigl(P_{\mathrm{keep}}^{A},P^{A}\bigr).
  \end{equation*}

  Next, $P_{\mathrm{keep}}^{A}$ and $P^{A}$ are pushforwards of
  $P_{\mathrm{keep}}$ and $P$ under the measurable map $A$. For any Borel set
  $B\subseteq\mathbb R$,
  \begin{equation*}
    P_{\mathrm{keep}}^{A}(B)-P^{A}(B)
    =
    P_{\mathrm{keep}}\bigl(A^{-1}(B)\bigr)-P\bigl(A^{-1}(B)\bigr),
  \end{equation*}
  hence
  \begin{equation*}
    \operatorname{TV}\bigl(P_{\mathrm{keep}}^{A},P^{A}\bigr)
    \le
    \operatorname{TV}\bigl(P_{\mathrm{keep}},P\bigr).
  \end{equation*}

  It remains to compute $\operatorname{TV}(P_{\mathrm{keep}},P)$. Since
  $P_{\mathrm{keep}}=P(\,\cdot\mid S(Z)\le t^\star)$, its Radon--Nikodym derivative with respect to $P$ is
  \begin{equation*}
    \frac{dP_{\mathrm{keep}}}{dP}(z)
    =
    \frac{\mathbf 1\{S(z)\le t^\star\}}{p_c}.
  \end{equation*}
  Therefore,
  \begin{align*}
    \operatorname{TV}\bigl(P_{\mathrm{keep}},P\bigr)
    &=
    \frac12
    \int
    \left|
    \frac{\mathbf 1\{S(z)\le t^\star\}}{p_c}-1
    \right|\,dP(z) \\
    &=
    \frac12
    \left[
      \int_{\{S\le t^\star\}}
      \left(\frac1{p_c}-1\right)dP
      +
      \int_{\{S>t^\star\}} 1\,dP
    \right] \\
    &=
    \frac12
    \left[
      p_c\left(\frac1{p_c}-1\right)+(1-p_c)
    \right] \\
    &=
    \frac12
    \left[
      (1-p_c)+(1-p_c)
    \right]
    =
    1-p_c,
  \end{align*}
  where we used
  \begin{equation*}
    p_c\left(\frac1{p_c}-1\right)=1-p_c.
  \end{equation*}
  Combining the three displays proves \eqref{eq:delta-bias-tv-bound}.
\end{proof}

\begin{proposition}[Covariance and retained--discarded representations of the clean trimming distortion]
  \label{prop:delta-bias-covariance}
  Let $P$ be the clean distribution, let $S:\mathcal Z\to\mathbb R$ be a measurable score, and fix a deterministic threshold $t^\star\in\mathbb R$.
  Let
  \begin{equation*}
    K
    :=
    \mathbf 1\{S(Z)\le t^\star\},
    \qquad
    p_c
    :=
    \mathbb P_{Z\sim P}(K=1),
  \end{equation*}
  and assume $p_c>0$. Define the retained clean law
  \begin{equation*}
    P_{\mathrm{keep}}
    :=
    P(\,\cdot\mid S(Z)\le t^\star),
  \end{equation*}
  and the clean trimming distortion
  \begin{equation*}
    \Delta_{\mathrm{trim}}^{A}
    :=
    \sup_{t\in\mathbb R}
    \bigl|
    F_{P_{\mathrm{keep}}}^{A}(t)-F_{P}^{A}(t)
    \bigr|.
  \end{equation*}

  Then for every $t\in\mathbb R$,
  \begin{equation}
    \label{eq:delta-bias-covariance-pointwise}
    F_{P_{\mathrm{keep}}}^{A}(t)-F_{P}^{A}(t)
    =
    \frac{
      \operatorname{Cov}_{P}
      \bigl(
        \mathbf 1\{A(Z)\le t\},\,K
      \bigr)
    }{p_c}.
  \end{equation}
  Consequently,
  \begin{equation}
    \label{eq:delta-bias-covariance-sup}
    \Delta_{\mathrm{trim}}^{A}
    =
    \frac1{p_c}
    \sup_{t\in\mathbb R}
    \left|
    \operatorname{Cov}_{P}
    \bigl(
      \mathbf 1\{A(Z)\le t\},\,K
    \bigr)
    \right|.
  \end{equation}

  If $p_c<1$, define the discarded clean law
  \begin{equation*}
    P_{\mathrm{drop}}
    :=
    P(\,\cdot\mid S(Z)>t^\star).
  \end{equation*}
  Then the clean trimming distortion also admits the exact retained--discarded representation
  \begin{equation}
    \label{eq:delta-bias-retained-discarded}
    \Delta_{\mathrm{trim}}^{A}
    =
    (1-p_c)
    \sup_{t\in\mathbb R}
    \bigl|
      F_{P_{\mathrm{keep}}}^{A}(t)
      -
      F_{P_{\mathrm{drop}}}^{A}(t)
    \bigr|.
  \end{equation}
  If $p_c = 1$, then $P_{\mathrm{keep}}=P$ and
  $\Delta_{\mathrm{trim}}^{A}=0$.
\end{proposition}

\begin{proof}
  For any $t\in\mathbb R$,
  \begin{equation*}
    F_{P_{\mathrm{keep}}}^{A}(t)
    =
    P\bigl(A(Z)\le t\mid K=1\bigr)
    =
    \frac{\mathbb E_P[\mathbf 1\{A(Z)\le t\}K]}{p_c}.
  \end{equation*}
  Therefore,
  \begin{align*}
    F_{P_{\mathrm{keep}}}^{A}(t)-F_{P}^{A}(t)
    &=
    \frac{\mathbb E_P[\mathbf 1\{A(Z)\le t\}K]}{p_c}
    -
    \mathbb E_P[\mathbf 1\{A(Z)\le t\}] \\
    &=
    \frac{
      \mathbb E_P[\mathbf 1\{A(Z)\le t\}K]
      -
      p_c\,\mathbb E_P[\mathbf 1\{A(Z)\le t\}]
    }{p_c} \\
    &=
    \frac{
      \operatorname{Cov}_{P}
      \bigl(
        \mathbf 1\{A(Z)\le t\},\,K
      \bigr)
    }{p_c},
  \end{align*}
  which proves \eqref{eq:delta-bias-covariance-pointwise}. Taking the supremum over $t$ yields \eqref{eq:delta-bias-covariance-sup}.

  If $p_c=1$, then $K=1$ $P$-almost surely, so $P_{\mathrm{keep}}=P$ and $\Delta_{\mathrm{trim}}^{A}=0$.

  It remains to consider the case $p_c<1$. Define
  \begin{equation*}
    P_{\mathrm{drop}}
    :=
    P(\,\cdot\mid K=0)
    =
    P(\,\cdot\mid S(Z)>t^\star).
  \end{equation*}
  Then, for every $t\in\mathbb R$,
  \begin{equation*}
    F_P^A(t)
    =
    p_c F_{P_{\mathrm{keep}}}^{A}(t)
    +
    (1-p_c)F_{P_{\mathrm{drop}}}^{A}(t).
  \end{equation*}
  Therefore
  \begin{align*}
    F_{P_{\mathrm{keep}}}^{A}(t)-F_P^A(t)
    &=
    F_{P_{\mathrm{keep}}}^{A}(t)
    -
    \Bigl[
      p_c F_{P_{\mathrm{keep}}}^{A}(t)
      +
      (1-p_c)F_{P_{\mathrm{drop}}}^{A}(t)
    \Bigr]
    \\
    &=
    (1-p_c)
    \Bigl(
      F_{P_{\mathrm{keep}}}^{A}(t)
      -
      F_{P_{\mathrm{drop}}}^{A}(t)
    \Bigr).
  \end{align*}
  Taking the supremum over $t\in\mathbb R$ proves
  \eqref{eq:delta-bias-retained-discarded}.
\end{proof}

\begin{remark}[Interpretation of the covariance and retained--discarded identities]
  \label{rem:delta-bias-covariance}
  The main content of Proposition~\ref{prop:delta-bias-covariance} is structural. The covariance identity \eqref{eq:delta-bias-covariance-pointwise} and its supremum form \eqref{eq:delta-bias-covariance-sup} show that clean trimming distortion is driven by dependence under the clean law $P$. Specifically, it depends on the relation between the retention indicator $K=\mathbf 1\{S(Z)\le t^\star\}$ and the nonconformity events $\{A(Z)\le t\}$.

  When $p_c<1$, the retained--discarded identity \eqref{eq:delta-bias-retained-discarded} gives an equivalent and often more interpretable form
  \begin{equation*}
    \Delta_{\mathrm{trim}}^{A}
    =
    (1-p_c)
    d_{\mathrm{Kol}}
    \bigl(
      F_{P_{\mathrm{keep}}}^{A},
      F_{P_{\mathrm{drop}}}^{A}
    \bigr),
  \end{equation*}
  where
  \begin{equation*}
    d_{\mathrm{Kol}}(F,G)
    :=
    \sup_{t\in\mathbb R}|F(t)-G(t)|.
  \end{equation*}
  Thus, clean trimming distortion is small when little clean mass is discarded $(1-p_c\approx0)$, or when retained and discarded clean points have similar nonconformity-score distributions. Taking $d_{\mathrm{Kol}}\le1$ immediately gives the total-variation bound
  \begin{equation*}
      \Delta_{\mathrm{trim}}^{A}\le 1-p_c.
  \end{equation*}
  When $p_c=1$, this bound is trivial because $\Delta_{\mathrm{trim}}^{A}=0$.
\end{remark}

\begin{corollary}[Quantile stability induced by trimming]
  \label{cor:delta-bias-quantile}
  Let $F_P^{A}$ denote the cumulative distribution function of $A(Z)$ under $Z\sim P$, and let
  \begin{equation*}
    q_P(1-\alpha)
    :=
    \inf\{x\in\mathbb R:F_P^{A}(x)\ge 1-\alpha\}
  \end{equation*}
  be its lower generalised $(1-\alpha)$-quantile. Let
  \begin{equation*}
    q_{\mathrm{keep}}(1-\alpha)
    :=
    \inf\{x\in\mathbb R:F_{P_{\mathrm{keep}}}^{A}(x)\ge 1-\alpha\}
  \end{equation*}
  be the corresponding retained-clean quantile.

  Assume the setting of Lemma~\ref{lem:delta-bias-tv}. Further assume that $F_P^A$ is continuous at $q_P(1-\alpha)$. Also assume that there exist constants $\lambda_{\min}>0$ and $\rho>0$ such that the two-sided CDF increment condition in Lemma~\ref{lem:quantile-perturbation} holds with $F=F_P^A$ and $p=1-\alpha$. A sufficient condition is that $F_P^A$ is absolutely continuous and admits a density version satisfying
  \begin{equation*}
    f_P^{A}(x)\ge \lambda_{\min}
    \qquad\text{for Lebesgue-a.e. }x\in
    [q_P(1-\alpha)-\rho,\,q_P(1-\alpha)+\rho].
  \end{equation*}
  If
  \begin{equation*}
    \Delta_{\mathrm{trim}}^{A}< \lambda_{\min}\rho,
  \end{equation*}
  and, in particular, if the sufficient condition
  \begin{equation*}
    1-p_c< \lambda_{\min}\rho
  \end{equation*}
  holds, then
  \begin{equation}
    \label{eq:delta-bias-quantile-shift}
    \bigl|q_{\mathrm{keep}}(1-\alpha)-q_P(1-\alpha)\bigr|
    \le
    \frac{\Delta_{\mathrm{trim}}^{A}}{\lambda_{\min}}
    \le
    \frac{1-p_c}{\lambda_{\min}}.
  \end{equation}
\end{corollary}

\begin{proof}
  Apply Lemma~\ref{lem:quantile-perturbation} with
  \begin{equation*}
    F=F_P^{A},
    \qquad
    G=F_{P_{\mathrm{keep}}}^{A},
    \qquad
    p=1-\alpha.
  \end{equation*}
  By definition of $\Delta_{\mathrm{trim}}^{A}$,
  \begin{equation*}
    \sup_{t\in\mathbb R}
    \bigl|F_P^{A}(t)-F_{P_{\mathrm{keep}}}^{A}(t)\bigr|
    =
    \Delta_{\mathrm{trim}}^{A}.
  \end{equation*}
  By Lemma~\ref{lem:delta-bias-tv},
  \begin{equation*}
    \Delta_{\mathrm{trim}}^{A}\le 1-p_c.
  \end{equation*}
  Hence the sufficient condition $1-p_c < \lambda_{\min}\rho$ implies
  \begin{equation*}
      \Delta_{\mathrm{trim}}^{A}< \lambda_{\min}\rho .
  \end{equation*}
  Under this sufficient condition, or under the stated direct condition $\Delta_{\mathrm{trim}}^{A}< \lambda_{\min}\rho$, the hypotheses of Lemma~\ref{lem:quantile-perturbation} are satisfied. Consequently,
  \begin{equation*}
    \bigl|q_{\mathrm{keep}}(1-\alpha)-q_P(1-\alpha)\bigr|
    \le
    \frac{\Delta_{\mathrm{trim}}^{A}}{\lambda_{\min}}
    \le
    \frac{1-p_c}{\lambda_{\min}},
  \end{equation*}
  proving \eqref{eq:delta-bias-quantile-shift}.
\end{proof}

\begin{remark}[Moderate trimming, aggressive trimming, and width interpretation]
  \label{rem:delta-bias-quantile-interpretation}
  The condition
  \begin{equation*}
      \Delta_{\mathrm{trim}}^{A}<\lambda_{\min}\rho
      \quad\text{(or the sufficient condition }1-p_c<\lambda_{\min}\rho\text{)}
  \end{equation*}
  is a moderate-trimming requirement. It can fail under aggressive trimming, because the retained clean mass $p_c$) may become too small. The local density lower bound around $q_P(1-\alpha)$ then cannot control the quantile perturbation. In that regime, Corollary~\ref{cor:delta-bias-quantile} should be read as a diagnostic for mild-enough trimming and oracle-quantile stability, not as a blanket statement over all thresholds.

  For the common symmetric interval rule
  \begin{equation*}
    C_\alpha(x)
    =
    [\hat f(x)-\tau,\,\hat f(x)+\tau],
  \end{equation*}
  the population-oracle width is $2\tau$. Therefore,
  \eqref{eq:delta-bias-quantile-shift} immediately implies that, whenever the assumptions of the corollary hold, the oracle width perturbation induced by trimming is controlled by
  \begin{equation*}
    2\bigl|q_{\mathrm{keep}}(1-\alpha)-q_P(1-\alpha)\bigr|
    \le
    \frac{2\Delta_{\mathrm{trim}}^{A}}{\lambda_{\min}}.
  \end{equation*}
\end{remark}

\section{Proofs of main generic results}
\label{app:proofs}

\subsection{Proof of Proposition~\ref{prop:retained-law-transfer}}
\begin{proof}
  Work conditionally on $\mathcal G$. Then $A$, $S$, and $t^\star$ are fixed, and the contaminated calibration observations are i.i.d. from $P^\star$.  Let
  \begin{equation*}
    K_i=\mathbf 1\{S(Z_i)\le t^\star\},
    \qquad
    \mathcal I_{\mathrm{keep}}=\{i:K_i=1\}.
  \end{equation*}
  Each $K_i$ has success probability $\mu_{\mathrm{keep}}$, so $N_{\mathrm{keep}}\sim\operatorname{Binomial}(m,\mu_{\mathrm{keep}})$. For any deterministic index set $I=\{i_1<\cdots<i_n\}$ with positive probability, the event $\mathcal I_{\mathrm{keep}}=I$ factors coordinate wise into retained and discarded events.  Independence gives
  \begin{equation*}
    \mathcal L\bigl((Z_{i_1},\ldots,Z_{i_n})\mid \mathcal I_{\mathrm{keep}}=I,\mathcal G\bigr)
    =R^{\otimes n}.
  \end{equation*}
  Since the right-hand side does not depend on $I$, the same product-law statement holds after conditioning only on $N_{\mathrm{keep}}=n$. Thus, the retained observations are i.i.d. from $R$. Their scores are exchangeable with an independent retained-law test score $A(W_0)$, where $W_0\sim R$.

  If $n=0$, or if $\lceil(n+1)(1-\alpha)\rceil=n+1$, the procedure sets $\hat\tau_\alpha=+\infty$, and retained-law coverage is one. Otherwise let
  \begin{equation*}
    r_n=\lceil(n+1)(1-\alpha)\rceil\le n.
  \end{equation*}
  Attach auxiliary tie-breakers $U_0,U_1,\ldots,U_n\stackrel{\mathrm{i.i.d.}}{\sim}{\rm Unif}(0,1)$, independent of the scores. Order the pairs $(A(W_i),U_i)$ lexicographically. Let $\widetilde R_0$ be the rank of the test pair among the $n+1$ pairs. By exchangeability and the continuous tie-breakers,
  \begin{equation*}
    \mathbb P\{\widetilde R_0\le r_n\mid N_{\mathrm{keep}}=n,\mathcal G\}
    =\frac{r_n}{n+1}
    \ge 1-\alpha.
  \end{equation*}
  Let $\hat\tau_\alpha$ be the ordinary inclusive $r_n$-th order statistic of the $n$ retained calibration scores. We claim
  \begin{equation*}
    \{\widetilde R_0\le r_n\}\subseteq\{A(W_0)\le\hat\tau_\alpha\}.
  \end{equation*}
  Indeed, if $A(W_0)>\hat\tau_\alpha$, then at least $r_n$ retained calibration scores are no larger than $\hat\tau_\alpha$, and hence are strictly smaller than $A(W_0)$. Their lexicographic pairs all precede the test pair, so $\widetilde R_0\ge r_n+1$, a contradiction. Therefore
  \begin{equation*}
    \mathbb P_R\{A(W_0)\le\hat\tau_\alpha\mid N_{\mathrm{keep}}=n,\mathcal G\}
    \ge 1-\alpha.
  \end{equation*}
  Averaging over $N_{\mathrm{keep}}$ proves
  \begin{equation*}
    \mathbb E\{F_R^A(\hat\tau_\alpha)\mid\mathcal G\}
    =
    \mathbb P_R\{A(W_0)\le\hat\tau_\alpha\mid\mathcal G\}
    \ge 1-\alpha.
  \end{equation*}

  For the transfer statement, set
  \begin{equation*}
    \Delta_{R,P_0,+}:=\sup_t\bigl(F_R^A(t)-F_{P_0}^A(t)\bigr)_+.
  \end{equation*}
  Then $F_{P_0}^A(t)\ge F_R^A(t)-\Delta_{R,P_0,+}$ for all finite $t$, and also at $+\infty$. Hence
  \begin{equation*}
    \mathbb P_{P_0}\{A(Z_{\mathrm{new}})\le\hat\tau_\alpha\mid\mathcal G\}
    =\mathbb E\{F_{P_0}^A(\hat\tau_\alpha)\mid\mathcal G\}
    \ge
    \mathbb E\{F_R^A(\hat\tau_\alpha)\mid\mathcal G\}
    -\Delta_{R,P_0,+}
    \ge 1-\alpha-\Delta_{R,P_0,+}.
  \end{equation*}
\end{proof}

\subsection{Proof of Theorem~\ref{thm:fixed-threshold-trimmed-coverage}}
\begin{proof}[Proof of Theorem~\ref{thm:fixed-threshold-trimmed-coverage}]
    Work conditionally on $\mathcal{G}$, and suppress this conditioning in the notation. Then $A$, $S$, and $t^\star$ are fixed. The calibration observations are i.i.d.\ from $P^\star$, and the target test point is independent of the calibration sample. The retained law
    \begin{equation*}
        R = P^\star(\,\cdot \mid S(Z) \le t^\star\,)
    \end{equation*}
    is well-defined because $\mu_{\rm keep} > 0$. We prove the identity for an arbitrary independent target law $P_0$, and recover the stated case by taking $P_0 = P$.
    
    By the thinning argument in the proof of Proposition~\ref{prop:retained-law-transfer}, conditional on $N_{\rm keep} = n$, the retained scores have the same joint law as
    \begin{equation*}
      A(W_1), \ldots, A(W_n), \qquad W_1, \ldots, W_n \overset{\mathrm{iid}}{\sim} R,
    \end{equation*}
    Moreover, $N_{\rm keep} \sim \mathrm{Binomial}(m, \mu_{\rm keep})$.
    
    Fix $n$. If $n = 0$ or $r_n = n + 1$, the algorithm sets $\hat\tau_\alpha = +\infty$, and
    \begin{equation*}
      \mathbb{P}_{P_0}\{A(Z_{\rm new}) \le \hat\tau_\alpha \mid N_{\rm keep} = n\} = 1.
    \end{equation*}
    This is the first branch of $\Psi_n(P_0, R)$.
    
    Suppose $1 \le r_n \le n$. Let $U_1, \ldots, U_n \overset{\mathrm{iid}}{\sim} \mathrm{Unif}(0,1)$ and set $X_i = Q_R^A(U_i)$. By the generalised inverse transform, $X_i \sim F_R^A$. Since $Q_R^A$ is non decreasing,
    \begin{equation*}
      X_{(r_n)} \;=\; Q_R^A(U_{(r_n)})
    \end{equation*}
    almost surely, even when $F_R^A$ has atoms. The order statistic $U_{(r_n)} \sim \mathrm{Beta}(r_n, n + 1 - r_n)$, so the inclusive retained cutoff conditional on $N_{\rm keep} = n$ has the same distribution as $Q_R^A(B_{r_n, n+1-r_n})$. Conditioning on the cutoff and using independence of the test point,
    \begin{equation*}
      \mathbb{P}_{P_0}\{A(Z_{\rm new}) \le \hat\tau_\alpha \mid \hat\tau_\alpha, N_{\rm keep} = n\}
      = F_{P_0}^A(\hat\tau_\alpha),
    \end{equation*}
    hence
    \begin{equation*}
      \mathbb{P}_{P_0}\{A(Z_{\rm new}) \le \hat\tau_\alpha \mid N_{\rm keep} = n\}
      = \mathbb{E}\bigl[F_{P_0}^A\{Q_R^A(B_{r_n, n+1-r_n})\}\bigr].
    \end{equation*}
    
    Combining the two branches and averaging over $N_{\rm keep} \sim \mathrm{Binomial}(m, \mu_{\rm keep})$ yields the target-law identity. Taking $P_0 = P$ gives \eqref{eq:exact-clean-coverage-identity}.
\end{proof}

\begin{remark}[Mixture upper bound on the retained-to-clean discrepancy]
\label{rem:mixture-upper-bound}
The componentwise upper bound used in Section~\ref{subsec:exact-scalar-mixture} motivates the clean-side and dirty-side decomposition. It follows from the retained-law identity \eqref{eq:retained-law-decomposition} by a short triangle-inequality argument, recorded here for reference.

Assume $p_c > 0$. Bayes' rule gives the retained-law decomposition $R = (1 - \tilde\varepsilon_\star) P_{\rm keep} + \tilde\varepsilon_\star Q_{\rm keep}$, with $Q_{\rm keep}$ arbitrary if $p_d = 0$. Therefore
\begin{equation*}
  F_R^A(t) - F_P^A(t)
  = (1 - \tilde\varepsilon_\star)\bigl(F_{P_{\rm keep}}^A(t) - F_P^A(t)\bigr)
  + \tilde\varepsilon_\star\bigl(F_{Q_{\rm keep}}^A(t) - F_P^A(t)\bigr),
\end{equation*}
and applying $(\cdot)_+$ pointwise followed by the supremum yields
\begin{equation}
  \label{eq:mixture-upper-bound}
  d_{R \to P,+}^A
  \le (1 - \tilde\varepsilon_\star) \Delta_{\rm trim,+}^A
  + \tilde\varepsilon_\star\, D_{Q,+}.
\end{equation}
The two-sided variant follows analogously, with $\Delta_{\rm trim}^A$ and $D_Q$ replacing their one-sided counterparts. Combining \eqref{eq:mixture-upper-bound} with Proposition~\ref{prop:finite-sample-scalar-bound} gives the population diagnostic lower bound stated in Section~\ref{subsec:exact-scalar-mixture}. This bound is also used in the operational certificate of Proposition~\ref{prop:operational-certificate}.
\end{remark}

\subsection{Proof of Proposition~\ref{prop:finite-sample-scalar-bound}}
\begin{proof}
  Work conditionally on $\mathcal G$, and write $d=d_{R\to P,+}^{A}$. By definition of $d$, for every score threshold $t$,
  \begin{equation*}
    F_P^A(t)\ge F_R^A(t)-d.
  \end{equation*}
  Since $F_R^A(Q_R^A(u))\ge u$ for the lower generalised quantile, this implies, for every $u\in(0,1)$,
  \begin{equation*}
    F_P^A\{Q_R^A(u)\}
    \ge
    u-d.
  \end{equation*}
  Because a CDF is non negative, $F_P^A\{Q_R^A(u)\}\ge0$ as well. Hence
  \begin{equation*}
    F_P^A\{Q_R^A(u)\}
    \ge
    (u-d)_+.
  \end{equation*}
  Applying this pointwise inequality to the exact identity \eqref{eq:exact-clean-coverage-identity} gives the first inequality in \eqref{eq:finite-sample-scalar-bound}. On the branch $r_n=n+1$, the cutoff is $+\infty$, and both the exact conditional coverage contribution and $\psi_n(d)$ equal one. On the branch $r_n\le n$, the exact contribution is
  \begin{equation*}
    \mathbb E\!\left[F_P^A\{Q_R^A(B_{r_n,n+1-r_n})\}\right]
    \ge
    \mathbb E\bigl[(B_{r_n,n+1-r_n}-d)_+\bigr].
  \end{equation*}
  Averaging over $N_{\rm keep}\sim{\rm Binomial}(m,\mu_{\rm keep})$ proves the first bound.

  The closed form for $\psi_n(d)$ follows from
  \begin{equation*}
    \mathbb E[(B-d)_+]
    =\mathbb E[B\mathbf 1\{B>d\}]-d\,\mathbb P(B>d),
  \end{equation*}
  with $B\sim {\rm Beta}(a,b)$, together with
  \begin{equation*}
    \mathbb P(B>d)=1-I_d(a,b),
    \qquad
    \mathbb E[B\mathbf 1\{B>d\}]
    =\frac{a}{a+b}\{1-I_d(a+1,b)\}.
  \end{equation*}
  Taking $a=r_n$ and $b=n+1-r_n$ gives the displayed expression.

  For the second bound, if $r_n=n+1$, then $\psi_n(d)=1\ge[1-\alpha-d]_+$. If $r_n\le n$, then
  \begin{equation*}
    \psi_n(d)
    =\mathbb E[(B_{r_n,n+1-r_n}-d)_+]
    \ge
    \mathbb E[B_{r_n,n+1-r_n}]-d
    =
    \frac{r_n}{n+1}-d
    \ge
    1-\alpha-d.
  \end{equation*}
  Since $\psi_n(d)\ge0$, this gives $\psi_n(d)\ge[1-\alpha-d]_+$ for every $n$. The binomial weights sum to one, so $L_{\rm fs}(d)\ge[1-\alpha-d]_+$.
\end{proof}

\subsection{Proof of Theorem~\ref{thm:one-sided-transfer-sharpness}}
\begin{proof}
  Let $p=1-\alpha$. We prove the minimax identity in two directions.

  Take any $(R^A,P^A)\in\mathfrak C_d$ and put
  \begin{equation*}
    d_0:=\sup_t\bigl(F_R^A(t)-F_P^A(t)\bigr)_+\le d .
  \end{equation*}
  The proof of Proposition~\ref{prop:finite-sample-scalar-bound} is a score-law argument and applies directly to the abstract experiment defining $\mathcal C_{m,\mu}$. It gives
  \begin{equation*}
    \mathcal C_{m,\mu}(R^A,P^A)
    \ge L_{\rm fs}^{m,\mu}(d_0).
  \end{equation*}
  For every non degenerate retained count, $\psi_n(x)=\mathbb E[(B_{r_n,n+1-r_n}-x)_+]$ is non increasing in $x$; in the degenerate branch it is identically one. Hence $L_{\rm fs}^{m,\mu}(x)$ is non increasing, and therefore
  \begin{equation*}
    \mathcal C_{m,\mu}(R^A,P^A)
    \ge L_{\rm fs}^{m,\mu}(d_0)
    \ge L_{\rm fs}^{m,\mu}(d).
  \end{equation*}
  Taking the infimum over $\mathfrak C_d$ proves the ``$\ge$'' direction of \eqref{eq:sharpness-minimax-form}.

  Take $R_d^A$ to be uniform on $[0,1]$, so that
  \begin{equation*}
    F_{R_d}^A(t)=
    \begin{cases}
      0, & t<0,\\
      t, & 0\le t<1,\\
      1, & t\ge1.
    \end{cases}
  \end{equation*}
  Define the continuous target score law
  \begin{equation*}
    F_{P_d}^A(t)=
    \begin{cases}
      0, & t<d,\\
      t-d, & d\le t<1,\\
      1-d+d(t-1), & 1\le t<2,\\
      1, & t\ge2.
    \end{cases}
  \end{equation*}
  This is a valid continuous CDF: it has density one on $[d,1]$, density $d$ on $[1,2]$, and total mass $(1-d)+d=1$. For $t\in[d,1]$,
  \begin{equation*}
    F_{R_d}^A(t)-F_{P_d}^A(t)=d.
  \end{equation*}
  For $t\in[0,d)$, the difference is $t\le d$; for $t\in[1,2]$, the difference is $d(2-t)\le d$; and outside these intervals it is non positive or zero. Hence
  \begin{equation*}
    \sup_t\{F_{R_d}^A(t)-F_{P_d}^A(t)\}_+=d,
  \end{equation*}
  so $(R_d^A, P_d^A)\in\mathfrak C_d$ with equality in the discrepancy constraint.

  Now condition on $N=n$. If $r_n = n + 1$, the cutoff is $+\infty$ and the coverage contribution is one, equal to $\psi_n(d)$. If $r_n\le n$, the cutoff is the $r_n$-th order statistic of $n$ i.i.d. uniform variables, so
  \begin{equation*}
    \widehat q_n\sim {\rm Beta}(r_n,n+1-r_n).
  \end{equation*}
  Since $\widehat q_n\in[0,1]$ almost surely,
  \begin{equation*}
    F_{P_d}^A(\widehat q_n)=(\widehat q_n-d)_+,
  \end{equation*}
  and the conditional target coverage is
  \begin{equation*}
    \mathbb E[(B_{r_n,n+1-r_n}-d)_+]=\psi_n(d).
  \end{equation*}
  Averaging these exact conditional contributions over $N\sim{\rm Binomial}(m,\mu)$ yields
  \begin{equation*}
    \mathcal C_{m,\mu}(R_d^A,P_d^A)=L_{\rm fs}^{m,\mu}(d).
  \end{equation*}
  This proves the ``$\le$'' direction and shows that the infimum is attained.

  Because $\mu\in(0,1]$ is fixed, $N\to\infty$ in probability as $m\to\infty$. Conditional on $N = n\to\infty$, the beta order statistic $B_{r_n,n+1-r_n}$ converges in probability to $p$, since $r_n/(n+1)\to p$. The variables $(B_{r_n,n+1-r_n}-d)_+$ are bounded by one. Bounded convergence therefore gives convergence of the binomial mixture expectation to $(p-d)_+$. The fallback probability tends to zero, and hence
  \begin{equation*}
    L_{\rm fs}^{m,\mu}(d)\to[p-d]_+=[(1-\alpha)-d]_+.
  \end{equation*}
\end{proof}

\subsection{Same-sample threshold selection}
\label{app:same-sample-selection}

The exact retained-law identity in Theorem~\ref{thm:fixed-threshold-trimmed-coverage} requires the threshold $t^\star$ to be fixed before the final contaminated calibration sample is used for the conformal quantile. If the threshold is selected from the same calibration data, the retained sample no longer has the fixed-threshold product law after conditioning on the selected threshold. The rank argument in Proposition~\ref{prop:retained-law-transfer} therefore no longer applies.

The following result gives a finite-grid, selection-aware alternative. It is a high-probability empirical-CDF transfer bound that holds uniformly over all candidate thresholds. It therefore allows arbitrary data-dependent selection rules, at the cost of a $\sqrt{\log K / N_t}$ grid-search penalty. The result is only a certificate template: the population losses $d_t$ must still be bounded independently to obtain a numerical guarantee.

\begin{proposition}[Finite-grid same-sample threshold selection]
  \label{prop:same-sample-threshold-selection}
  Condition on $\mathcal{G}$. Let $\mathcal{T} = \{t_1, \ldots, t_K\} \subset \overline{\mathbb{R}}$ be a finite set of candidate thresholds such that $\mu_t := P^\star(S \le t) > 0$ for every $t \in \mathcal{T}$. For each $t \in \mathcal{T}$, define
  \begin{equation*}
    I_t = \{i : S(Z_i) \le t\}, \qquad
    N_t = |I_t|, \qquad
    R_t = P^\star(\,\cdot \mid S \le t\,),
  \end{equation*}
  and let $\hat\tau_t$ be the trimmed conformal cutoff computed from $I_t$ by the inclusive quantile rule of Algorithm~\ref{alg:fixed-threshold-trimmed-cp}. Define the per-threshold retained-to-clean loss
  \begin{equation*}
    d_t = d_{R_t \to P,+}^A
        = \sup_a \bigl(F_{R_t}^A(a) - F_P^A(a)\bigr)_+ .
  \end{equation*}
  Fix $\beta\in(0,1)$. Set $r_t = \lceil (N_t + 1)(1 - \alpha) \rceil$ for all $t$, and, when $N_t \ge 1$, set
  \begin{equation*}
    \eta_t = \sqrt{\frac{\log(2K/\beta)}{2 N_t}}.
  \end{equation*}
  Let $\widehat t$ be any data-dependent rule taking values in $\mathcal{T}$. It may be chosen after inspecting all calibration scores and all candidate cutoffs.

  Then, with conditional probability at least $1 - \beta$ over the final contaminated calibration sample, the following holds simultaneously for every $t \in \mathcal{T}$ with $N_t \ge 1$ and $r_t \le N_t$,
  \begin{equation}
    \label{eq:same-sample-selection-uniform}
    \mathbb{P}_P\bigl\{ A(Z_{\rm new}) \le \hat\tau_t \big| Z_1, \ldots, Z_m, \mathcal{G} \bigr\}
    \ge
    \frac{r_t}{N_t} - d_t - \eta_t .
  \end{equation}
  For thresholds with $N_t = 0$ or $r_t = N_t + 1$, the algorithm sets $\hat\tau_t = +\infty$, and the conditional clean coverage equals one. On the same event, the bound \eqref{eq:same-sample-selection-uniform} holds at $t = \widehat t$ whenever the selected threshold is non degenerate. If it is degenerate, the conditional clean coverage is one. No independence between $\widehat t$ and the calibration sample is required, because the event holds uniformly over $\mathcal{T}$.
\end{proposition}

The loss $d_t$ is a population retained-to-clean discrepancy for the candidate threshold. It is not directly observable. Proposition~\ref{prop:same-sample-threshold-selection} is therefore selection-aware for the empirical retained CDF, but it is not by itself a data-only coverage certificate. To obtain a numerical guarantee, the losses $\{d_t\}_{t \in \mathcal{T}}$ must be bounded using independent information. This can be done, for example, with a clean audit sample (Proposition~\ref{prop:exact-audit-certificate}) or with simultaneous independent CDF bounds over the grid.

\begin{proof}[Proof of Proposition~\ref{prop:same-sample-threshold-selection}]
Work conditionally on $\mathcal G$ throughout. Then the candidate grid is fixed, and $Z_1,\ldots,Z_m$ are i.i.d. from $P^\star$.

Fix $t\in\mathcal T$. Since $\mu_t>0$, the retained law
\begin{equation*}
    R_t=P^\star(\cdot\mid S\le t)
\end{equation*}
is well-defined. Conditional on $N_t=n\ge1$, the retained observations $\{Z_i:i\in I_t\}$ are i.i.d. from $R_t$ by independent thinning. Hence their retained-score empirical CDF
\begin{equation*}
  \widehat F_t^A(a) := \frac{1}{N_t}\sum_{i\in I_t}\mathbf 1\{A(Z_i)\le a\},
  \qquad N_t\ge1,
\end{equation*}
satisfies, for every $n\ge1$,
\begin{align*}
  &\mathbb P\left(
    \sup_a\left|\widehat F_t^A(a)-F_{R_t}^A(a)\right|
    >\sqrt{\frac{\log(2K/\beta)}{2n}}
    \,\Bigm|\, N_t=n,\mathcal G
  \right) \\
  &\hspace{7em}\le
  2\exp\left\{-2n\cdot\frac{\log(2K/\beta)}{2n}\right\}
  =\frac{\beta}{K},
\end{align*}
by the two-sided Dvoretzky--Kiefer--Wolfowitz inequality with Massart's constant.
Define the per-threshold good event
\begin{equation*}
  \mathcal E_t
  :=
  \{N_t=0\}
  \cup
  \left\{
    N_t\ge1,
    \sup_a\left|\widehat F_t^A(a)-F_{R_t}^A(a)\right|
    \le
    \sqrt{\frac{\log(2K/\beta)}{2N_t}}
  \right\}.
\end{equation*}
Equivalently, $\mathcal E_t^c$ is the event where $N_t\ge1$ and the empirical-CDF deviation exceeds the random threshold $\eta_t$. Averaging the conditional DKW bound over the possible positive retained counts gives
\begin{align*}
  \mathbb P(\mathcal E_t^c\mid\mathcal G)
  &=
  \sum_{n=1}^m
  \mathbb P(N_t=n\mid\mathcal G)
  \mathbb P\left(
    \sup_a\left|\widehat F_t^A(a)-F_{R_t}^A(a)\right|
    >\sqrt{\frac{\log(2K/\beta)}{2n}}
    \,\Bigm|\, N_t=n,\mathcal G
  \right) \\
  &\le
  \frac{\beta}{K}\,\mathbb P(N_t\ge1\mid\mathcal G)
  \le
  \frac{\beta}{K}.
\end{align*}
This is where the zero-retained-count case must be included in the good event. Otherwise, a threshold with $N_t=0$ would incorrectly make $\mathcal E_t$ fail.

Let
\begin{equation*}
  \mathcal E := \bigcap_{t\in\mathcal T}\mathcal E_t.
\end{equation*}
By the preceding display and the union bound,
\begin{equation*}
  \mathbb P(\mathcal E^c\mid\mathcal G)
  \le
  \sum_{t\in\mathcal T}\mathbb P(\mathcal E_t^c\mid\mathcal G)
  \le
  K\cdot\frac{\beta}{K}
  =\beta.
\end{equation*}
Thus $\mathcal E$ holds with conditional probability at least $1-\beta$. On $\mathcal E$, for every candidate threshold with $N_t\ge1$,
\begin{equation}
  \label{eq:same-sample-dkw-on-event}
  \sup_a\left|\widehat F_t^A(a)-F_{R_t}^A(a)\right|\le \eta_t.
\end{equation}

Work on $\mathcal E$, and fix any $t\in\mathcal T$ with $N_t\ge1$ and $r_t\le N_t$. The cutoff $\hat\tau_t$ is the $r_t$-th order statistic of the $N_t$ retained scores. Hence the inclusive empirical CDF satisfies
\begin{equation*}
  \widehat F_t^A(\hat\tau_t)\ge \frac{r_t}{N_t}.
\end{equation*}
Combining this with \eqref{eq:same-sample-dkw-on-event},
\begin{equation*}
  F_{R_t}^A(\hat\tau_t)
  \ge
  \widehat F_t^A(\hat\tau_t)-\eta_t
  \ge
  \frac{r_t}{N_t}-\eta_t.
\end{equation*}
By definition of $d_t$, for every finite score threshold $a$,
\begin{equation*}
  F_P^A(a)\ge F_{R_t}^A(a)-d_t.
\end{equation*}
Therefore
\begin{equation*}
  F_P^A(\hat\tau_t)
  \ge
  F_{R_t}^A(\hat\tau_t)-d_t
  \ge
  \frac{r_t}{N_t}-d_t-\eta_t.
\end{equation*}
Given $\mathcal G$, the clean test point is independent of the calibration sample. The cutoff $\hat\tau_t$ is measurable with respect to $Z_1,\ldots,Z_m$ and $\mathcal G$. Hence
\[
  \mathbb{P}_P\bigl\{ A(Z_{\rm new}) \le \hat\tau_t \,\big|\, Z_1, \ldots, Z_m, \mathcal{G} \bigr\}
  =F_P^A(\hat\tau_t),
\]
which proves \eqref{eq:same-sample-selection-uniform}.

If $N_t=0$ or $r_t=N_t+1$, the algorithm sets $\hat\tau_t=+\infty$, and the conditional clean coverage is $F_P^A(+\infty)=1$. The event $\mathcal E$  is defined before selecting a threshold and holds uniformly over the finite grid. Hence, after $\mathcal E$ occurs, any measurable rule $\widehat t\in\mathcal T$ inherits the corresponding bound at its selected threshold when the selected branch is non degenerate. If it is degenerate, it inherits the trivial coverage-one statement. This holds regardless of how $\widehat t$ depends on the calibration sample.
\end{proof}

%%%%%%%%%%%%%%%%%%%%%%%%%%%%%%%%%%%%%%%%%%%%%%%%%%%%%%%%%%%%%%%%%%%%%%%%%

\begin{corollary}[Finite-library score-threshold selection]
\label{cor:finite_library}
Condition on $\mathcal G$.
Let
\begin{equation*}
    \mathcal H=\{(S_h,t_h):h=1,\ldots,H\}
\end{equation*}
be a finite library of measurable score-threshold pairs, fixed conditionally on $\mathcal G$. Assume that the nonconformity score $A$ is also fixed conditionally on $\mathcal G$, and that
\begin{equation*}
    \mu_h=P^\star(S_h\leq t_h)>0
\end{equation*}
for every $h$. For each candidate, let $N_h$ be the retained count, $\widehat\tau_h$ the inclusive conformal cutoff, and define
\begin{equation*}
    R_h=P^\star(\cdot\mid S_h\leq t_h),
    \qquad
    d_h=d^A_{R_h\to P,+}.
\end{equation*}

Fix $\beta\in(0,1)$. For each $h$ with $N_h\geq 1$, set
\begin{equation*}
    r_h=\left\lceil(N_h+1)(1-\alpha)\right\rceil,
    \qquad
    \eta_h=
    \sqrt{\frac{\log(2H/\beta)}{2N_h}}.
\end{equation*}

Then, with conditional probability at least $1-\beta$ over the final calibration sample, simultaneously for all candidates satisfying $N_h\geq 1$ and $r_h\leq N_h$,
\begin{equation*}
    \mathbb P_P\left( A(Z_{\mathrm{new}})\leq\widehat\tau_h \mid Z_{1:m},\mathcal G \right)
    \geq
    \frac{r_h}{N_h}-d_h-\eta_h.
\end{equation*}

Consequently, the same inequality holds for any data-dependent selection
$\widehat h\in\{1,\ldots,H\}$ whose selected cutoff is finite. For a selected candidate with cutoff $+\infty$, the conditional coverage equals one.
\end{corollary}

\begin{proof}
Fix a candidate $h$.
Conditional on its retained count $N_h=n\geq 1$,
the retained observations are i.i.d. from $R_h$.

The Dvoretzky--Kiefer--Wolfowitz inequality gives
\begin{equation*}
    \mathbb P\left( \sup_a \left| \widehat F_h^A(a)-F_{R_h}^A(a) \right| > \eta_h \mid N_h=n,\mathcal G \right)
    \leq
    \frac{\beta}{H}.
\end{equation*}

Averaging over the retained count and applying a union bound over the $H$ candidates gives a simultaneous event with probability at least $1-\beta$.

On this event, every finite selected cutoff obeys
\begin{equation*}
    F_{R_h}^A(\widehat\tau_h)
    \geq
    \widehat F_h^A(\widehat\tau_h)-\eta_h
    \geq
    \frac{r_h}{N_h}-\eta_h.
\end{equation*}

By the definition of $d_h$,
\begin{equation*}
    F_P^A(\widehat\tau_h)
    \geq
    F_{R_h}^A(\widehat\tau_h)-d_h.
\end{equation*}

Combining these inequalities proves the result. The event holds for the entire library before selection, so it also holds for any data-dependent choice from the library.
\end{proof}

\subsection{Proof of Proposition~\ref{prop:trimming-bias-orthogonality}}
\begin{proof}[Proof of Proposition~\ref{prop:trimming-bias-orthogonality}]
    The proof has three parts: the covariance identity itself, the homoscedastic specialisation, and the approximate residual bound under model misspecification.

    For any $t \in \mathbb{R}$,
    \begin{equation*}
      F_{P_{\mathrm{keep}}}^A(t)
      =
      P\bigl(A(Z) \le t \big| K = 1\bigr)
      =
      \frac{\mathbb{E}_P\bigl[\mathbf{1}\{A(Z) \le t\} K\bigr]}{p_c},
    \end{equation*}
    hence
    \begin{align*}
      F_{P_{\mathrm{keep}}}^A(t) - F_P^A(t)
      &=
      \frac{\mathbb{E}_P[\mathbf{1}\{A(Z) \le t\}\,K]}{p_c}
      - \mathbb{E}_P[\mathbf{1}\{A(Z) \le t\}] \\
      &=
      \frac{\mathrm{Cov}_P\bigl(\mathbf{1}\{A(Z) \le t\},\,K\bigr)}{p_c},
    \end{align*}
    which establishes the pointwise covariance identity. Taking positive parts and the supremum over $t$ gives the one-sided form $\Delta_{\mathrm{trim},+}^A = p_c^{-1}\sup_t\{\mathrm{Cov}_P(\mathbf{1}\{A(Z) \le t\}, K)\}_+$, and taking absolute values gives the two-sided form $\Delta_{\mathrm{trim}}^A = p_c^{-1}\sup_t |\mathrm{Cov}_P(\mathbf{1}\{A(Z) \le t\}, K)|$. If $A(Z)$ and $K$ are independent under $P$, every covariance in these displays vanishes, so $F_{P_{\mathrm{keep}}}^A \equiv F_P^A$ and $\Delta_{\mathrm{trim}}^A = \Delta_{\mathrm{trim},+}^A = \Delta_{\mathrm{trim},-}^A = 0$.
 
    Suppose $Y = f^\star(X) + \sigma\xi$ with $\xi \ind X$, $\hat f = f^\star$, $A(X,Y) = |Y - \hat f(X)|$, and $S$ covariate-only. Then
    \begin{equation*}
      A(X, Y) = |Y - \hat f(X)| = \sigma|\xi|,
    \end{equation*}
    which depends on $X$ only through $\xi$. Since $\xi \ind X$, the variable $A(X, Y)$ is independent of any covariate-only retention indicator $K = \mathbf{1}\{S_X(X) \le t^\star\}$ under $P$. Then gives $\Delta_{\mathrm{trim}}^A = 0$.

    Now suppose only that $Y = f^\star(X) + \sigma\xi$ with $\xi \ind X$, but allow $\hat f \ne f^\star$ with $\|f^\star - \hat f\|_\infty \le \eta$. Define $b(X) := f^\star(X) - \hat f(X)$ and $A_0 := \sigma|\xi|$, so that $\|b\|_\infty \le \eta$ and
    \begin{equation*}
      A(X, Y) \;=\; |\sigma\xi + b(X)|.
    \end{equation*}
    The reverse triangle inequality gives $|A - A_0| \le |b(X)| \le \eta$. Under the same covariate-only retention condition, $A_0 \ind K$ under $P$, so the law of $A_0$ is unaffected by conditioning on $\{K = 1\}$:
    \begin{equation*}
      F_{P_{\mathrm{keep}}}^{A_0} \;\equiv\; F_P^{A_0}.
    \end{equation*}
    For any $t$, the inequality $A \le t$ implies $A_0 \le t + \eta$, and conversely $A_0 \le t - \eta$ implies $A \le t$. Hence
    \begin{equation*}
      F_{P_{\mathrm{keep}}}^A(t)
      \;\le\;
      F_{P_{\mathrm{keep}}}^{A_0}(t + \eta)
      \;=\;
      F_P^{A_0}(t + \eta),
      \qquad
      F_P^A(t)
      \;\ge\;
      F_P^{A_0}(t - \eta),
    \end{equation*}
    which give
    \begin{equation*}
      F_{P_{\mathrm{keep}}}^A(t) - F_P^A(t)
      \;\le\;
      F_P^{A_0}(t + \eta) - F_P^{A_0}(t - \eta)
      \;=\;
      \mathbb{P}\{A_0 \in [t - \eta, t + \eta]\}.
    \end{equation*}
    Reversing the roles of $P_{\mathrm{keep}}$ and $P$ in the same chain of inequalities gives the symmetric upper bound on $F_P^A(t) - F_{P_{\mathrm{keep}}}^A(t)$. Taking the supremum over $t$ yields the residual smoothness bound
    \begin{equation*}
        \Delta_{\mathrm{trim}}^A
        \le
        \omega_{A_0}(\eta)
        :=
        \sup_t \mathbb{P}{A_0 \in [t-\eta,t+\eta]}.
    \end{equation*}
    If $A_0$ has density bounded by $M$, every interval of length $2\eta$ has probability at most $2M\eta$. Hence $\omega_{A_0}(\eta)\le 2M\eta$.
\end{proof}

\subsection{Proof of Proposition~\ref{thm:separation-dichotomy-main}}
\label{app:proof-separation-dichotomy-main}

\begin{proof}[Proof of Proposition~\ref{thm:separation-dichotomy-main}]
Work conditionally on $\mathcal G$. Since $p_c>0$ and $\varepsilon\in[0,1)$, the denominator in
\begin{equation*}
  \tilde\varepsilon_\star
  =
  \frac{\varepsilon p_d}{(1-\varepsilon)p_c+\varepsilon p_d}
\end{equation*}
is positive. For fixed $p_c>0$ and $\varepsilon\in[0,1)$, define
\begin{equation*}
  g(x)=\frac{\varepsilon x}{(1-\varepsilon)p_c+\varepsilon x},
  \qquad x\ge0.
\end{equation*}
Then
\begin{equation*}
  g'(x)
  =
  \frac{\varepsilon(1-\varepsilon)p_c}
  {\{(1-\varepsilon)p_c+\varepsilon x\}^2}
  \ge0,
\end{equation*}
so $g$ is nondecreasing in the dirty retention probability. If $p_d\le\lambda p_c$, then
\begin{equation*}
  \tilde\varepsilon_\star
  =g(p_d)
  \le g(\lambda p_c)
  =
  \frac{\varepsilon\lambda p_c}
  {(1-\varepsilon)p_c+\varepsilon\lambda p_c}
  =
  \frac{\varepsilon\lambda}{1-\varepsilon+\varepsilon\lambda}.
\end{equation*}
In particular, if $p_d/p_c\to0$, take $\lambda=p_d/p_c$ to obtain $\tilde\varepsilon_\star\to0$. If $p_d=p_c$, then
\begin{equation*}
  \tilde\varepsilon_\star
  =
  \frac{\varepsilon p_c}{(1-\varepsilon)p_c+\varepsilon p_c}
  =\varepsilon.
\end{equation*}
This proves the proposition.
\end{proof}

    \subsection{Proofs for independent adaptive threshold selection}
        \begin{proof}
    Work conditionally on $\mathcal G\vee\mathcal H$. Under the theorem’s assumptions, $A$, $S$, and $\widehat t$ are fixed after this conditioning. The contaminated calibration observations remain i.i.d. from $P^\star$ and independent of the clean test point.

    Therefore, Theorem~\ref{thm:fixed-threshold-trimmed-coverage} applies verbatim with $t^\star=\widehat t$. It gives the untruncated bound $1-\alpha-d_{R_{\widehat t}\to P,+}^{A}$. Non-negativity of coverage gives the positive-part version stated in the theorem, on the event $\mu_{\mathrm{keep}}(\widehat t)>0$.

    The componentwise mixture upper-bound diagnostic additionally requires $p_c(\widehat t)>0$. The exact identity also follows, using generalised retained-score quantiles and without any continuity requirement.
\end{proof}

        \begin{proof}
  Work on \(\mathcal E_{\rm tune}\cap\{\widehat{\mathcal T}_\eta\ne\varnothing\}\).  Since \(\widehat t\in\widehat{\mathcal T}_\eta\), \(\widehat B(\widehat t)\le\eta\).  The uniform tuning error gives
  \[
    B(\widehat t)\le\widehat B(\widehat t)+r_B\le\eta+r_B,
  \]
  proving \eqref{eq:oracle-coverage-loss-control}.  Since \(d_{R_{\widehat t}\to P,+}^{A}\le B(\widehat t)\), Theorem~\ref{thm:independent-adaptive-threshold} yields the stated conditional coverage bound on this event.

  For efficiency, let \(t\in\mathcal T\) satisfy \(B(t)\le\eta-r_B\).  On \(\mathcal E_{\rm tune}\), \(\widehat B(t)\le B(t)+r_B\le\eta\), so \(t\in\widehat{\mathcal T}_\eta\).  Therefore \(\widehat W(\widehat t)\le\widehat W(t)\).  Applying the uniform efficiency error twice gives
  \[
    W(\widehat t)
    \le \widehat W(\widehat t)+r_W
    \le \widehat W(t)+r_W
    \le W(t)+2r_W.
  \]
  Taking the infimum over \(t\) with \(B(t)\le\eta-r_B\) proves \eqref{eq:oracle-efficiency-control}.

  For the marginal statement, on \(\mathcal E_{\rm tune}\cap\{\widehat{\mathcal T}_\eta\ne\varnothing\}\) the conditional coverage is at least \([1-\alpha-\eta-r_B]_+\).  The complement of this event has probability at most \(\beta+\beta_\varnothing\), and coverage probabilities are non-negative. Therefore, the marginal coverage is at least
  \[
    [1-\beta-\beta_\varnothing]_+[1-\alpha-\eta-r_B]_+
    \ge 1-\alpha-\eta-r_B-\beta-\beta_\varnothing.
  \]
  If \(\widehat{\mathcal T}_\eta\) is empty, the proposition explicitly falls back to \(t_{\rm fb}\). Theorem~\ref{thm:independent-adaptive-threshold} still gives the fixed-threshold guarantee for this fallback. A certified fallback loss can then replace the crude probability subtraction.
\end{proof}

        \subsection{Componentwise certificate interface}
\label{app:operational-certificate}

The diagnostic bounds in Section~\ref{subsec:exact-scalar-mixture} are population quantities for the realised auxiliary stage. They become numerical certificates only when their ingredients are bounded by independent information. The following proposition formalises the componentwise route. Given independent bounds on the retention probabilities, the clean trimming distortion, and the retained dirty discrepancy, it converts the mixture upper-bound diagnostic into a deterministic clean-coverage lower bound. The result is stated in one-sided form to match the scalar clean-transfer theorem. It also allows the contamination fraction to be known only through an upper bound.

\begin{proposition}[One-sided certificate interface]
  \label{prop:operational-certificate}
  Under the setup of Section~\ref{subsec:trimmed-setup}, suppose that $\varepsilon \le \varepsilon_{\max} < 1$ and that deterministic or $\mathcal{G}$-measurable quantities
  \[
    L_c, \quad U_d, \quad B_{\Delta,+}, \quad B_{Q,+}
  \]
  satisfy
  \[
    0 < L_c \le p_c, \qquad
    0 \le p_d \le U_d, \qquad
    \Delta_{\mathrm{trim},+}^A \le B_{\Delta,+}, \qquad
    D_{Q,+} \le B_{Q,+} \le 1.
  \]
  Define
  \[
    \bar\varepsilon
    \;:=\;
    \frac{\varepsilon_{\max} U_d}{(1 - \varepsilon_{\max}) L_c + \varepsilon_{\max} U_d}.
  \]
  Then, conditionally on $\mathcal{G}$,
  \begin{equation}
    \label{eq:operational-certificate-refined}
    \mathbb{P}\bigl(Y_{\mathrm{new}} \in C_\alpha(X_{\mathrm{new}}) \,\big|\, \mathcal{G}\bigr)
    \;\ge\;
    \bigl[\,1 - \alpha - B_{\Delta,+} - \bar\varepsilon \max\{B_{Q,+} - B_{\Delta,+},\,0\}\,\bigr]_+.
  \end{equation}
  A simpler but looser consequence is
  \begin{equation}
    \label{eq:operational-certificate-simple}
    \mathbb{P}\bigl(Y_{\mathrm{new}} \in C_\alpha(X_{\mathrm{new}}) \,\big|\, \mathcal{G}\bigr)
    \;\ge\;
    \bigl[\,1 - \alpha - B_{\Delta,+} - \bar\varepsilon\, B_{Q,+}\,\bigr]_+.
  \end{equation}
  If the certificates are non-random, the conditioning can be dropped by averaging. If they are $\mathcal{G}$-measurable and hold only on a high-probability auxiliary event, the marginalised version in Corollary~\ref{cor:marginalized-coverage} applies.
\end{proposition}

The worst-case dirty-side choice $B_{Q,+} = 1$ is always valid, but often conservative. Smaller values of $B_{Q,+}$, or non-trivial bounds on $p_d$, require additional information about the contamination mechanism. The proposition is intentionally an \emph{interface}. It states which independent inputs suffice for a numerical certificate, but does not construct those inputs from the contaminated calibration sample.

\begin{proof}[Proof of Proposition~\ref{prop:operational-certificate}]
Work conditionally on $\mathcal{G}$. The argument has three steps: a starting bound from retained-law transfer, a monotonicity bound on the retained mixture coefficient, and a componentwise bound on the mixture loss.

By Proposition~\ref{prop:finite-sample-scalar-bound} and the mixture upper bound on $d_{R\to P,+}^A$ derived in Section~\ref{subsec:exact-scalar-mixture},
\[
  \mathbb{P}\bigl\{Y_{\mathrm{new}} \in C_\alpha(X_{\mathrm{new}}) \mid \mathcal{G}\bigr\}
  \;\ge\;
  1 - \alpha - \bigl[(1 - \tilde\varepsilon_\star)\Delta_{\mathrm{trim},+}^A + \tilde\varepsilon_\star D_{Q,+}\bigr],
\]
and non-negativity of coverage gives the corresponding positive-part bound. It remains to upper-bound the bracketed loss using the certificate inputs.

For fixed non-negative $p_d$ and positive $p_c$, the map $e \mapsto e p_d / [(1-e)p_c + e p_d]$ is non-decreasing on $[0,1)$. Combined with $p_c \ge L_c$, $p_d \le U_d$, and $\varepsilon \le \varepsilon_{\max}$, this gives
\[
  \tilde\varepsilon_\star
  \;=\; \frac{\varepsilon p_d}{(1-\varepsilon) p_c + \varepsilon p_d}
  \;\le\; \frac{\varepsilon_{\max} U_d}{(1-\varepsilon_{\max}) L_c + \varepsilon_{\max} U_d}
  \;=\; \bar\varepsilon.
\]

Applying the distortion certificates $\Delta_{\mathrm{trim},+}^A \le B_{\Delta,+}$ and $D_{Q,+} \le B_{Q,+}$ and rearranging,
\[
  (1 - \tilde\varepsilon_\star)\Delta_{\mathrm{trim},+}^A + \tilde\varepsilon_\star D_{Q,+}
  \;\le\;
  (1 - \tilde\varepsilon_\star) B_{\Delta,+} + \tilde\varepsilon_\star B_{Q,+}
  \;=\;
  B_{\Delta,+} + \tilde\varepsilon_\star (B_{Q,+} - B_{\Delta,+}).
\]
If $B_{Q,+} \ge B_{\Delta,+}$, monotonicity in $\tilde\varepsilon_\star$ and Step 2 give an upper bound of $B_{\Delta,+} + \bar\varepsilon (B_{Q,+} - B_{\Delta,+})$. If $B_{Q,+} < B_{\Delta,+}$, the second term is non-positive and the upper bound is simply $B_{\Delta,+}$. In either case,
\[
  (1 - \tilde\varepsilon_\star)\Delta_{\mathrm{trim},+}^A + \tilde\varepsilon_\star D_{Q,+}
  \;\le\;
  B_{\Delta,+} + \bar\varepsilon \max\{B_{Q,+} - B_{\Delta,+},\,0\},
\]
which proves \eqref{eq:operational-certificate-refined} after taking positive parts. The simpler bound \eqref{eq:operational-certificate-simple} follows from $\max\{B_{Q,+} - B_{\Delta,+}, 0\} \le B_{Q,+}$.
\end{proof}

        \subsection{Audit-based certificates}
\label{app:audit-certificate}

The componentwise interface in Proposition~\ref{prop:operational-certificate} certifies the prediction set indirectly, by bounding each ingredient of the mixture decomposition. When an independent clean labelled audit sample is available, a tighter and more direct route is to certify the realised prediction set itself. The following proposition gives two such certificates. Exact binomial inversion for the final selected set, which is agnostic to how the set was constructed. And a one-sided Kolmogorov certificate for any data-dependent cutoff, expressed through an audited score-CDF gap.

\begin{proposition}[Final-set and score-CDF audit certificates]
  \label{prop:exact-audit-certificate}
  Let $\mathcal{A}$ be a sigma-field containing all data and randomness used to construct a prediction set-valued map $\widehat C$, and suppose an audit sample $(X_j^{\mathrm{aud}}, Y_j^{\mathrm{aud}})_{j=1}^{n_{\mathrm{aud}}} \overset{\mathrm{iid}}{\sim} P$ is independent of $\mathcal{A}$. Define
  \[
    p_{\widehat C} = P\{Y \in \widehat C(X) \mid \mathcal{A}\},
    \qquad
    H_j = \mathbf{1}\{Y_j^{\mathrm{aud}} \in \widehat C(X_j^{\mathrm{aud}})\},
    \qquad
    M = \sum_{j=1}^{n_{\mathrm{aud}}} H_j.
  \]
  For $\beta \in (0,1)$, set
  \[
    L_{\mathrm{bin}}(M, n_{\mathrm{aud}}, \beta) =
    \begin{cases}
      0, & M = 0, \\
      \mathrm{BetaInv}\{\beta;\, M,\, n_{\mathrm{aud}} - M + 1\}, & M \ge 1,
    \end{cases}
  \]
  where $\mathrm{BetaInv}\{\beta; a, b\}$ is the $\beta$-quantile of a $\mathrm{Beta}(a, b)$ distribution. Then, conditionally on $\mathcal{A}$,
  \begin{equation}
    \label{eq:exact-binomial-audit-certificate}
    \mathbb{P}\bigl\{p_{\widehat C} \ge L_{\mathrm{bin}}(M, n_{\mathrm{aud}}, \beta) \,\big|\, \mathcal{A}\bigr\}
    \;\ge\; 1 - \beta.
  \end{equation}

  Furthermore, let $I$ be any selected calibration index set of size $k \ge 1$, $\mathcal{A}$-measurable, and let
  \[
    \widehat F_I^A(a) = \frac{1}{k}\sum_{i \in I} \mathbf{1}\{A(Z_i) \le a\}.
  \]
  Let $A_j^{\mathrm{aud}} = A(X_j^{\mathrm{aud}}, Y_j^{\mathrm{aud}})$ and $\widehat F_{\mathrm{aud}}^A$ be their empirical CDF, and suppose $c_{n_{\mathrm{aud}}, \beta}^+$ satisfies
  \[
    \mathbb{P}\!\left(
      \sup_a \bigl(\widehat F_{\mathrm{aud}}^A(a) - F_P^A(a)\bigr) > c_{n_{\mathrm{aud}}, \beta}^+
      \,\Bigm|\, \mathcal{A}
    \right)
    \;\le\; \beta.
  \]
  Then, with conditional probability at least $1 - \beta$, every $\mathcal{A}$-measurable cutoff $\hat\tau$ satisfies
  \begin{equation}
    \label{eq:ks-audit-selected-cdf-certificate}
    P\{A(Z_{\mathrm{new}}) \le \hat\tau \mid \mathcal{A}\}
    \;\ge\;
    \widehat F_I^A(\hat\tau)
    - \sup_a \bigl\{\widehat F_I^A(a) - \widehat F_{\mathrm{aud}}^A(a)\bigr\}_+
    - c_{n_{\mathrm{aud}}, \beta}^+.
  \end{equation}
  In particular, if $\hat\tau$ is the $r$-th order statistic of $\{A(Z_i) : i \in I\}$, then $\widehat F_I^A(\hat\tau) \ge r/k$, yielding the explicit bound
  \[
    P\{A(Z_{\mathrm{new}}) \le \hat\tau \mid \mathcal{A}\}
    \;\ge\;
    \frac{r}{k}
    - \sup_a \bigl\{\widehat F_I^A(a) - \widehat F_{\mathrm{aud}}^A(a)\bigr\}_+
    - c_{n_{\mathrm{aud}}, \beta}^+.
  \]
\end{proposition}

The two certificates serve complementary purposes. The binomial inversion certificate \eqref{eq:exact-binomial-audit-certificate} is exact. It applies to \emph{any} prediction-set construction, including fixed-threshold trimming, same-sample threshold selection, Stein thinning, and other selected-subset procedures. The only requirement is that the audit sample is independent of all selection decisions.

The KS certificate \eqref{eq:ks-audit-selected-cdf-certificate} is asymptotic in $n_{\mathrm{aud}}$, but it is more flexible. It directly certifies the realised score CDF. This certificate can then be combined with empirical-CDF gaps. When the selected calibration scores are well aligned with the audit scores, this can yield sharper guarantees.

\begin{proof}[Proof of Proposition~\ref{prop:exact-audit-certificate}]
We prove the two certificates separately.

Conditional on $\mathcal{A}$, the prediction-set map $\widehat C$ is fixed and the audit indicators $H_1, \ldots, H_{n_{\mathrm{aud}}}$ are i.i.d.\ Bernoulli with success probability $p_{\widehat C}$. The lower endpoint $L_{\mathrm{bin}}(M, n_{\mathrm{aud}}, \beta)$ is the standard one-sided Clopper--Pearson inversion of the binomial model. For each fixed $p \in [0, 1]$, the probability under $\mathrm{Binomial}(n_{\mathrm{aud}}, p)$ that the inverted lower endpoint exceeds $p$ is at most $\beta$. Applying this with $p = p_{\widehat C}$ gives \eqref{eq:exact-binomial-audit-certificate}.

Work on the audit-side event
\[
  \mathcal{E}_{\mathrm{aud}}
  := \biggl\{
    \sup_a \bigl(\widehat F_{\mathrm{aud}}^A(a) - F_P^A(a)\bigr) \le c_{n_{\mathrm{aud}}, \beta}^+
  \biggr\},
\]
which has conditional probability at least $1 - \beta$ by assumption. On $\mathcal{E}_{\mathrm{aud}}$, for every $a$,
\[
  F_P^A(a) \;\ge\; \widehat F_{\mathrm{aud}}^A(a) - c_{n_{\mathrm{aud}}, \beta}^+,
\]
so for every $a$,
\[
  \widehat F_I^A(a) - F_P^A(a)
  \;\le\;
  \widehat F_I^A(a) - \widehat F_{\mathrm{aud}}^A(a) + c_{n_{\mathrm{aud}}, \beta}^+.
\]
Taking suprema and positive parts,
\[
  d_{I \to P,+}^A
  \;:=\; \sup_a \bigl\{\widehat F_I^A(a) - F_P^A(a)\bigr\}_+
  \;\le\;
  \sup_a \bigl\{\widehat F_I^A(a) - \widehat F_{\mathrm{aud}}^A(a)\bigr\}_+
  + c_{n_{\mathrm{aud}}, \beta}^+.
\]
By the definition of $d_{I \to P,+}^A$, every $\mathcal{A}$-measurable cutoff $\hat\tau$ satisfies $F_P^A(\hat\tau) \ge \widehat F_I^A(\hat\tau) - d_{I \to P,+}^A$. Substituting yields \eqref{eq:ks-audit-selected-cdf-certificate}. The explicit-rank bound follows from the inclusive empirical-CDF inequality $\widehat F_I^A(\hat\tau) \ge r/k$ when $\hat\tau$ is the $r$-th selected order statistic.
\end{proof}

    \subsection{Upper-bound, efficiency, and marginalisation statements}
        \begin{proof}
  Work conditionally on $\mathcal G$. We suppress this conditioning in the notation. Define the retained mixture law as
  \begin{equation*}
      R:=P^\star(,\cdot\mid S(Z)\le t^\star).
  \end{equation*}
  By the same conditional-thinning argument used in Theorem~\ref{thm:fixed-threshold-trimmed-coverage}, the following holds. Conditional on $N_{\mathrm{keep}}=n$, the retained calibration scores are i.i.d. draws from the law of $A(W)$, where $W\sim R$. The clean test score used for retained-law evaluation is also an independent draw from the same score law.

  Fix $n$. Suppose
  \begin{equation*}
      r_n:=\lceil(n+1)(1-\alpha)\rceil\le n .
  \end{equation*}
  Under the no-ties assumption, we obtain the exact retained-law identity. The same identity also holds if the procedure uses randomised threshold tie-breaking in a consistent way.

  If ties occur at the cutoff, the situation is different. The ordinary inclusive rule may have retained-law coverage strictly above the randomised-rank value. Therefore, the upper-bound statement requires either continuity of $F_R^A$, or consistent randomised tie-breaking.
  \[
    \mathbb P_R(A_0\le \hat\tau_\alpha\mid N_{\mathrm{keep}}=n)
    =
    \frac{r_n}{n+1}.
  \]
  Since
  \[
    r_n\le (n+1)(1-\alpha)+1,
  \]
  this is at most
  \[
    1-\alpha+\frac{1}{n+1}.
  \]
  If \(r_n=n+1\), then \(\hat\tau_\alpha=+\infty\) and the retained-law coverage
  is one. In this case \(r_n=n+1\) implies
  \[
    (n+1)(1-\alpha)>n,
  \]
  hence \(\alpha<1/(n+1)\), and therefore
  \[
    1\le 1-\alpha+\frac{1}{n+1}.
  \]
  Thus, for every \(n\ge0\),
  \[
    \mathbb E\bigl[
      F_R^A(\hat\tau_\alpha)
      \mid N_{\mathrm{keep}}=n
    \bigr]
    \le
    1-\alpha+\frac{1}{n+1}.
  \]
  Averaging over \(N_{\mathrm{keep}}\) gives
  \[
    \mathbb E\bigl[F_R^A(\hat\tau_\alpha)\bigr]
    \le
    1-\alpha+
    \mathbb E\!\left[
      \frac{1}{N_{\mathrm{keep}}+1}
    \right].
  \]

  For the randomised lexicographic rule, the same argument is applied to the pair
  \(T=(A,U)\).  Define the pair CDF
  \[
    G_M(a,u)=M\{A<a\}+uM\{A=a\},
    \qquad u\in[0,1].
  \]
  Then the randomised retained-rank calculation gives the same upper bound for
  \(\mathbb E[G_R(\widehat T)]\).  Moreover,
  \[
    G_P(a,u)-G_R(a,u)
    =(1-u)\{F_P^A(a^-)-F_R^A(a^-)\}
    +u\{F_P^A(a)-F_R^A(a)\}
    \le d_{P\to R,+}^{A},
  \]
  so the target-law comparison below is unchanged.  In the continuous no-ties
  case, this reduces to the scalar CDF comparison.

  The pointwise one-sided comparison gives, for all
  \(t\in\mathbb R\cup\{+\infty\}\),
  \[
    F_P^A(t)
    \le
    F_R^A(t)+d_{P\to R,+}^{A}.
  \]
  Therefore
  \[
    \mathbb P\bigl(Y_{\mathrm{new}}\in C_\alpha(X_{\mathrm{new}})\mid\mathcal G\bigr)
    =
    \mathbb E\bigl[F_P^A(\hat\tau_\alpha)\bigr]
    \le
    1-\alpha
    +
    d_{P\to R,+}^{A}
    +
    \mathbb E\!\left[
      \frac{1}{N_{\mathrm{keep}}+1}
    \right].
  \]
  Since
  \[
    d_{P\to R,+}^{A}
    \le
    \delta_{\mathrm{mix},-}^{\mathrm{ref}}
    \le
    \delta_{\mathrm{mix}}^{\mathrm{ref}}(D_Q),
  \]
  the two subsequent upper bounds in \eqref{eq:coverage-upper-bound} follow.

  Finally, \(N_{\mathrm{keep}}\sim\operatorname{Binomial}(m,\mu_{\mathrm{keep}})\).
  Since \(\mu_{\mathrm{keep}}>0\),
  \begin{align*}
    \mathbb E\!\left[\frac{1}{N_{\mathrm{keep}}+1}\right]
    &=
    \sum_{n=0}^m
    \frac{1}{n+1}
    \binom{m}{n}
    \mu_{\mathrm{keep}}^n
    (1-\mu_{\mathrm{keep}})^{m-n} \\
    &=
    \frac{1}{(m+1)\mu_{\mathrm{keep}}}
    \sum_{n=0}^m
    \binom{m+1}{n+1}
    \mu_{\mathrm{keep}}^{n+1}
    (1-\mu_{\mathrm{keep}})^{m-n} \\
    &=
    \frac{
      1-(1-\mu_{\mathrm{keep}})^{m+1}
    }{
      (m+1)\mu_{\mathrm{keep}}
    }.
  \end{align*}
  This proves \eqref{eq:coverage-upper-bound}.
\end{proof}

        \begin{proof}
  Write \(R=P^\star(\cdot\mid S(Z)\le t^\star)\), and let
  \[
    \widehat F_n(t)
    :=
    \frac{1}{n}
    \sum_{i\in\mathcal I_{\mathrm{keep}}}
    \mathbf 1\{A_i\le t\}
  \]
  be the empirical CDF of the retained scores when \(N_{\mathrm{keep}}=n\ge1\).
  Conditional on \(N_{\mathrm{keep}}=n\), the retained scores are i.i.d.\ from
  \(F_R^A\). Hence the Dvoretzky--Kiefer--Wolfowitz inequality gives \citep{dvoretzky1956asymptotic,massart1990tight}
  \[
    \mathbb P\!\left(
      \sup_{t\in\mathbb R}
      |\widehat F_n(t)-F_R^A(t)|
      >
      \eta_{n_0,\beta}
      \,\middle|\,
      N_{\mathrm{keep}}=n
    \right)
    \le
    2e^{-2n\eta_{n_0,\beta}^2}
    \le
    2e^{-2n_0\eta_{n_0,\beta}^2}
    =
    \beta
  \]
  for every \(n\ge n_0\). Therefore the event
  \[
    \mathcal E_{n_0,\beta}
    :=
    \left\{
      N_{\mathrm{keep}}\ge n_0
      \ \text{and}\
      \sup_{t\in\mathbb R}
      |\widehat F_{N_{\mathrm{keep}}}(t)-F_R^A(t)|
      \le
      \eta_{n_0,\beta}
    \right\}
  \]
  has probability at least \(1-\mathbb P(N_{\mathrm{keep}}<n_0)-\beta\).

  Work on \(\mathcal E_{n_0,\beta}\), and write \(n=N_{\mathrm{keep}}\). Since
  \(n\ge n_0\ge \lceil 1/\alpha\rceil-1\), the conformal index
  \[
    r_n:=\lceil(n+1)(1-\alpha)\rceil
  \]
  satisfies \(r_n\le n\), so \(\hat\tau_\alpha\) is a finite order statistic.
  Let \(p=1-\alpha\), \(q=q_P(p)\), and
  \[
    \Gamma=\Gamma_{n_0,\beta}.
  \]
  By assumption, \(h:=\Gamma/\lambda_{\min}<\rho\).  The local CDF-increment condition from Lemma~\ref{lem:quantile-perturbation} gives
  \[
    F_P^A(q+h)\ge p+\Gamma,
    \qquad
    F_P^A(q-h)\le p-\Gamma.
  \]

  The retained-to-clean Kolmogorov comparison gives
  \[
    \sup_{t\in\mathbb R}|F_R^A(t)-F_P^A(t)|
    \le
    \delta_{\mathrm{mix}}^{\mathrm{ref}}(D_Q).
  \]
  Hence
  \[
    F_R^A(q+h)
    \ge
    p+\eta_{n_0,\beta}+\frac{2}{n_0},
  \]
  and on \(\mathcal E_{n_0,\beta}\),
  \[
    \widehat F_n(q+h)
    \ge
    p+\frac{2}{n_0}.
  \]
  On the other hand,
  \[
    \frac{r_n}{n}
    \le
    \frac{(n+1)p+1}{n}
    =
    p+\frac{p+1}{n}
    \le
    p+\frac{2}{n}
    \le
    p+\frac{2}{n_0}.
  \]
  Therefore \(\widehat F_n(q+h)\ge r_n/n\), which implies
  \[
    \hat\tau_\alpha\le q+h.
  \]

  Similarly,
  \[
    F_R^A(q-h)
    \le
    p-\eta_{n_0,\beta}-\frac{2}{n_0},
  \]
  and hence
  \[
    \widehat F_n(q-h)\le p-\frac{2}{n_0}<\frac{r_n}{n}.
  \]
  Thus \(q-h\) lies strictly below the \(r_n\)-th retained order statistic, and
  \[
    \hat\tau_\alpha\ge q-h.
  \]
  Combining the upper and lower inequalities yields
  \[
    |\hat\tau_\alpha-q_P(1-\alpha)|
    \le
    \frac{\Gamma_{n_0,\beta}}{\lambda_{\min}}
  \]
  on \(\mathcal E_{n_0,\beta}\), which proves
  \eqref{eq:threshold-efficiency-prob}. Under the additional residual-score assumption \(A(x,y)=|y-\hat f(x)|\), the
  prediction set is the symmetric residual interval displayed in the proposition.
  Its width is \(2\hat\tau_\alpha\), so the stated width bound follows
  immediately.
\end{proof}

        \begin{proof}
  Let
  \[
    H
    :=
    \mathbb P\bigl(
      Y_{\mathrm{new}}\in C_\alpha(X_{\mathrm{new}})
      \,\big|\,
      \mathcal G
    \bigr).
  \]
  Since \(0\le H\le1\), and on \(\mathcal E\) one has
  \[
    H\ge 1-\alpha-L\ge 1-\alpha-L_0,
  \]
  it follows that
  \[
    H\mathbf 1_{\mathcal E}
    \ge
    [1-\alpha-L_0]_+\,\mathbf 1_{\mathcal E}.
  \]
  Taking expectations gives
  \[
    \mathbb P\bigl(
      Y_{\mathrm{new}}\in C_\alpha(X_{\mathrm{new}})
    \bigr)
    =
    \mathbb E[H]
    \ge
    [1-\alpha-L_0]_+\mathbb P(\mathcal E)
    \ge
    (1-\beta)[1-\alpha-L_0]_+,
  \]
  proving \eqref{eq:marginalized-coverage-product}. The additive form follows
  from the elementary inequality
  \[
    (1-\beta)c\ge c-\beta
    \qquad\text{for }0\le c\le1,
  \]
  applied with \(c=[1-\alpha-L_0]_+\), and from
  \[
    [1-\alpha-L_0]_+\ge 1-\alpha-L_0.
  \]
  The final statement is the union bound applied to
  \(\mathcal E_s^c\cup\mathcal E_{\mathrm{ref}}(\delta_\ell)^c\).
\end{proof}

    \subsection{Proof of Proposition~\ref{prop:mixture-coverage-fixed}}
        \begin{proof}
  The direct bound \eqref{eq:mixture-direct-transfer} follows from Proposition~\ref{prop:retained-law-transfer} applied with target law \(P_0=P^\star\), together with non-negativity of coverage for the positive-part form. The target-law exact identity follows from the proof of Theorem~\ref{thm:fixed-threshold-trimmed-coverage} with \(P_0=P^\star\).

  For the clean-component reduction, write the mixture-test coverage as
  \begin{align*}
    &\mathbb P_{Z_{\mathrm{new}}\sim P^\star}
    \{Y_{\mathrm{new}}\in C_\alpha(X_{\mathrm{new}})\mid\mathcal G\} \\
    &\qquad =
    (1-\varepsilon)
    \mathbb P_{Z_{\mathrm{new}}\sim P}
    \{Y_{\mathrm{new}}\in C_\alpha(X_{\mathrm{new}})\mid\mathcal G\}
    +
    \varepsilon
    \mathbb P_{Z_{\mathrm{new}}\sim Q}
    \{Y_{\mathrm{new}}\in C_\alpha(X_{\mathrm{new}})\mid\mathcal G\}.
  \end{align*}
  Dropping the non-negative dirty-component coverage gives \eqref{eq:mixture-reduction}. Applying Theorem~\ref{thm:fixed-threshold-trimmed-coverage} to the clean component gives \eqref{eq:mixture-coverage-fixed}. Since \(D_Q\le1\), \eqref{eq:mixture-coverage-fixed-worst} follows.
\end{proof}

    \subsection{Supplementary corollary for contaminated-mixture coverage}
        \begin{corollary}[Conservative additive-form mixture coverage bound]
  \label{cor:mixture-additive}
  Under the assumptions of Proposition~\ref{prop:mixture-coverage-fixed}, on any realisation of \(\mathcal G\) satisfying those assumptions,
  \begin{align}
    \mathbb P_{Z_{\mathrm{new}}\sim P^\star}
    \bigl(
      Y_{\mathrm{new}}\in C_\alpha(X_{\mathrm{new}})
      \,\big|\,
      \mathcal G
    \bigr)
    &\ge
    (1-\varepsilon)(1-\alpha)
    -(1-\varepsilon)(1-\tilde\varepsilon_\star)\Delta_{\mathrm{trim}}^{A}
    -(1-\varepsilon)\tilde\varepsilon_\star D_Q.
    \label{eq:mixture-additive}
  \end{align}
  In particular, since \(D_Q\le1\), \(1-\varepsilon\le1\), and
  \(1-\tilde\varepsilon_\star\le1\), the weaker but simpler bound
  \begin{align}
    \mathbb P_{Z_{\mathrm{new}}\sim P^\star}
    \bigl(
      Y_{\mathrm{new}}\in C_\alpha(X_{\mathrm{new}})
      \,\big|\,
      \mathcal G
    \bigr)
    &\ge
    (1-\varepsilon)(1-\alpha)
    -\Delta_{\mathrm{trim}}^{A}
    -\tilde\varepsilon_\star
    \label{eq:mixture-additive-simple}
  \end{align}
  also holds.
\end{corollary}

\begin{proof}
  Proposition~\ref{prop:mixture-coverage-fixed} gives
  \[
    \mathbb P_{Z_{\mathrm{new}}\sim P^\star}
    \bigl(
      Y_{\mathrm{new}}\in C_\alpha(X_{\mathrm{new}})
      \,\big|\,
      \mathcal G
    \bigr)
    \ge
    (1-\varepsilon)
    \bigl[
      1-\alpha-\delta_{\mathrm{mix}}^{\mathrm{ref}}(D_Q)
    \bigr].
  \]
  Expanding
  \[
    \delta_{\mathrm{mix}}^{\mathrm{ref}}(D_Q)
    =
    (1-\tilde\varepsilon_\star)\Delta_{\mathrm{trim}}^{A}
    +
    \tilde\varepsilon_\star D_Q
  \]
  gives \eqref{eq:mixture-additive}.

  The simple bound follows from the non-negativity of
  \(\Delta_{\mathrm{trim}}^{A}\), \(\tilde\varepsilon_\star\), and \(D_Q\), together
  with
  \[
    0\le 1-\varepsilon\le1,
    \qquad
    0\le 1-\tilde\varepsilon_\star\le1,
    \qquad
    0\le D_Q\le1.
  \]
\end{proof}

        \begin{remark}[Interpreting the conservative clean-component mixture bound]
  \label{rem:mixture-interpretation}
  Proposition~\ref{prop:mixture-coverage-fixed} contains two distinct statements for mixture-test coverage. The direct retained-to-$P^\star$ transfer is the relevant scalar CDF bound when the deployment target is the contaminated mixture itself.

  The clean-component reduction is more conservative. It lower-bounds mixture-test coverage using only the clean part of the mixture. It also uses only the trivial fact that dirty-test coverage is non-negative.
  In conditional form, the clean-component reduction gives
  \[
    \mathbb P_{Z_{\mathrm{new}}\sim P^\star}
    \bigl(
      Y_{\mathrm{new}}\in C_\alpha(X_{\mathrm{new}})
      \,\big|\,
      \mathcal G
    \bigr)
    \ge
    (1-\varepsilon)
    \bigl[
      1-\alpha-\delta_{\mathrm{mix}}^{\mathrm{ref}}(D_Q)
    \bigr]_+.
  \]
  Expanding the retained-mixture loss yields
  \[
    (1-\varepsilon)
    \Bigl[
      1-\alpha
      -(1-\tilde\varepsilon_\star)\Delta_{\mathrm{trim}}^{A}
      -\tilde\varepsilon_\star D_Q
    \Bigr]_+.
  \]
  This display separates three mechanisms. First, the prefactor $(1-\varepsilon)$ is a worst-case deployed-mixture penalty. It appears because we make no assumption about coverage under an arbitrary dirty test law.

  Second, the retained mixture coefficient $\tilde\varepsilon_\star$ is the dirty-component weight actually seen by the retained conformal quantile. Third, the clean distortion $\Delta_{\mathrm{trim}}^{A}$ is the price paid for conditioning clean calibration points on the retained event.

  Thus, the raw contamination level $\varepsilon$ is not, by itself, the right diagnostic for the conformal calibration step after trimming. The relevant population tuple is
  \begin{equation*}
      \bigl(p_c,p_d,\tilde\varepsilon_\star,
      \Delta_{\mathrm{trim}}^{A},D_Q\bigr).
  \end{equation*}
  No additional $m^{-1/2}$ term appears in this population theorem. Such terms arise only from separate empirical certificates, or from events used to justify data-driven threshold selection.
\end{remark}

        \begin{remark}[Interpretation of the \((1-\varepsilon)\) factor in the mixture-coverage bound]
  \label{rem:mixture-worst-q}

  The prefactor $(1-\varepsilon)$ in the conservative clean-component bounds of Proposition~\ref{prop:mixture-coverage-fixed} and Corollary~\ref{cor:mixture-additive} should be viewed as a worst-case contamination benchmark. It appears because the deployed-mixture coverage lower bound uses only the trivial inequality
  \begin{equation*}
      \mathbb P_{Z_{\mathrm{new}}\sim Q}
      \bigl(
      Y_{\mathrm{new}}\in C_\alpha(X_{\mathrm{new}})
      \bigr)
      \ge 0.
  \end{equation*}
  Thus, the bound does not exploit any favourable structure of the contaminating law $Q$. Consequently, the resulting mixture-law guarantee is conservative whenever the prediction set also achieves none trivial coverage under $Q$. Stronger deployed-mixture coverage guarantees are possible in settings where we can separately lower-bound
  \begin{equation*}
      \mathbb P_{Z_{\mathrm{new}}\sim Q}
      \bigl(
      Y_{\mathrm{new}}\in C_\alpha(X_{\mathrm{new}})
      \bigr).
  \end{equation*}
  They are also possible when one can exploit overlap or regularity shared by $P$ and $Q$. The direct retained-to-$P^\star$ bound in Proposition~\ref{prop:mixture-coverage-fixed} is the appropriate scalar transfer bound for mixture testing. The clean-component corollaries should be interpreted differently. They provide robust worst-case benchmarks under arbitrary dirty test behaviour. They should not be read as sharp descriptions of mixture-law coverage in favourable contamination regimes.
\end{remark}

\section{Benchmark regimes}
    \label{app:benchmark-regimes}
    \subsection{Ideal separation}
        \begin{theorem}[Best-case validity under ideal separation for fixed-threshold trimming]
  \label{thm:ideal-separation-fixed}

  Assume the setup of Theorem~\ref{thm:fixed-threshold-trimmed-coverage}. Let
  \begin{equation*}
      \mathcal A_{\mathrm{clean}}
  \end{equation*}
  denote the realised clean auxiliary stage. This includes all auxiliary clean randomness used to construct the fitted functions $A$ and $S$, as well as the fixed threshold $t^\star$.

  Throughout this theorem, all conditional probability statements are interpreted given the realised clean auxiliary stage. Thus, conditional on $\mathcal A_{\mathrm{clean}}$, the functions $A$, $S$, and the threshold $t^\star$ are deterministic.

  Suppose that, for the realised score function $S$ and threshold $t^\star$, the fixed threshold perfectly separates the clean and contaminating distributions:
    \begin{equation}
      \label{eq:ideal-separation-fixed-threshold}
      \mathbb P_{Z\sim P}\bigl(S(Z)\le t^\star\bigr)=1,
      \qquad
      \mathbb P_{Z\sim Q}\bigl(S(Z)\le t^\star\bigr)=0,
    \end{equation}
  Here, the probabilities are taken over an independent draw $Z$ from $P$ or $Q$. The score function $S$ and the threshold $t^\star$ are held fixed at their realised auxiliary-stage values.

  Then, for every integer $n$ with positive conditional probability under the realised auxiliary stage, conditional on
  \(\mathcal A_{\mathrm{clean}}\) and on \(N_{\mathrm{keep}}=n\),
  \begin{equation}
    \label{eq:ideal-separation-conditional-coverage}
    \mathbb P\Bigl(
      Y_{\mathrm{new}}\in C_\alpha(X_{\mathrm{new}})
      \,\Big|\,
      \mathcal A_{\mathrm{clean}},\,N_{\mathrm{keep}}=n
    \Bigr)
    \ge
    \frac{\left\lceil (n+1)(1-\alpha)\right\rceil}{n+1}
    \ge
    1-\alpha.
  \end{equation}
  Consequently,
  \[
    \mathbb P\bigl(
      Y_{\mathrm{new}}\in C_\alpha(X_{\mathrm{new}})
    \bigr)
    \ge
    1-\alpha.
  \]
\end{theorem}

\begin{remark}[Rare degenerate case under ideal separation]
  \label{rem:ideal-separation-degenerate}
  Under the ideal-separation assumptions,
  \[
    \mathcal I_{\mathrm{keep}}=\mathcal C
    \qquad\text{and hence}\qquad
    N_{\mathrm{keep}}=|\mathcal C|
  \]
  almost surely, where \(\mathcal C\) is the latent clean index set in the mixture construction. Therefore
  \[
    N_{\mathrm{keep}}
    \sim
    \mathrm{Binomial}(m,1-\varepsilon),
    \qquad
    \mathbb P\bigl(N_{\mathrm{keep}}=0\bigr)=\varepsilon^m.
  \]
  Thus, the case $N_{\mathrm{keep}}=0$ is not impossible. However, it is a rare corner event when $\varepsilon<1$ and $m$ are moderate. On this event, the algorithm returns
  \begin{equation*}
      C_\alpha(x)=\mathcal Y.
  \end{equation*}
  This gives coverage one, but it provides no efficiency.
\end{remark}

\begin{proof}
  Work on the expanded mixture probability space with latent component labels.
  Let
  \[
    \mathcal C
    :=
    \{i\in\{1,\dots,m\}: Z_i \text{ is drawn from } P\},
    \qquad
    \mathcal D
    :=
    \{i\in\{1,\dots,m\}: Z_i \text{ is drawn from } Q\}.
  \]
  Thus, $\mathcal C$ and $\mathcal D$ form a partition of $\{1,\dots,m\}$. The partition is determined by whether each calibration point comes from the clean component or the contaminating component of the mixture.

  We first show that the retained set equals the clean component almost surely. Conditional on $\mathcal A_{\mathrm{clean}}$, the score $S$ and threshold
  $t^\star$ are fixed. By the first part of
  \eqref{eq:ideal-separation-fixed-threshold}, for each clean calibration point $i\in\mathcal C$,
  \[
    \mathbb P\bigl(
      S(Z_i)>t^\star
      \,\big|\,
      i\in\mathcal C,\mathcal A_{\mathrm{clean}}
    \bigr)
    =
    0.
  \]
  Since the calibration sample is finite, a union bound over the clean indices gives
  \[
    \mathbb P\Bigl(
      \exists i\in\mathcal C:\ S(Z_i)>t^\star
      \,\Big|\,
      \mathcal C,\mathcal D,\mathcal A_{\mathrm{clean}}
    \Bigr)
    \le
    \sum_{i\in\mathcal C}
    \mathbb P\bigl(
      S(Z_i)>t^\star
      \,\big|\,
      i\in\mathcal C,\mathcal A_{\mathrm{clean}}
    \bigr)
    =
    0.
  \]
  Similarly, by the second part of
  \eqref{eq:ideal-separation-fixed-threshold}, for each dirty calibration point
  \(i\in\mathcal D\),
  \[
    \mathbb P\bigl(
      S(Z_i)\le t^\star
      \,\big|\,
      i\in\mathcal D,\mathcal A_{\mathrm{clean}}
    \bigr)
    =
    0,
  \]
  and another finite union bound yields
  \[
    \mathbb P\Bigl(
      \exists i\in\mathcal D:\ S(Z_i)\le t^\star
      \,\Big|\,
      \mathcal C,\mathcal D,\mathcal A_{\mathrm{clean}}
    \Bigr)
    =
    0.
  \]
  Therefore, almost surely,
  \[
    S(Z_i)\le t^\star \quad \text{for every } i\in\mathcal C,
    \qquad
    S(Z_i)>t^\star \quad \text{for every } i\in\mathcal D.
  \]
  Since
  \[
    \mathcal I_{\mathrm{keep}}
    =
    \{i\in\{1,\dots,m\}: S(Z_i)\le t^\star\},
  \]
  it follows that, almost surely,
  \begin{equation}
    \label{eq:ikeep-equals-clean}
    \mathcal I_{\mathrm{keep}}=\mathcal C.
  \end{equation}
  Consequently,
  \begin{equation}
    \label{eq:nkeep-equals-nclean}
    N_{\mathrm{keep}}
    =
    |\mathcal I_{\mathrm{keep}}|
    =
    |\mathcal C|.
  \end{equation}

  Next, condition on \(\mathcal A_{\mathrm{clean}}\). Once the full clean auxiliary stage is fixed, the fitted functions \(A\) and \(S\) and the threshold \(t^\star\) are deterministic measurable objects. Conditional on the latent clean index set \(\mathcal C\), the clean calibration points
  \[
    \{Z_i:i\in\mathcal C\}
  \]
  are independent draws from \(P\). Since the identities \eqref{eq:ikeep-equals-clean} and \eqref{eq:nkeep-equals-nclean} hold almost surely, the retained calibration scores satisfy
  \[
    \{A_i:i\in\mathcal I_{\mathrm{keep}}\}
    =
    \{A(Z_i):i\in\mathcal C\}.
  \]

  We now justify the conditioning on \(N_{\mathrm{keep}}\) explicitly. Fix an integer \(n\in\{0,\dots,m\}\). On the ideal-separation event, \(N_{\mathrm{keep}}=|\mathcal C|\). For every deterministic subset \(I\subseteq\{1,\dots,m\}\) with \(|I|=n\), conditional on
  \[
    \mathcal A_{\mathrm{clean}}
    \quad\text{and}\quad
    \mathcal C=I,
  \]
  the collection
  \[
    \{Z_i:i\in I\}
  \]
  consists of \(n\) independent draws from \(P\). Moreover, the clean test point
  \[
    Z_{\mathrm{new}}=(X_{\mathrm{new}},Y_{\mathrm{new}})\sim P
  \]
  is independent of the calibration sample and of the clean auxiliary stage. Hence, conditional on
  \[
    \mathcal A_{\mathrm{clean}}
    \quad\text{and}\quad
    \mathcal C=I,
  \]
  the nonconformity scores
  \[
    \{A(Z_i):i\in I\}
    \cup
    \{A(Z_{\mathrm{new}})\}
  \]
  are exchangeable.

  Averaging over all deterministic subsets $I$ of size $n$ preserves this exchangeable joint law. Indeed, for each such $I$, the conditional law of the retained clean scores and the test score is the same. It is the law induced by $n+1$ i.i.d. draws from $P$, after applying the fixed map $A$. Therefore, conditional on
  \[
    \mathcal A_{\mathrm{clean}}
    \quad\text{and}\quad
    N_{\mathrm{keep}}=n,
  \]
  the collection
  \[
    \{A_i:i\in\mathcal I_{\mathrm{keep}}\}\cup\{A(Z_{\mathrm{new}})\}
  \]
  is exchangeable.

  We now fix the realised value \(N_{\mathrm{keep}}=n\) and consider two cases.

  \paragraph{Case 1: \(N_{\mathrm{keep}}=0\).}
  By the algorithmic definition of the procedure, \(\hat\tau_\alpha=+\infty\).
  Hence
  \[
    C_\alpha(x)=\mathcal Y
    \qquad\text{for every }x\in\mathcal X,
  \]
  and therefore
  \[
    \mathbb P\Bigl(
      Y_{\mathrm{new}}\in C_\alpha(X_{\mathrm{new}})
      \,\Big|\,
      \mathcal A_{\mathrm{clean}},\,N_{\mathrm{keep}}=0
    \Bigr)
    =
    1.
  \]
  Since
  \[
    \frac{\left\lceil (0+1)(1-\alpha)\right\rceil}{0+1}
    =
    \left\lceil 1-\alpha\right\rceil
    =
    1,
  \]
  the bound \eqref{eq:ideal-separation-conditional-coverage} holds in this
  case.

  \paragraph{Case 2: \(n\ge1\).}
  Let
  \[
    r_{\mathrm{keep}}
    :=
    \bigl\lceil (n+1)(1-\alpha)\bigr\rceil.
  \]

  If \(r_{\mathrm{keep}}=n+1\), then by the algorithmic definition
  \(\hat\tau_\alpha=+\infty\), and hence
  \[
    \mathbb P\Bigl(
      Y_{\mathrm{new}}\in C_\alpha(X_{\mathrm{new}})
      \,\Big|\,
      \mathcal A_{\mathrm{clean}},\,N_{\mathrm{keep}}=n
    \Bigr)
    =
    1
    \ge
    \frac{r_{\mathrm{keep}}}{n+1}.
  \]

  If instead \(r_{\mathrm{keep}}\le n\), then \(\hat\tau_\alpha\) is the
  \(r_{\mathrm{keep}}\)-th order statistic of the retained scores. By
  exchangeability, the usual split-conformal rank argument gives
  \[
    \mathbb P\Bigl(
      A(Z_{\mathrm{new}})\le \hat\tau_\alpha
      \,\Big|\,
      \mathcal A_{\mathrm{clean}},\,N_{\mathrm{keep}}=n
    \Bigr)
    \ge
    \frac{r_{\mathrm{keep}}}{n+1}
    \ge
    1-\alpha.
  \]
  Since \(Y_{\mathrm{new}}\in C_\alpha(X_{\mathrm{new}})\) is equivalent to
  \(A(Z_{\mathrm{new}})\le \hat\tau_\alpha\), this proves
  \eqref{eq:ideal-separation-conditional-coverage} in the second case as well.

  Finally, averaging \eqref{eq:ideal-separation-conditional-coverage} over the
  distribution of \(N_{\mathrm{keep}}\) conditional on \(\mathcal A_{\mathrm{clean}}\),
  and then over \(\mathcal A_{\mathrm{clean}}\), yields the unconditional lower
  bound
  \[
    \mathbb P\bigl(
      Y_{\mathrm{new}}\in C_\alpha(X_{\mathrm{new}})
    \bigr)
    \ge
    1-\alpha.
  \]
\end{proof}

\subsection{No separation}
\begin{theorem}[No first-order retained-contamination gain under no separation]
  \label{thm:no-separation-fixed}

  Work conditionally on the auxiliary sigma-field \(\mathcal G\). Assume the setup of Theorem~\ref{thm:fixed-threshold-trimmed-coverage}, so that \(A\), \(S\), and \(t^\star\) are fixed after conditioning and all quantities below are \(\mathcal G\)-measurable. Suppose that the anomaly score is uninformative at the fixed threshold \(t^\star\), in the retained-probability sense that
  \[
    p_c
    =
    \mathbb P_{Z\sim P}\bigl(S(Z)\le t^\star\bigr)
    =
    \mathbb P_{Z\sim Q}\bigl(S(Z)\le t^\star\bigr)
    =
    p_d
    =:
    r
    \in(0,1].
  \]

  Under \(Z\sim P\), define
  \[
    K:=\mathbf 1\{S(Z)\le t^\star\}.
  \]
  Since \(p_c=r\), \(\mathbb E_P[K]=r\). Define the clean retained-event
  supremum covariance discrepancy
  \[
    \rho_{A,K}
    :=
    \sup_{t\in\mathbb R}
    \left|
      \operatorname{Cov}_P
      \bigl(
        \mathbf 1\{A(Z)\le t\},K
      \bigr)
    \right|.
  \]
  Also define the retained dirty nonconformity discrepancies
  \[
    Q_{\mathrm{keep}}:=Q(\,\cdot\mid S(Z)\le t^\star),
    \qquad
    D_{Q,+}:=\sup_{t\in\mathbb R}
    \bigl(F_{Q_{\mathrm{keep}}}^{A}(t)-F_P^A(t)\bigr)_+,
  \]
  and
  \[
    D_Q:=\sup_{t\in\mathbb R}
    \bigl|F_{Q_{\mathrm{keep}}}^{A}(t)-F_P^A(t)\bigr|.
  \]

  Then the following statements hold:
  \begin{enumerate}
    \item The retained mass satisfies
      \[
        \mu_{\mathrm{keep}}
        =
        (1-\varepsilon)p_c+\varepsilon p_d
        =
        r.
      \]

    \item The retained mixture coefficient is unchanged:
      \[
        \tilde\varepsilon_\star
        =
        \frac{\varepsilon p_d}{(1-\varepsilon)p_c+\varepsilon p_d}
        =
        \varepsilon.
      \]

    \item The clean trimming distortion is exactly controlled by the retained-event
      supremum covariance discrepancy:
      \[
        \Delta_{\mathrm{trim}}^{A}
        =
        \frac{\rho_{A,K}}{r}.
      \]
      In particular, if \(\rho_{A,K}\le c\), then
      \[
        \Delta_{\mathrm{trim}}^{A}\le\frac{c}{r}.
      \]
      As a strong sufficient special case, if \(S(Z)\) and \(A(Z)\) are
      independent under \(P\), then \(\rho_{A,K}=0\) and hence
      \[
        \Delta_{\mathrm{trim}}^{A}=0.
      \]
      More generally, the relevant requirement is not literal independence but
      small \(\rho_{A,K}\).

    \item Consequently, the mixture upper-bound consequence of
      Theorem~\ref{thm:fixed-threshold-trimmed-coverage} reduces to
      \[
      \begin{aligned}
        \mathbb P\bigl(
          Y_{\mathrm{new}}\in C_\alpha(X_{\mathrm{new}})
          \,\big|\,
          \mathcal G
        \bigr)
        &\ge
        \biggl[1-\alpha
        -(1-\varepsilon)\frac{\rho_{A,K}}{r}
        -\varepsilon D_{Q,+}\biggr]_+ \\
        &\ge
        \biggl[1-\alpha
        -(1-\varepsilon)\frac{\rho_{A,K}}{r}
        -\varepsilon D_Q\biggr]_+.
      \end{aligned}
      \]
      Since \(D_Q\le1\), the worst-case dirty-side version is
      \[
        \mathbb P\bigl(
          Y_{\mathrm{new}}\in C_\alpha(X_{\mathrm{new}})
          \,\big|\,
          \mathcal G
        \bigr)
        \ge
        \biggl[1-\alpha
        -\varepsilon
        -(1-\varepsilon)\frac{\rho_{A,K}}{r}\biggr]_+.
      \]
      In particular, if \(\rho_{A,K}\le c\), then
      \[
        \mathbb P\bigl(
          Y_{\mathrm{new}}\in C_\alpha(X_{\mathrm{new}})
          \,\big|\,
          \mathcal G
        \bigr)
        \ge
        \biggl[1-\alpha
        -\varepsilon
        -(1-\varepsilon)\frac{c}{r}\biggr]_+.
      \]
      In the independence special case this further simplifies to
      \[
        \mathbb P\bigl(
          Y_{\mathrm{new}}\in C_\alpha(X_{\mathrm{new}})
          \,\big|\,
          \mathcal G
        \bigr)
        \ge
        [1-\alpha-\varepsilon]_+.
      \]
      Thus, in the no-separation regime $p_c=p_d=r$, fixed-threshold trimming does not reduce the retained mixture coefficient. In this case, $\tilde\varepsilon_\star$ remains equal to $\varepsilon$. Any further benefit must come from one of two sources. The first is structural information implying $D_{Q,+}<1$. The second is that the retained event is nearly harmless for the clean nonconformity score. Equivalently, $\rho_{A,K}$ must be small.
  \end{enumerate}
\end{theorem}

\begin{remark}[Interpretation of the no-separation regime]
  \label{rem:no-separation-fixed-interpretation}
  Theorem~\ref{thm:no-separation-fixed} gives a negative characterisation of the retained-contamination part of the mixture upper-bound diagnostic. It is not a universal impossibility theorem.

  When $p_c=p_d=r$, trimming does not reduce the retained mixture coefficient:
  \begin{equation*}
      \tilde\varepsilon_\star=\varepsilon.
  \end{equation*}
  In the endpoint case $r=1$, no observations are trimmed. Then $K=1$ $P$-almost surely, $\rho_{A,K}=0$, and
  \begin{equation*}
      \Delta_{\mathrm{trim}}^{A}=0.
  \end{equation*}
  However, the retained mixture coefficient remains
  \begin{equation*}
      \tilde\varepsilon_\star=\varepsilon.
  \end{equation*}

  For $r<1$, trimming may also distort the clean nonconformity distribution. This distortion is measured exactly by
  \begin{equation*}
      \Delta_{\mathrm{trim}}^{A}=\frac{\rho_{A,K}}{r}.
  \end{equation*}
  Thus, without retained-probability separation, trimming is not a free lunch. It leaves the retained mixture coefficient unchanged. It can also introduce a covariance-distortion cost.

  The refined dirty-side factor $D_{Q,+}$ identifies the only way the retained dirty law can help in this regime. Its retained nonconformity distribution must itself be close to the clean nonconformity distribution. Therefore, the theorem rules out improvement through the retained contamination coefficient. It does not rule out improvement through a favourable change in the retained dirty score distribution $Q_{\mathrm{keep}}$.
\end{remark}

\begin{proof}
  By assumption \(p_c=p_d=r\), hence
  \[
    \mu_{\mathrm{keep}}
    =
    (1-\varepsilon)p_c+\varepsilon p_d
    =
    (1-\varepsilon)r+\varepsilon r
    =
    r,
  \]
  proving part~(1).

  The retained mixture coefficient is
  \[
    \tilde\varepsilon_\star
    =
    \frac{\varepsilon p_d}{(1-\varepsilon)p_c+\varepsilon p_d}.
  \]
  Substituting \(p_c=p_d=r\) gives
  \[
    \tilde\varepsilon_\star
    =
    \frac{\varepsilon r}{(1-\varepsilon)r+\varepsilon r}
    =
    \varepsilon,
  \]
  proving part~(2).

  By Proposition~\ref{prop:delta-bias-covariance} applied with \(p_c=r\), we
  have the exact identity
  \[
    \Delta_{\mathrm{trim}}^{A}
    =
    \frac{1}{r}
    \sup_{t\in\mathbb R}
    \left|
      \operatorname{Cov}_P
      \bigl(\mathbf 1\{A(Z)\le t\},K\bigr)
    \right|
    =
    \frac{\rho_{A,K}}{r}.
  \]
  This proves part~(3). The bound with \(\rho_{A,K}\le c\) is immediate. If
  \(S(Z)\) and \(A(Z)\) are independent under \(P\), then
  \(K=\mathbf 1\{S(Z)\le t^\star\}\) is independent of each event
  \(\{A(Z)\le t\}\), so all the displayed covariances vanish and
  \(\rho_{A,K}=0\).

  Finally, substitute
  \[
    \tilde\varepsilon_\star=\varepsilon,
    \qquad
    \Delta_{\mathrm{trim}}^{A}=\frac{\rho_{A,K}}{r}
  \]
  into the one-sided mixture upper-bound consequence of Theorem~\ref{thm:fixed-threshold-trimmed-coverage}.  Since \(\Delta_{\mathrm{trim},+}^{A}\le\Delta_{\mathrm{trim}}^{A}=\rho_{A,K}/r\), this gives the refined bound
  \[
    \mathbb P\bigl(
      Y_{\mathrm{new}}\in C_\alpha(X_{\mathrm{new}})
      \,\big|\,
      \mathcal G
    \bigr)
    \ge
    \biggl[1-\alpha
    -(1-\varepsilon)\frac{\rho_{A,K}}{r}
    -\varepsilon D_{Q,+}\biggr]_+.
  \]
  The displayed bound with \(D_Q\) follows from \(D_{Q,+}\le D_Q\).  Since \(D_Q\le1\), the worst-case dirty-side version and its consequences follow.
\end{proof}

\section{Detailed experimental summaries}
\label{app:detailed-experiments}
\subsection{Experimental setup}
The empirical study is organised around the clean-test retained-law identity and mixture upper-bound diagnostic in Theorem~\ref{thm:fixed-threshold-trimmed-coverage}, the contaminated-mixture comparison in Proposition~\ref{prop:mixture-coverage-fixed}, and the ideal- and no-separation benchmark regimes. All simulation summaries use 100 Monte Carlo repetitions. Code to reproduce our experiments is available via this \href{https://github.com/congyewang/When-Does-Trimming-Help-Conformal-Prediction-A-Retained-Law-Diagnostic-under-Calibration-Contaminat}{GitHub link}.

In the regression experiments, the conformal nonconformity score is the absolute residual from a linear predictive backbone. A clean fitting split trains the prediction rule and, when needed, an auxiliary anomaly-score model. An independent clean reference split is used only when the threshold policy requires clean reference threshold selection. The final contaminated calibration sample is then trimmed by the fixed anomaly score and conformalised using the retained calibration scores.

Several experiments use a plug-in Stein score-norm as the anomaly score, following the kernelised Stein discrepancy and related kernel-discrepancy literature \citep{gorham2015measuring, liu2016kernelized, gorham2017kernels, jitkrittum2020conditional, oates2022minimum, wang2023stein, liu2026probabilisticinferencelearningsteins}. For a representation $u=T(x)$, a fitted clean score model $\hat s$, and a smooth scalar kernel $k$, define
\begin{equation*}
  S(z)=\sqrt{\max\{h_{\hat s}(T(x),T(x)),0\}},
\end{equation*}
where
\[
  h_{\hat s}(u,v)
  =\hat s(u)^\top k(u,v)\hat s(v)
  +\hat s(u)^\top\nabla_v k(u,v)
  +\hat s(v)^\top\nabla_u k(u,v)
  +\mathrm{tr}\{\nabla_u\nabla_v^\top k(u,v)\}.
\]
The one-dimensional Gaussian experiments use $T(x)=x$. A Gaussian clean score is fitted on the independent clean fitting split
\begin{equation*}
    \hat s(u)=-\frac{u-\hat\mu_X}{\hat\sigma_X^2},
\end{equation*}
where $\hat\mu_X$ and $\hat\sigma_X^2$ are the empirical means and variances of the clean fitting covariates. The kernel is the Gaussian RBF kernel
\begin{equation*}
    k(u,v)=\exp\{-(u-v)^2/(2h^2)\},
\end{equation*}
with bandwidth $h$ chosen by the median heuristic on the transformed clean fitting split.

Thresholds labelled ``Stein $q$'' use the corresponding clean-reference or population $q$-quantile of this anomaly score, as stated in the table caption. In multivariate synthetic geometry checks, the same formula is applied componentwise with $T(x)=x$ and the empirical covariance-scaled Gaussian score. Rows labelled Mahalanobis or GMM Mahalanobis use the stated scores instead of the Stein score-norm.

These rows should be read as one covariate-geometric score implementation. The evaluated coverage statements are the score-agnostic retained-law bounds, not a separate Stein-specific certificate.

The run summaries distinguish \emph{mixture diagnostics} from \emph{finite/audit-style lower bounds}.  The former use knowledge of the simulation laws to evaluate $p_c,p_d,\tilde\varepsilon_\star,\Delta_{\rm trim,+}^A,D_{Q,+}$, and $L_{\rm mix,+}$.  The latter use independent structure/audit samples and concentration corrections.  They are reported as conservative certification diagnostics, not as the exact/scalar theorem bound.  In the displayed experiments, these lower bounds are finite and directionally informative, but remain too conservative to certify nominal $0.9$ coverage.

\subsection{Separation phase diagram}
Table~\ref{tab:phase-diagram} shows a representative slice of the separation phase diagram at $\varepsilon=0.3$ and $m=800$.  When separation is small, the selected method is either ordinary split or a clean-reference trimming rule that retains almost all dirty points.  As separation increases, $p_d$ and $\tilde\varepsilon_\star$ fall, and empirical clean coverage moves into the near-nominal range.

\begin{table}[h]
  \centering
  \scriptsize
  \setlength{\tabcolsep}{3pt}
  \resizebox{\linewidth}{!}{%
  \begin{tabular}{lcccccc}
    \toprule
    Separation & Selected rule & Clean coverage & Width & $p_d$ & $\tilde\varepsilon_\star$ & $L_{\rm mix,+}$ \\
    \midrule
    0.0 & Ordinary split & $0.8593\,[0.8557,0.8629]$ & 1.7720 & 1.0000 & 0.3000 & 0.6555 \\
    0.5 & Stein clean-ref $q=0.950$ & $0.8595\,[0.8558,0.8631]$ & 1.7717 & 0.9601 & 0.3035 & 0.6498 \\
    1.0 & Stein clean-ref $q=0.950$ & $0.8635\,[0.8599,0.8670]$ & 1.7907 & 0.8753 & 0.2841 & 0.6662 \\
    2.0 & Stein clean-ref $q=0.950$ & $0.8806\,[0.8775,0.8836]$ & 1.8731 & 0.4706 & 0.1749 & 0.7550 \\
    4.0 & Stein clean-ref $q=0.950$ & $0.9013\,[0.8987,0.9039]$ & 1.9888 & 0.0064 & 0.0028 & 0.8934 \\
    8.0 & Stein clean-ref $q=0.950$ & $0.9015\,[0.8988,0.9041]$ & 1.9900 & 0.0000 & 0.0000 & 0.8958 \\
    \bottomrule
  \end{tabular}}
  \caption{Separation phase diagram slice at $\varepsilon=0.3$ and $m=800$. The transition is driven by dirty retention $p_d$ and retained mixture coefficient $\tilde\varepsilon_\star$, not by trimming alone.}
  \label{tab:phase-diagram}
\end{table}

\subsection{Clean-reference and threshold-policy diagnostics}
Clean-reference $q=0.950$ thresholding is efficient but can be coverage-aggressive. In the clean-reference concentration experiment, increasing the clean reference size from $64$ to $512$ changes coverage only from $0.8847$ to $0.8872$. The dominant effect is not reference-sample noise, but the population clean distortion induced by $q=0.950$. Table~\ref{tab:threshold-policy} summarises the threshold-policy comparison. The KSD-geometric policy is more conservative and closer to nominal coverage than fixed or clean-reference $q=0.950$, but gives wider intervals.

\begin{table}[h]
  \centering
  \scriptsize
  \setlength{\tabcolsep}{3pt}
  \resizebox{\linewidth}{!}{%
  \begin{tabular}{lcccccc}
    \toprule
    Policy & Clean coverage & Width & $p_c$ & $\Delta_{+}$ & $L_{\rm mix,+}$ & Audit LB$_+$ \\
    \midrule
    Clean-reference $q=0.950$ & $0.8923\,[0.8880,0.8967]$ & 2.9203 & 0.9453 & 0.0153 & 0.8847 & 0.6031 \\
    Fixed $q=0.950$ & $0.8927\,[0.8883,0.8970]$ & 2.9231 & 0.9486 & 0.0145 & 0.8854 & 0.5996 \\
    KSD-geometric & $0.8979\,[0.8937,0.9020]$ & 2.9741 & 0.9861 & 0.0048 & 0.8951 & 0.6138 \\
    Oracle selected & $0.8927\,[0.8883,0.8970]$ & 2.9231 & 0.9486 & 0.0145 & 0.8854 & 0.5996 \\
    \bottomrule
  \end{tabular}}
  \caption{Threshold-policy comparison. The KSD-geometric policy improves the coverage-efficiency trade-off relative to $q=0.950$ rules. However, the audit-style lower bound remains far below nominal coverage.}
  \label{tab:threshold-policy}
\end{table}

Localised threshold policies must also report fallback rates. In the localised certificate experiment, loc4/R30 gives coverage $0.8998$ with fallback rate $0.01$. By contrast, loc6/R30 gives similar coverage, $0.8977$, but with fallback rate $0.96$. The latter should be read as a fallback-driven policy outcome, not as stable success of the primary localised certificate.

\subsection{Large-calibration behaviour}
\label{app:large calibration behaviour}
Table~\ref{tab:large-m} studies the nominal $q=0.950$ policy as the calibration size grows, using a separate large-calibration run. This run refits the auxiliary score and threshold policy. It is therefore not the same population-threshold instance as Table~\ref{tab:main-tradeoff}. Dirty-retention numbers should not be compared as if $p_d$ were fixed across the two tables.

Within the large-calibration run, the one-sided mixture diagnostic is approximately $L_{\rm mix,+}=0.8853$. The conservative lower bound rises from $0.7132$ at $m=320$ to $0.8464$ at $m=10000$. This supports the interpretation that finite-sample slack shrinks with calibration size, while the limiting clean-coverage level is governed by retained-law-to-clean-law transfer.

\begin{table}[h]
  \centering
  \scriptsize
  \setlength{\tabcolsep}{4pt}
  \begin{tabular}{rrrrrr}
    \toprule
    $m$ & Clean coverage & Width & $L_{\rm mix,+}$ & Cons. LB$_+$ & Audit LB$_+$ \\
    \midrule
    320 & $0.8878\,[0.8838,0.8918]$ & 2.8659 & 0.8853 & 0.7132 & 0.6560 \\
    1000 & $0.8903\,[0.8880,0.8925]$ & 2.8775 & 0.8853 & 0.7788 & 0.7225 \\
    3000 & $0.8909\,[0.8894,0.8924]$ & 2.8812 & 0.8853 & 0.8191 & 0.7633 \\
    10000 & $0.8888\,[0.8878,0.8898]$ & 2.8601 & 0.8853 & 0.8464 & 0.7910 \\
    \bottomrule
  \end{tabular}
  \caption{Large-calibration behaviour for the nominal Stein $q=0.950$ policy in a separate large-calibration run. Increasing $m$ reduces the slack in the conservative lower bound, but the population mixture diagnostic remains below the nominal $0.9$ level.}
  \label{tab:large-m}
\end{table}

\subsection{Label-only and stress-test regimes}
Label-only contamination has $Q_X=P_X$, so a covariate-only anomaly score has no direct signal for corrupted labels. In the label-only experiment, ordinary split conformal has clean coverage $0.9723$, $0.9998$, and $1.0000$ at $\varepsilon=0.1,0.2,0.3$, with widths $2.7200$, $5.6192$, and $7.8897$, respectively. The issue is therefore not undercoverage, but inefficient overcoverage caused by response-side corruption. This lies outside the favourable score-visible trimming regime.

The stress tests separate the mechanisms. Heavy-tail stress gives Stein clean-reference coverage $0.8978$ with width $2.9616$, suggesting useful trimming. High-leverage stress causes severe overcoverage and very wide intervals for non-trimming baselines. Thus, selecting maximal coverage alone can choose inefficient intervals. Response stress has $p_d\approx p_c$, so covariate anomaly trimming has little leverage over response-side corruption.

\subsection{Non-Gaussian score geometry}
Table~\ref{tab:nongaussian} summarises representative non-Gaussian geometry results. The selected score differs across regimes. This is expected under the score-agnostic theorem. Validity depends on the retained law and score CDFs after the score is fixed. Efficiency and diagnostic quality depend on how well the score matches the clean geometry and the direction of contamination.

\begin{table}[h]
  \centering
  \scriptsize
  \setlength{\tabcolsep}{3pt}
  \resizebox{\linewidth}{!}{%
  \begin{tabular}{lcccccc}
    \toprule
    Setting & Selected score & Clean coverage & Width & $p_d$ & $\tilde\varepsilon_\star$ & $L_{\rm mix,+}$ \\
    \midrule
    Bimodal moderate & Mahalanobis & $0.8918\,[0.8876,0.8960]$ & 1.9351 & 0.4401 & 0.1360 & 0.7822 \\
    Bimodal overlap & GMM Mahalanobis & $0.8948\,[0.8904,0.8991]$ & 1.9516 & 0.3590 & 0.1144 & 0.7999 \\
    Bimodal separated & Stein score-norm & $0.8997\,[0.8962,0.9031]$ & 1.9765 & 0.1985 & 0.0497 & 0.8571 \\
    Banana bridge & GMM Mahalanobis & $0.8961\,[0.8934,0.8989]$ & 1.6758 & $3.4\!\times10^{-5}$ & $9.0\!\times10^{-6}$ & 0.8921 \\
    Two moons gap & Stein score-norm & $0.9708\,[0.9654,0.9761]$ & 2.3689 & 0.1519 & 0.0777 & 0.8787 \\
    \bottomrule
  \end{tabular}}
  \caption{Non-Gaussian score-geometry diagnostics. The Stein score-norm is useful in some settings. Mixture-aware Mahalanobis scores are competitive or better in others. Coverage-best or overcoverage rows should not be interpreted as the shortest valid methods.}
  \label{tab:nongaussian}
\end{table}

\subsection{Compute resources}
The experiments do not require specialised compute. All reported simulations were run on a personal laptop using Python~3.12, an AMD Ryzen~9 8945HS CPU with Radeon~780M Graphics, and 32GB memory. Reproducing the full set of reported experiments required approximately 40 minutes of wall-clock time. No GPU acceleration was used.

%%%%%%%%%%%%%%%%%%%%%%%%%%%%%%%%%%%%%%%%%%%%%%%%%%%%%%%%%%%%%%%%%%%%%%%%%%%%%

\subsection{Semi-synthetic ACS classification}
\label{app:acs_classification}

We evaluate the retained law diagnostic on the 2018 American Community Survey (ACS) ACSEmployment classification task using Folktables \citep{ding2021retiring}. California (CA) defines the clean target distribution $P$. Texas (TX), New York (NY), Florida (FL), and their equal mixture define the contaminating distribution $Q$. We consider contamination levels $\varepsilon\in\{0.1,0.2,0.3\}$.

All methods use the same logistic regression predictor and randomised adaptive prediction sets (APS) score at nominal coverage $0.9$. The experiment uses 50 Monte Carlo repetitions. We report the results for the equal mixture of TX, NY, and FL under the fixed $0.95$ threshold policy.

We compare direct clean calibration, ordinary APS, source oracle trimming, and three covariate based trimming methods: GMM Mahalanobis, Isolation Forest, and Mahalanobis. Source labels are used only to construct the contamination and evaluate the oracle. They are not available to the three implementable trimming methods.

\begin{table}[t]
    \centering
    \label{tab:clean_coverage_comparison}
    \small
    \setlength{\tabcolsep}{4pt}
    \resizebox{\linewidth}{!}{%
    \begin{tabular}{@{}cllcccc@{}}
        \toprule
        $\varepsilon$
        & Method
        & Clean coverage [95\% MC CI]
        & Mean set size
        & $p_c$
        & $p_d$
        & $\tilde{\varepsilon}$ \\
        \midrule

        \multirow{6}{*}{0.1}
        & Direct clean calibration
        & 0.8990 [0.8974, 0.9005]
        & 1.4386 & --- & --- & --- \\
        & Ordinary APS
        & 0.9004 [0.8991, 0.9018]
        & 1.4432 & 1.0000 & 1.0000 & 0.1000 \\
        & Source-oracle trim
        & 0.9009 [0.8995, 0.9024]
        & 1.4446 & 1.0000 & 0.0000 & 0.0000 \\
        & GMM-Mahalanobis trim
        & 0.9006 [0.8993, 0.9019]
        & 1.4438 & 0.9494 & 0.9477 & 0.0998 \\
        & Isolation Forest trim
        & 0.9003 [0.8990, 0.9016]
        & 1.4430 & 0.9490 & 0.9527 & 0.1003 \\
        & Mahalanobis trim
        & 0.9005 [0.8991, 0.9018]
        & 1.4433 & 0.9502 & 0.9482 & 0.0998 \\

        \midrule

        \multirow{6}{*}{0.2}
        & Direct clean calibration
        & 0.8990 [0.8974, 0.9005]
        & 1.4386 & --- & --- & --- \\
        & Ordinary APS
        & 0.9006 [0.8989, 0.9022]
        & 1.4438 & 1.0000 & 1.0000 & 0.2000 \\
        & Source-oracle trim
        & 0.9015 [0.8998, 0.9033]
        & 1.4468 & 1.0000 & 0.0000 & 0.0000 \\
        & GMM-Mahalanobis trim
        & 0.9004 [0.8987, 0.9021]
        & 1.4431 & 0.9494 & 0.9477 & 0.1997 \\
        & Isolation Forest trim
        & 0.9003 [0.8986, 0.9020]
        & 1.4429 & 0.9490 & 0.9527 & 0.2006 \\
        & Mahalanobis trim
        & 0.9004 [0.8988, 0.9021]
        & 1.4432 & 0.9502 & 0.9482 & 0.1997 \\

        \midrule

        \multirow{6}{*}{0.3}
        & Direct clean calibration
        & 0.8990 [0.8974, 0.9005]
        & 1.4386 & --- & --- & --- \\
        & Ordinary APS
        & 0.8992 [0.8974, 0.9010]
        & 1.4393 & 1.0000 & 1.0000 & 0.3000 \\
        & Source-oracle trim
        & 0.9002 [0.8984, 0.9019]
        & 1.4420 & 1.0000 & 0.0000 & 0.0000 \\
        & GMM-Mahalanobis trim
        & 0.8992 [0.8974, 0.9010]
        & 1.4393 & 0.9494 & 0.9477 & 0.2996 \\
        & Isolation Forest trim
        & 0.8989 [0.8971, 0.9007]
        & 1.4385 & 0.9490 & 0.9527 & 0.3008 \\
        & Mahalanobis trim
        & 0.8990 [0.8972, 0.9007]
        & 1.4384 & 0.9502 & 0.9482 & 0.2996 \\

        \bottomrule
    \end{tabular}%
    }
    \caption{Semi synthetic ACS classification under cross state calibration contamination. California is the clean target, and an equal mixture of Texas, New York, and Florida is the contaminating distribution. Results are averaged over 50 Monte Carlo repetitions at nominal coverage $0.9$. Brackets give 95\% Monte Carlo confidence intervals for mean clean coverage. Mean set size is the average number of labels in the prediction set. Here, $p_c$ and $p_d$ are the clean and dirty retention probabilities, and $\tilde{\varepsilon}$ is the retained contamination coefficient. Direct clean calibration uses clean calibration data, while source oracle trimming uses true component labels. The three practical trimming methods use covariate anomaly scores without source labels.}
    \label{tab:clean_coverage_comparison}
\end{table}

\paragraph{Results and interpretation.}

Table~\ref{tab:clean_coverage_comparison} shows that all three practical anomaly scores have similar clean and dirty retention probabilities. Across these methods, $p_c$ ranges from $0.9490$ to $0.9502$, while $p_d$ ranges from $0.9477$ to $0.9527$. Consequently, the retained contamination coefficient $\tilde{\varepsilon}$ remains close to the original contamination level $\varepsilon$.

The practical trimming methods provide little improvement over ordinary APS. For $\varepsilon=0.1$, their mean clean coverage ranges from $0.9003$ to $0.9006$, compared with $0.9004$ for ordinary APS. For $\varepsilon=0.2$, the corresponding range is $0.9003$ to $0.9004$, compared with $0.9006$ for ordinary APS. For $\varepsilon=0.3$, it is $0.8989$ to $0.8992$, compared with $0.8992$ for ordinary APS.

The largest absolute difference in mean coverage between a practical trimming method and ordinary APS is only $0.0003$ in this experiment. Mean prediction set sizes also remain nearly unchanged.

Source oracle trimming removes all contaminating observations while retaining all clean observations, giving $p_c=1$, $p_d=0$, and $\tilde{\varepsilon}=0$. Its mean coverage remains close to nominal coverage. This shows that even perfect identification of contaminated observations need not produce a large coverage improvement when ordinary APS already achieves near nominal coverage.

These results support the coefficient level separation dichotomy in Proposition~\ref{thm:separation-dichotomy-main}. When $p_c\approx p_d$, trimming cannot substantially reduce contamination through the retained mixture coefficient. This does not rule out improvements through changes in the retained score distributions. Rather, it shows that common covariate anomaly scores provide little practical benefit in this overlapping domain setting.

\subsection{Label only contamination and structured robust baselines}
We conduct a separate label only contamination experiment for the California ACS classification task \citep{ding2021retiring}. Since the contaminating mechanism keeps the covariate distribution unchanged, $Q_X=P_X$. As a result, covariate only anomaly scores cannot preferentially remove corrupted observations by separating retention probabilities.

We compare ordinary APS against Contamination Robust Conformal Prediction (CRCP) \citep{clarkson2024contamination}. We also compare ordinary label conditional APS against the known noise method of \citet{sesia2025adaptive}. The robust methods are assessed under the true synthetic noise mechanism.

Table~\ref{tab:label_only_contamination} presents marginal coverage and mean prediction set size for ordinary APS and CRCP. For the label conditional methods, we present worst class coverage and mean prediction set size.

CRCP lessens the conservatism of ordinary APS. At contamination level $\varepsilon=0.2$, for instance, its marginal coverage is $0.9074$, while ordinary APS has $0.9309$, and its mean set size falls from $1.5429$ to $1.4662$.

The method of \citet{sesia2025adaptive} raises worst class coverage at all three evaluated noise levels. It also lowers mean prediction set size at the two higher noise levels.

These results indicate that methods built for structured label corruption can be more appropriate than covariate trimming when the contamination mechanism is known. They tackle a different problem from the retained law diagnostic, which describes the effect of a proposed trimming rule instead of directly correcting the label noise mechanism.

\begin{table}[t]
    \centering
    \small
    \begin{tabular}{ccccc}
        \toprule
        $\varepsilon$
        & \shortstack{Ordinary APS Cov. /\\ Size}
        & \shortstack{CRCP Cov. /\\ Size}
        & \shortstack{Ordinary LC-APS Worst-Class /\\ Size}
        & \shortstack{Sesia LC Worst-Class /\\ Size} \\
        \midrule
        0.1 & 0.9150 / 1.4905 & 0.9017 / 1.4478
            & 0.9039 / 1.4867 & 0.9109 / 1.4864 \\
        0.2 & 0.9309 / 1.5429 & 0.9074 / 1.4662
            & 0.9104 / 1.5434 & 0.9163 / 1.5086 \\
        0.3 & 0.9448 / 1.5912 & 0.9141 / 1.4877
            & 0.9164 / 1.5987 & 0.9247 / 1.5371 \\
        \bottomrule
    \end{tabular}
    \caption{
        Classification results under label-only contamination.
        Ordinary APS and CRCP report marginal coverage.
        The label-conditional methods report worst-class coverage.
        Set size is reported for all methods.
        The robust baselines use the true synthetic noise mechanism.
    }
    \label{tab:label_only_contamination}
\end{table}

\subsection{Allocation of a Limited Clean Data Budget}

We compare three uses of a limited clean labelled sample. These uses are direct conformal calibration, independent threshold tuning, and auditing a fixed final prediction set.

Direct calibration uses the clean sample to construct a new conformal cut-off. Independent tuning uses the sample to select a trimming policy before final calibration. Auditing instead assesses the coverage of an already constructed prediction set without modifying that set.

We consider clean budgets of 32, 64, 128, 256, and 512 observations. All results are averaged over 50 Monte Carlo repetitions.

In this experiment, independent tuning always selects the no trimming fallback. The candidate prediction sets have clean coverage close to the nominal target, so no candidate has a sufficient margin to pass the joint tuning screen.

The fixed final prediction set has a mean clean coverage of 0.8994. However, the exact 95\% binomial audit certificate verifies coverage of at least 0.9 in only 2\% to 10\% of the repetitions.

Direct clean calibration provides a different prediction set and the usual finite sample marginal coverage guarantee when its exchangeability assumptions hold. In contrast, auditing provides an independent confidence statement about an already selected set.

These findings clarify the different roles of calibration, tuning, and auditing. Direct clean calibration is often preferable when enough clean labelled data are available and marginal validity is the primary objective. Auditing is valuable when the final prediction set must remain fixed, but its ability to certify a target coverage level depends on the available sample size and the coverage margin.

\begin{table}[t]
    \centering
    \small
    \setlength{\tabcolsep}{4pt}
    \resizebox{\linewidth}{!}{%
    \begin{tabular}{cccccccc}
        \toprule
        \shortstack{Clean\\ budget}
        & \shortstack{Direct cal.\\ coverage}
        & \shortstack{Direct\\ size}
        & \shortstack{Tuning\\ coverage}
        & \shortstack{Tuning\\ size}
        & \shortstack{No-trim\\ rate}
        & \shortstack{Audit mean\\ 95\% LB}
        & \shortstack{Audit cert.\\ rate} \\
        \midrule
        32  & 0.9171 & 1.5099 & 0.8993 & 1.4381
            & 1.00 & 0.760 & 0.02 \\
        64  & 0.9021 & 1.4529 & 0.8993 & 1.4381
            & 1.00 & 0.820 & 0.04 \\
        128 & 0.9071 & 1.4650 & 0.8993 & 1.4381
            & 1.00 & 0.850 & 0.04 \\
        256 & 0.9042 & 1.4545 & 0.8993 & 1.4381
            & 1.00 & 0.865 & 0.10 \\
        512 & 0.8968 & 1.4310 & 0.8993 & 1.4381
            & 1.00 & 0.876 & 0.08 \\
        \bottomrule
    \end{tabular}%
    }
    \caption{
        Comparison of three uses of a limited clean-data budget
        over 50 Monte Carlo repetitions.
        Direct calibration constructs a new prediction set.
        Independent tuning selects a policy using clean data.
        Auditing certifies a fixed final prediction set.
        The audit certification rate is the fraction of
        repetitions in which the exact 95\% lower bound
        reaches 0.9.
    }
    \label{tab:clean_budget_comparison}
\end{table}

\subsection{Diagnostic ranking and coverage efficiency trade off}

We ask whether the population diagnostics retain the empirical ordering of candidate trimming policies. To this end, we compare the component wise mixture lower bound $L_{\mathrm{mix}}$ with a direct scalar gap lower bound $L_{\mathrm{direct}}$. Agreement in ranking is summarised by Kendall's $\tau_b$ and pairwise concordance. The component wise diagnostic yields Kendall's $\tau_b=0.496$, with a stratified bootstrap 95\% confidence interval of [0.207, 0.569]. The direct scalar gap diagnostic yields $\tau_b=0.482$, with a confidence interval of [0.188, 0.569]. Across the 843 resolved comparisons, concordance rates are 96.9\% and 97.3\%, respectively. These resolved pairs represent approximately 58.5\% of all 1,440 pairs. The remaining comparisons remain unresolved under the evaluation criterion. Taken together, these results indicate that the diagnostics can offer useful relative rankings among empirically distinguishable candidates. However, they do not establish that the component wise lower bound is necessary or sufficient for actual coverage improvement. Coverage ranking, moreover, does not capture prediction set efficiency. We therefore examine empirical Pareto frontiers based on coverage and prediction set size. Isolation Forest, GMM Mahalanobis, Mahalanobis, and ordinary APS each appear on the empirical frontiers. Across the evaluated settings, no single method dominates the others. Accordingly, the diagnostics are best used to interpret and screen candidate trimming rules, whereas empirical coverage and set size remain necessary for judging practical performance.

\begin{table}[t]
    \centering
    \small
    \setlength{\tabcolsep}{4pt}
    \resizebox{\linewidth}{!}{%
    \begin{tabular}{@{}lcccccc@{}}
        \toprule
        Diagnostic
        & Kendall $\tau_b$
        & 95\% CI
        & Concordant
        & Discordant
        & Unresolved
        & Concordance among resolved \\
        \midrule
        Componentwise $L_{\mathrm{mix}}$
        & 0.496
        & [0.207, 0.569]
        & 817
        & 26
        & 597
        & 96.9\% \\
        Direct scalar-gap $L_{\mathrm{direct}}$
        & 0.482
        & [0.188, 0.569]
        & 820
        & 23
        & 597
        & 97.3\% \\
        \bottomrule
    \end{tabular}%
    }
    \caption{Ranking consistency of the component wise and direct scalar-gap diagnostics.}
    \label{tab:ranking_consistency}
\end{table}

\begin{table}[t]
    \centering
    \small
    \setlength{\tabcolsep}{6pt}
    \makebox[\linewidth][c]{%
        \begin{tabular}{@{}cc@{}}
            \toprule
            Method family
            & Number of frontier appearances \\
            \midrule
            Isolation Forest trimming
            & 55 \\
            GMM-Mahalanobis trimming
            & 54 \\
            Mahalanobis trimming
            & 52 \\
            Ordinary APS
            & 12 \\
            \bottomrule
        \end{tabular}%
    }
    \caption{Pareto frontier appearances by method family.}
    \label{tab:pareto_frontier}
\end{table}

\subsection{Additional Covariance and Scalar-Bound Checks}
\label{app:additional covariance and scalar-bound checks}
\paragraph{The covariance identity is visible in the data}
\label{subsec:experiment-covariance}

Proposition~\ref{prop:trimming-bias-orthogonality} predicts that the clean trimming distortion equals a one-sided covariance under $P$
\begin{equation*}
  \Delta_{\rm trim,+}^A
  \;=\;
  p_c^{-1}\,\sup_t\bigl\{\mathrm{Cov}_P(\mathbf{1}\{A(Z) \le t\},\,K)\bigr\}_+ .
\end{equation*}
Reducing $\Delta_{\rm trim,+}^A$ is therefore equivalent to reducing the implied covariance envelope $p_c,\Delta_{\rm trim,+}^A$, after adjusting for the clean retention probability. Table~\ref{tab:covariance-identity-summary} interprets three score-visible thresholds through this identity. As $q$ increases from $0.950$ to $0.990$, the threshold becomes less aggressive. It retains more clean mass and trims fewer points. Dirty retention is negligible in all three rows. The implied covariance envelope decreases by about a factor of four, from $1.30 \times 10^{-2}$ to $3.6 \times 10^{-3}$. The clean-side trimming cost decreases accordingly. Improvement in this regime is therefore mainly a clean score-neutrality effect, not stronger contamination removal.

\begin{table}[t]
  \centering
  \small
  \setlength{\tabcolsep}{6pt}
  \begin{tabular}{lcccc}
    \toprule
    Threshold & $p_c$ & $p_d$ & $\Delta_{\rm trim,+}^A$ & $p_c\,\Delta_{\rm trim,+}^A$ \\
    \midrule
    Stein $q = 0.950$ & $0.9494$ & $2.8 \times 10^{-5}$ & $0.0137$ & $0.0130$ \\
    Stein $q = 0.975$ & $0.9752$ & $6.1 \times 10^{-5}$ & $0.0075$ & $0.0073$ \\
    Stein $q = 0.990$ & $0.9897$ & $2.6 \times 10^{-4}$ & $0.0036$ & $0.0036$ \\
    \bottomrule
  \end{tabular}
  \caption{Covariance-identity interpretation of three score-visible thresholds. The last column reports the implied one-sided covariance envelope $\sup_t\{\mathrm{Cov}_P(\mathbf{1}\{A \le t\}, K)\}_+$ up to rounding. Less aggressive trimming reduces this envelope by roughly a factor of four.}
  \label{tab:covariance-identity-summary}
\end{table}

\paragraph{The scalar bound is finite-sample sharp}
\label{subsec:experiment-sharpness}

Theorem~\ref{thm:one-sided-transfer-sharpness} gives the precise minimax meaning of sharpness. Once the score-law comparison is reduced to the scalar constraint $d_{R \to P,+}^A\le d$, the bound $L_{\rm fs}^{m,\mu}(d)$ is the best possible uniform lower bound.

Table~\ref{tab:sharpness-values} reports this curve at the main calibration scale $m=320$, retention probability $\mu=0.95$, and $\alpha=0.1$. The finite-sample bound is slightly above the asymptotic line $[1-\alpha-d]_+$. This is because the inclusive conformal quantile is finite-sample conservative. The sharpness construction in the proof attains this curve exactly.

\begin{table}[t]
  \centering
  \small
  \setlength{\tabcolsep}{8pt}
  \begin{tabular}{ccc}
    \toprule
    $d$ & $L_{\rm fs}^{320,\,0.95}(d)$ & $[1 - \alpha - d]_+$ \\
    \midrule
    $0.000$ & $0.9015$ & $0.9000$ \\
    $0.002$ & $0.8995$ & $0.8980$ \\
    $0.005$ & $0.8965$ & $0.8950$ \\
    $0.010$ & $0.8915$ & $0.8900$ \\
    $0.020$ & $0.8815$ & $0.8800$ \\
    $0.050$ & $0.8515$ & $0.8500$ \\
    \bottomrule
  \end{tabular}
  \caption{Finite-sample scalar lower bound $L_{\rm fs}^{m,\mu}(d)$ and its asymptotic counterpart for $m=320$, $\mu=0.95$, and $\alpha=0.1$. Theorem~\ref{thm:one-sided-transfer-sharpness} shows that the $L_{\rm fs}$ column is attainable within the scalar-discrepancy class.}
  \label{tab:sharpness-values}
\end{table}

\section{Additional experimental diagnostics}
\label{app:experimental-diagnostics}

\subsection{Reporting status}
\label{app:experimental-code-audit}
The simulation summaries report the one-sided quantities required by Theorem~\ref{thm:fixed-threshold-trimmed-coverage}: $\Delta_{\rm trim,+}^{A}$, $D_{Q,+}$ when the retained dirty component is estimable, the one-sided transfer loss, $L_{\rm mix,+}$, a conservative finite/audit lower bound, and the symmetric diagnostics used in supplementary sensitivity checks.

The manuscript keeps these objects separate. Population diagnostics use knowledge of the simulation laws. Conservative lower bounds use independent audit information and concentration corrections. Empirical coverage is estimated from finite Monte Carlo test sets.

Rows with $p_d=0$ display the retained dirty discrepancy as ``N/A'', rather than as a numerical failure. When $p_d$ is extremely small, $D_{Q,+}$ is a conditional rare-event quantity and may be unstable as a stand-alone diagnostic. In such rows, the dirty contribution $\tilde\varepsilon_\star D_{Q,+}$, or its worst-case upper bound $\tilde\varepsilon_\star$, is more informative for the coverage lower bound.

\subsection{Population diagnostics versus conservative lower bounds}
The one-sided mixture diagnostic can be close to the observed clean coverage. In the main score-visible row, the observed clean coverage is 0.89499, while $L_{\rm mix,+}=0.89638$. The Monte Carlo interval contains the mixture diagnostic, so we interpret this as agreement within simulation error.

By contrast, the conservative finite/audit lower bounds are much smaller. Thus, $L_{\rm mix,+}$ explains the retained-law regime, while the conservative lower bounds are safe but usually too loose to certify nominal 0.9 coverage in these examples.

\subsection{Coverage-best tables}
Coverage-best tables are kept only as appendix diagnostics. They include oracle or empirical coverage-maximising choices, which often select severely over-covering and wide intervals. For example, in non-Gaussian coverage-best summaries, methods with coverage near one can have much wider intervals than fixed or independently tuned policies. These tables should not be described as implementable best-method comparisons. They are maximum empirical-coverage benchmarks.

\subsection{Localised certificates}
The localised-certificate summaries are aggregated at the policy level, rather than split by realised fallback labels. Reporting the fallback rate is essential. In these summaries, loc4/R30 has fallback rate 0.01 and coverage 0.89979, while loc6/R30 has fallback rate 0.96 and coverage 0.89766. The latter mainly reflects the fallback policy. It is not evidence that the primary localised certificate reliably selects the displayed threshold.

\subsection{Recommended table convention}
For experimental tables supporting the retained-law claim, report the method, threshold source, clean coverage with Monte Carlo interval, average width with Monte Carlo interval, $p_c$, $p_d$, $\tilde\varepsilon_\star$, $\Delta_{\rm trim,+}^{A}$, $D_{Q,+}$ or the dirty contribution $\tilde\varepsilon_\star D_{Q,+}$, $L_{\rm mix,+}$, the conservative finite/audit lower bound, and the fallback rate when the threshold policy is adaptive. This convention keeps theoretical transfer quantities, certification diagnostics, and empirical performance separate.

\end{document}